\pdfoutput=1

\documentclass{article}

\usepackage{arxiv}
\usepackage{natbib}

\usepackage{amsmath,amssymb,amsfonts,amsthm,mathtools}

\usepackage[utf8]{inputenc}
\usepackage[T1]{fontenc}

\usepackage[hidelinks]{hyperref}
\usepackage{url}
\usepackage{booktabs}
\usepackage{nicefrac}
\usepackage{microtype}
\usepackage{xcolor}
\usepackage{graphicx}
\usepackage{subcaption}
\usepackage{multirow}
\usepackage{algorithm}
\usepackage{algorithmic}

\usepackage[capitalize,noabbrev]{cleveref}

\theoremstyle{plain}
\newtheorem{theorem}{Theorem}
\newtheorem{proposition}{Proposition}
\newtheorem{lemma}{Lemma}
\newtheorem{corollary}{Corollary}
\theoremstyle{definition}
\newtheorem{definition}{Definition}
\newtheorem{assumption}{Assumption}
\theoremstyle{remark}
\newtheorem{remark}{Remark}

\theoremstyle{plain}
\newtheorem*{theoremrestate}{Theorem}
\newtheorem*{propositionrestate}{Proposition}

\crefname{assumption}{Assumption}{Assumptions}
\Crefname{assumption}{Assumption}{Assumptions}

\newcommand{\E}{\mathbb{E}}
\newcommand{\Prob}{\mathbb{P}}
\newcommand{\Var}{\mathrm{Var}}
\DeclareMathOperator{\KL}{KL}
\DeclareMathOperator{\TV}{TV}

\DeclareMathOperator{\clip}{clip}
\newcommand{\Bern}{\mathrm{Bern}}

\newcommand{\cX}{\mathcal{X}}
\newcommand{\cA}{\mathcal{A}}
\newcommand{\cK}{\mathcal{K}}
\newcommand{\cW}{\mathcal{W}}
\newcommand{\cF}{\mathcal{F}}

\newcommand{\Sacc}{S_\lambda}
\newcommand{\RQ}{R_Q}
\newcommand{\CQ}{C_Q}
\newcommand{\GQ}{G_Q}
\newcommand{\betastar}{\beta^{*}}
\newcommand{\betastarhat}{\hat{\beta}^{*}}
\newcommand{\budget}{\varphi}
\newcommand{\lat}{\Lambda}
\newcommand{\nb}{\rho_\lambda}
\newcommand{\vtwo}{\mathrm{v2}}
\newcommand{\dtol}{\delta_{\mathrm{tol}}}
\newcommand{\nocert}{\perp}
\newcommand{\locfun}[1]{\frac{\E_P\!\left[w^2 #1\right]}{\left(\E_P\!\left[w #1\right]\right)^2}}

\DeclareMathOperator{\UCB}{UCB}
\DeclareMathOperator{\LCB}{LCB}

\makeatletter
\def\thm@space@setup{\thm@preskip=5pt plus 1pt minus 1pt
\thm@postskip=\thm@preskip}
\makeatother

\makeatletter
\@ifundefined{endkeywords}{%
  \@ifundefined{keywords}{\let\jmlr@kw@def\newenvironment}%
                         {\let\jmlr@kw@def\renewenvironment}%
  \jmlr@kw@def{keywords}{%
    \par\addvspace\medskipamount\noindent
  }{\par}%
}{}%
\makeatother

\title{Certify or Refuse: A Cross-Model Map for Selective Risk Control with Coverage Floors under Covariate Shift}

\author{%
  Jiamiao Liu \quad Dewen Qiao \quad Yu Zhang \quad Xuetao Chen\thanks{Corresponding author. Code and data: \url{https://github.com/yahiko-l/certify-or-refuse}.} \\
  \normalfont Department of Information, Xinqiao Hospital \\
  \normalfont Army Medical University (Third Military Medical University) \\
  \normalfont Chongqing 400037, China \\
  \normalfont \texttt{xuetao@tmmu.edu.cn}
}
\date{}

\begin{document}

\renewcommand{\thefootnote}{\fnsymbol{footnote}}
\maketitle
\renewcommand{\thefootnote}{\arabic{footnote}}
\setcounter{footnote}{0}

\begin{abstract}
Selective predictors with certified risk attain whatever coverage they attain; operators, however, impose an automation floor: answer at least a $\beta$-fraction of shifted target traffic while keeping at most an $\alpha$-fraction of returned answers wrong.
Under bounded-ratio covariate shift, we prove the \emph{Floor Certification Map}: once that hard coverage floor $\beta$ must be certified alongside the selection-conditioned risk level $\alpha$, certification acquires a feasibility frontier and a two-resource sample-complexity map, additive up to constants: risk paid in labeled source samples, the floor in unlabeled target samples.
Bounded-ratio shift alone buys the frontier; the two-resource rates are local, holding under a regular frontier margin and at slack below the local-regime threshold, with the upper bounds on a pre-registered nested-threshold lattice meeting a lattice margin condition and the lower bound on a compatible lattice constructed per slack. The displayed resource split is the operational route, since under oracle weights the floor additionally admits a labeled-source estimate that only weakens the requirement.
We give the map as three model-tagged results: a lower bound (Model-B), a matching oracle-weight upper bound (Model-A), and an implementable upper bound (Model-B$'$), valid under a pre-registered \emph{exact} stratified-shift model with nuisance cost priced explicitly.
The match is \emph{across} these three models rather than a single-model minimax theorem, and necessarily so: over the full bounded-ratio class no unknown-weight procedure matches at any sample size (Model-B is inconsistent, witnessed at $\alpha = \beta = \tfrac12$). So no single-model law holds over that class, and the cross-model map is the correct object, not a substitute for one. The nuisance's necessity is only \emph{partially} settled.
The upper bound's variance term is an accepted-region quantity bounded by a localized second moment, and the lower-bound hard slices localize to the same region, with no global effective-sample-size (ESS) term appearing on either side, though a fixed-ESS separation theorem is left open; both lower-bound axes vanish as $\beta \to 0$, so the floor creates the map.
Empirically, the registered bite family diverges with log--log slope $-2.002$ (within its pre-registered band); a $1{,}024$-cell audit records 0 violations where the formal certificates fire; and a single-corpus SQuAD$\to$NewsQA feasibility audit returns \emph{honest refusal}.

\end{abstract}


\section{Introduction}
\label{sec:intro}

\looseness=-1 An operator deploying a selective question-answering system commits to answering automatically at least a $\beta$-fraction of target traffic, with at most an $\alpha$-fraction of the automatically returned answers wrong.
Given $n$ labeled samples from the calibration distribution and $m$ unlabeled queries from the shifted target, when can this double promise be \emph{certified}, and when is the only honest output a refusal?

\looseness=-1 Selection-conditioned risk control exists in several forms (linear-expectation reformulation \citep{wang2025lec}, weighted e-values \citep{bai2026score}, cascaded calibration \citep{jia2026balancerag}), built on selective classification \citep{geifman2017selective}, risk control as multiple testing \citep{angelopoulos2021learn}, and weighted conformal prediction \citep{tibshirani2019conformal}; heuristic confidence deferral provably degrades under shift \citep{jitkrittum2023when}.
Yet all of these certify the risk of \emph{whatever coverage they attain}.
The classical reject-option line does characterize a risk--coverage frontier, but only the \emph{population} frontier under a guaranteed-coverage (bounded-abstention) objective \citep{franc2023optimal}.
What is missing is the \emph{finite-sample} object: no prior work proves lower bounds indexed by a coverage floor under covariate shift, nor maps what that floor costs in labeled versus unlabeled data; the nearest distribution-free lower-bound theory \citep{aldirawi2026conformal} is floor-free: one axis, no coverage-floor parameter anywhere.
Concretely, a risk-only certificate can be vacuously safe by abstaining on nearly all queries; it bounds the error of the few answers it does return while never guaranteeing that the operator's $\beta$-fraction of target traffic is automated at all; the floor is exactly the promise these certificates cannot make.
\looseness=-1 Figure~\ref{fig:hero} previews both halves of the result: the map's two-resource certifiable region in the $(n, m)$ plane (left), and the full-grid validity audit where the formal certificates fire with zero violations (right).

\begin{figure}[t]
    \centering
    \begin{subfigure}[b]{0.575\textwidth}
        \centering
        \includegraphics[width=\linewidth]{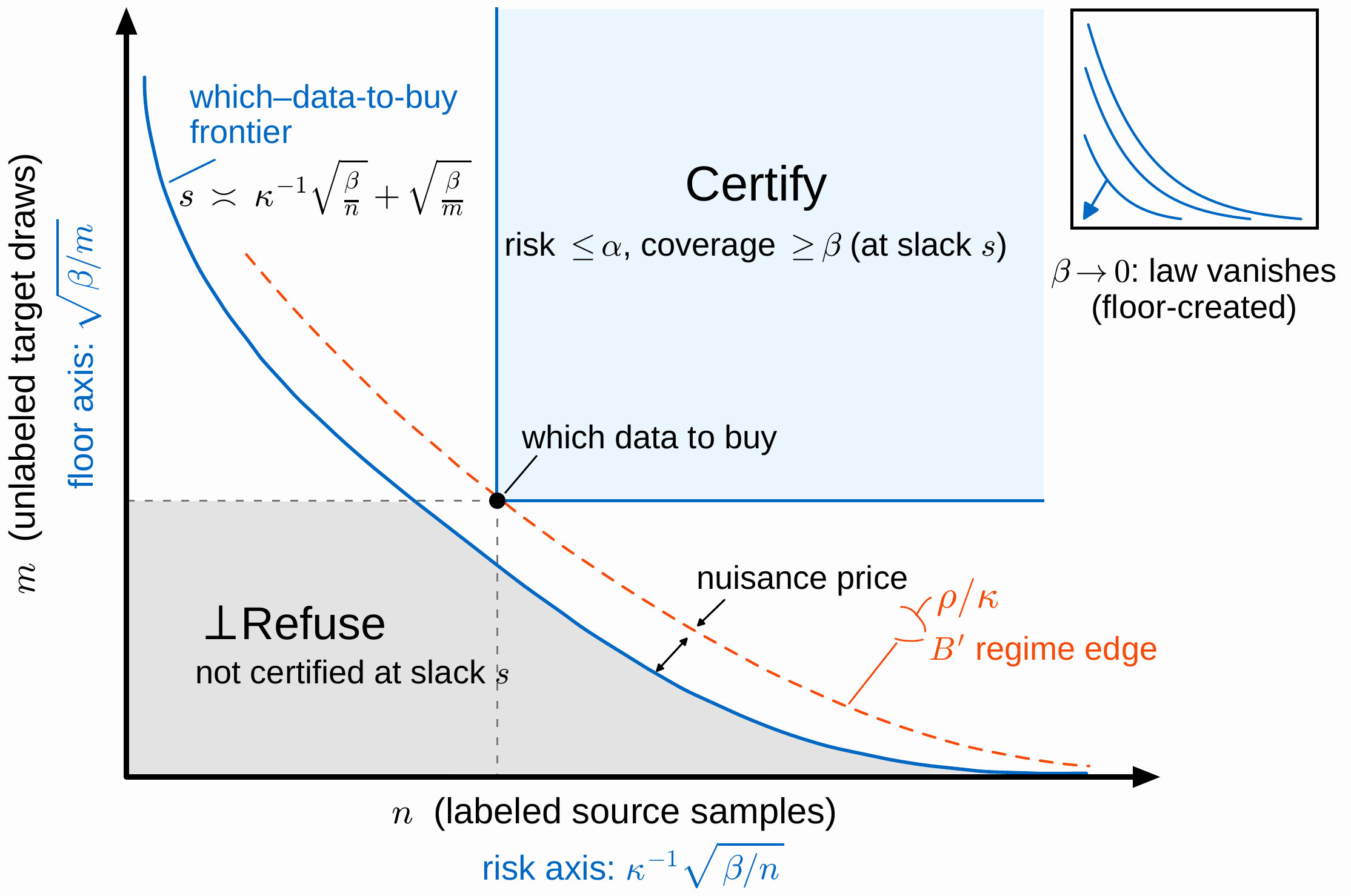}
        \caption{}
    \end{subfigure}
    \hfill
    \begin{subfigure}[b]{0.41\textwidth}
        \centering
        \includegraphics[width=\linewidth]{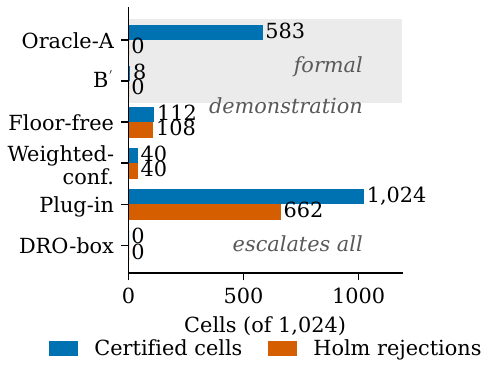}
        \caption{}
    \end{subfigure}
    \caption{\textbf{The Floor Certification Map and its validity at scale.}
    (a)~Certifying both the selection-conditioned risk $R_Q \le \alpha$ and a hard
    coverage floor $C_Q \ge \beta$ under bounded-ratio covariate shift splits the
    sample cost into two resources: at floor slack $s$, the certifiable region in the
    $(n, m)$ plane contains a quadrant whose corner traces
    $s \asymp \kappa^{-1}\sqrt{\beta/n} + \sqrt{\beta/m}$ (the canonical two-axis route shown; in Model-A a source-weighted floor branch enlarges it by an $n$-only half-strip); the implementable
    Model-B$'$ certificate adds an explicitly priced shift-model nuisance axis
    ($\rho_\lambda \lesssim \kappa s$; $K$-free core necessary (lattice-valued), unknown-$\eta$ open). Both lower-bound axes vanish
    as $\beta \to 0$: the map is created by the floor. (b)~Full-grid validity audit
    (1{,}024 cells): the formal arms certify 583/1{,}024 (oracle-A) and 8/1{,}024 (B$'$)
    cells with 0 violations and 0 Holm rejections; three demonstration arms
    without a joint risk-and-floor guarantee under shift violate at scale
    (Holm rejections: floor-free 108, weighted-conformal 40, plug-in 662);
    the distributionally robust optimization (DRO)-box baseline certifies 0. B$'$'s conservatism relative to oracle-A is
    the visible practical price of estimated weights. (Audit protocol and per-arm
    Clopper--Pearson bounds: Table~\ref{tab:validity_constants}.)}
    \label{fig:hero}
\end{figure}

This paper shows the floor changes the statistical problem.
We prove the Floor Certification Map: a hard coverage floor under covariate shift splits certification cost into a labeled-source risk axis and an unlabeled-target floor axis.
We also run a single-corpus SQuAD$\to$NewsQA feasibility audit of the certificate's honest refusal.
Certifying means selecting one policy $\lambda$ from a pre-fixed lattice of thresholds and returning the certificate ``$\RQ(\lambda) \le \alpha$ and $\CQ(\lambda) \ge \beta$ with confidence $1-\delta$,'' or else issuing an honest refusal $\nocert$ when no policy can be certified.
The frontier $\betastar(\alpha, Q)$ is the largest coverage certifiable at risk $\alpha$.

Three models tag how much is known about the covariate-shift weights $w$: Model-A (oracle weights, $w$ given exactly: the rate reference), Model-B (unknown weights, only $w \le B$ known: the lower-bound model), and Model-B$'$ (weights estimated under a pre-registered \emph{exact} stratified-shift model, nuisance priced explicitly).
The map states three model-tagged bounds: a Model-B lower bound, a Model-A oracle upper bound attaining the same rates, and an implementable Model-B$'$ upper bound.

At feasibility-frontier slack $s$, a canonical rate template organizes all three bounds:
\begin{equation}
\label{eq:law}
\underbrace{\;s^{\,\mathrm{certifiable}}\;}_{\text{slack threshold}} \;\;\asymp\;\; \underbrace{\kappa^{-1}\sqrt{\beta/n}}_{\text{labeled: risk}} \;+\; \underbrace{\sqrt{\beta/m}}_{\text{unlabeled: floor}} \;+\; \underbrace{\rho/\kappa}_{\text{shift-model nuisance}} .
\end{equation}
Here $\kappa$ is the local frontier margin and $\rho$ the accepted-region shift-model nuisance, both made precise in \S\ref{sec:setup}.
The display is an organizing template, not a same-model law: each model-tagged claim refines it.

\textbf{Scope of the map.}
The displayed axes hold at a fixed bounded-ratio constant $B$, whose factors are folded into the Model-A constants and appear explicitly in the Model-B$'$ nuisance axis ($B^2 K$, where $K$ is the number of pre-registered shift cells; Claim 3).
Theorem~\ref{thm:t3-target} shows this axis is the histogram estimator's price, not a proved-necessary axis. The necessary core is $K$-free (lattice-valued), leaving only the unknown-$\eta$ edge open.
The displayed split (risk in labeled source, floor in unlabeled target) is the operational Model-B/B$'$ route; under oracle weights (Model-A) the floor additionally admits a source-weighted (pooled) estimate, which only weakens the requirement (Appendix~\ref{app:modelA}).

Two structural facts shape the map.
First, the lower bounds are Le~Cam pairs at bounded likelihood ratio, so the impossibility vanishes as $\beta \to 0$: the map is floor-created, not overlap-created.
Second, the upper-bound variance proxy is driven by a localized accepted-region second moment $\E_P[w^2 \Sacc]$, not by global effective sample size (ESS); its normalized reading $\locfun{\Sacc}$ is heuristic (Remark~\ref{rmk:localized}).
The lower-bound hard-slice geometry localizes to the same accepted region, so the correspondence runs through the matched bound structure rather than a standalone argument.
A fixed-ESS separation theorem is left open.

\textbf{Contributions} (what the floor changes; how the cost splits across resources; what functional governs hardness; what the certificate honestly does on real shift).
\begin{itemize}
\setlength{\itemsep}{1pt}
\setlength{\topsep}{2pt}
\setlength{\parsep}{0pt}
\item \looseness=-1 \emph{(1) The Floor Certification Map.}
Theorem~\ref{thm:law} proves the two-axis map as three model-tagged bounds, with the Model-B$'$ nuisance budget explicitly priced ($\nb \lesssim \kappa s$).
The bounds match across models, and a single-model match over the full bounded-ratio class is provably impossible (Model-B inconsistency, Theorem~\ref{thm:modelB-incon}).
A within-B$'$ refinement sharpens the two resource axes (Corollary~\ref{cor:bprime-minimax}), and a necessary $K$-free core for the nuisance is established (Theorem~\ref{thm:t3-target}) [\S\ref{sec:theory}].
\item \emph{(2) The geometry.}
Experiments identify the localized accepted-region functional $\locfun{\Sacc}$, not global ESS, as the operative complexity proxy for B$'$ required-$n$ (family-1 Spearman $.936$ vs.\ $.026$; the restriction is part of the claim) [\S\ref{sec:synthetic}].
\item \emph{(3) Mechanism and validity at scale.}
In the synthetic study, the \emph{Registered-family} reproduces the map's central prediction: the labeled-sample requirement blows up near the certifiable boundary at a near-quadratic rate (the ``bite'' divergence, log--log slope $-2.002$ inside the pre-registered exponent band; near-quadratic across three structurally distinct families).
A full-grid validity audit records 0 violations and 0 Holm rejections where the formal certificate arms fire, against 108/40/662 for the non-formal demonstration arms (floor-free/weighted-conformal/plug-in), and characterizes the certification envelope including B$'$'s $K$-premium [\S\ref{sec:synthetic}].
\item \emph{(4) Real-workload feasibility audit (honest refusal).}
On a real-workload SQuAD$\to$NewsQA audit, the certificate returns honest refusal-with-attribution, consuming no evaluation labels, and a post-hoc diagnosis places the \emph{estimated} frontier below the operational floor across the tested plane [\S\ref{sec:real}].
\end{itemize}

\looseness=-1 Full proofs and extended experimental records are in Appendices~\ref{app:frontier}--\ref{app:necessity}.

\section{Related Work and Theorem-Level Comparison}
\label{sec:related}

\looseness=-1 We delimit our claim sharply against the strongest neighbors. What is new is the covariate-shift theory the floor creates, the $\beta$-indexed lower bounds, the two-resource map, the localized complexity, and the budgeted-nuisance analysis, none of which these lines produce: no lower bound indexed by a coverage floor, no $\beta$-indexed matching rates, no frontier and feasibility theory. The certificate's \emph{validity} half is not new, and we say so concretely: a floor-augmented e-value construction in the spirit of SCoRE \citep{bai2026score}, or BalanceRAG-like cascaded-calibration machinery \citep{jia2026balancerag}, recovers a U1-style certificate, and in the i.i.d.\ setting a joint finite-sample certificate for selection-conditioned risk and an acceptance floor is already available \citep{yu2026joint}; we therefore claim no novelty for the floor certificate itself, only for the four items above. Consistent with this scoping, our SCoRE-inspired betting arm (a risk-test ablation inside B$'$'s construction, not a source-faithful SCoRE implementation) is valid and even modestly \emph{more} powerful than Algorithm~\ref{alg:bprime} in observed certification frequency (0 violations, certify-frequency $\ge$ B$'$'s across all 1{,}024 audited cells at the \emph{same} powered-cell count; \S\ref{sec:validity}, Appendix~\ref{app:experiments}): a finite-sample power advantage for the certificate is neither claimed nor needed. A separate i.i.d.\ selective conformal risk-control line combines confident-sample selection with conformal risk control (CRC) under exchangeability/PAC-style guarantees \citep{xu2025selective}; it does not address bounded-ratio covariate shift, $\beta$-indexed lower bounds, or the labeled-source/unlabeled-target floor-certification map studied here.
Table~\ref{tab:theory_comparison} (expanded row by row in Appendix~\ref{app:nonderiv}) compares at theorem level.

\paragraph{Selection-conditioned and ratio risk control.}
\looseness=-1 The accepted-answer ratio risk $\E[L S]/\E[S]$ is now an active neighborhood (linear-expectation reformulation \citep{wang2025lec}, weighted e-values for general selective risks \citep{bai2026score}, cascaded lattice calibration \citep{jia2026balancerag}, risk-aware routing with abstention \citep{hao2026racer}), with the lineage running back to selective classification and classification with a reject option \citep{chow1970optimum, elyaniv2010foundations, bartlett2008classification, geifman2017selective, franc2023optimal} and conformal selection, whose false discovery rate is itself a selected-set ratio \citep{jin2023selection, jin2023weighted}.
These certify whatever coverage the policy attains rather than a certified floor under shift; the reject-option formulations \citep{franc2023optimal} do target a guaranteed-coverage (bounded-abstention) objective, but characterize the \emph{population} optimum, without finite-sample floor-indexed certificates or rates under covariate shift.

\paragraph{Risk control under covariate shift.}
\looseness=-1 Importance weighting under covariate shift \citep{shimodaira2000improving} and weighted conformal prediction, a covariate-shift extension of conformal prediction \citep{vovk2022algorithmic}, introduced the density-ratio reweighting we build on \citep{tibshirani2019conformal}, extended beyond exchangeability \citep{barber2023conformal} and to (high-probability) risk control under shift \citep{zecchin2025generalization, almeida2025high}, with direct ratio estimators and doubly robust, debiased, clipped, fine-grained, and kernel-matched weights making estimation practical \citep{kanamori2009least, yang2024doubly, kato2023double, wang2026weight, laghuvarapu2026kmmcp, ai2024finegrained}.
This family certifies \emph{marginal} risk or coverage, never the selected-set ratio with a floor, with global ESS-type complexity rather than our localized functional.

\paragraph{Lower bounds and feasibility for distribution-free certification.}
\looseness=-1 A lower-bound culture exists: matching minimax bounds for non-monotone conformal risk control \citep{aldirawi2026conformal}, tight abstention bounds under adversarial injections \citep{edelman2026reliable}, overlap obstructions \citep{damour2021overlap}, prediction sets adaptive to unknown shift \citep{qiu2023prediction}, and tolerant learning with abstention \citep{goel2024tolerant}.
All are floor-free: no $\beta$ in any rate, hence no frontier, no two-resource $(n,m)$ map. Our Claim-1 constructions live at bounded ratio, so they are not overlap results in disguise (Appendix~\ref{app:lower}).

\paragraph{Learnability and estimability under arbitrary covariate shift.}
\looseness=-1 Closest in spirit to our impossibility leg are two lines asking when shift is tractable at all: \emph{reject-to-learn}, where abstaining on unlabeled target points recovers efficient learnability under \emph{arbitrary} covariate shift through a reliable-learner oracle \citep{kalai21a}; and \emph{covariate-shifted mean estimation}, which recovers $\E_Q[f]$ for an unknown bounded $f$ from labeled source and unlabeled target samples under structural conditions \citep{adil26a}. Our Model-B inconsistency (Theorem~\ref{thm:modelB-incon}) is the \emph{certification} analogue of the estimability boundary these navigate: over the unrestricted bounded-ratio class no $(n, m)$-uniform certificate exists, and the pre-registered $K$-cell Model-B$'$ is the structural restriction that restores it, paralleling the structure those works impose. Neither subsumes floor-indexed certification (no coverage floor, no $\beta$-indexed rate, no risk--coverage frontier), but they locate why the floor is hard under shift.

\paragraph{Retrieval-augmented generation (RAG) and cascade deployment.}
\looseness=-1 Certified RAG generation risk \citep{kang2024crag}, conformal retrieval-augmented question answering \citep{li2024traq}, conformal factuality \citep{mohri2024language}, conformal abstention \citep{yadkori2024mitigating}, and chance-constrained hallucination control \citep{mohandas2026chance} supply the loss substrate and baselines for our real-workload audit; the proof that confidence-based deferral fails under shift \citep{jitkrittum2023when}, together with selective question answering under domain shift via a learned calibrator \citep{kamath2020selective}, supplies its motivation; none certify a coverage floor or price what the floor costs.

\begin{table}[t]
\centering
\caption{\textbf{Theorem-level comparison with the nearest verified lines}
(non-derivability analysis). $\checkmark$ = present, $(\checkmark)$ = partial,
$\times$ = absent. SCoRE controls the selection-conditioned ratio under covariate
shift with an estimated-weight robustness analysis but, as published, proves no lower bounds and
certifies no floor; LEC and BalanceRAG own ratio control in the exchangeable (i.i.d.)\ setting;
the weighted/marginal CRC-under-shift line owns the weight machinery for
\emph{marginal} risk/coverage, while \citet{almeida2025high} reaches a
selection-conditioned ratio under shift without a co-certified floor or lower bound;
non-monotone CRC \citep{aldirawi2026conformal} proves a matching
single-axis lower bound without a floor (nearest in lower-bound theory); no floor means no
second resource. This paper treats the hard coverage floor $C_Q \ge \beta$ as a
\emph{finite-sample} certified object under covariate shift (the reject-option
line poses the floor only as a population objective; \citealp{franc2023optimal}); its B$'$ upper is priced by an explicit nuisance $\rho_\lambda$
under a pre-registered stratified-shift model (nuisance necessity partially characterized:
$K$-free core necessary, histogram $K$ not proved necessary, unknown-$\eta$ open), and no
distribution-free unknown-weight matching claim is made.}
\label{tab:theory_comparison}
\small
\resizebox{0.96\textwidth}{!}{%
\begin{tabular}{lcccccc}
\toprule
& \begin{tabular}{@{}c@{}}Ratio risk\\under shift\end{tabular}
& \begin{tabular}{@{}c@{}}Hard floor as\\certified object\end{tabular}
& \begin{tabular}{@{}c@{}}$\beta$-indexed\\lower bound\end{tabular}
& \begin{tabular}{@{}c@{}}Two-resource\\$(n,m)$ law\end{tabular}
& \begin{tabular}{@{}c@{}}Localized\\functional\end{tabular}
& \begin{tabular}{@{}c@{}}Nuisance-priced\\upper\end{tabular} \\
\midrule
Reject-option classifiers \citep{franc2023optimal} & $\times$ & $(\checkmark)$\textsuperscript{g} & $\times$ & $\times$ & $\times$ & $\times$ \\
SCoRE \citep[\S6]{bai2026score} & $\checkmark$ & $\times$ & $\times$ & $\times$ & $\times$ & $(\checkmark)$\textsuperscript{a} \\
LEC \citep{wang2025lec} / BalanceRAG \citep{jia2026balancerag} & $\times$\textsuperscript{b} & $\times$ & $\times$ & $\times$ & $\times$ & $\times$ \\
Weighted/marginal CRC under shift\textsuperscript{c} & $\times$\textsuperscript{d} & $\times$ & $\times$ & $\times$ & $\times$ & $(\checkmark)$\textsuperscript{d} \\
High-probability risk control under shift \citep{almeida2025high} & $\checkmark$\textsuperscript{i} & $\times$ & $\times$ & $\times$ & $\times$ & $(\checkmark)$\textsuperscript{i} \\
Non-monotone CRC \citep{aldirawi2026conformal} & $\times$ & $\times$ & $\times$\textsuperscript{e} & $\times$ & $\times$ & $\times$ \\
\textbf{This paper} & $\checkmark$ & $\checkmark$ & $\checkmark$ & $\checkmark$ & $\checkmark$\textsuperscript{h} & $\checkmark$\textsuperscript{f} \\
\bottomrule
\end{tabular}}

\vspace{0.25em}
{\footnotesize
\textsuperscript{a}Estimated-$w$ robustness analysis only; no rate theory.
\textsuperscript{b}Ratio control in the exchangeable (i.i.d.)\ setting; no shift
($\alpha$-infeasibility declaration $\neq$ coverage floor).
\textsuperscript{c}\citet{tibshirani2019conformal}; weighted CRC under shift
\citep{zecchin2025generalization}.
\textsuperscript{d}Marginal estimand (not the selection-conditioned ratio);
estimated-$w$ analyses exist for the marginal object; as analyzed in those works,
complexity is reported via global ESS-type quantities rather than a localized
accepted-region functional.
\textsuperscript{e}Proves a matching minimax lower bound (single-axis excess risk),
but it is not indexed by a coverage floor: no coverage constraint and no
coverage-floor parameter.
\textsuperscript{f}Model-B$'$ upper with explicitly priced nuisance
$\rho_\lambda \lesssim \kappa s$ under the pre-registered stratified-shift model;
nuisance necessity partially characterized: $K$-free core necessary, histogram $K$ not
proved necessary, unknown-$\eta$ open.
\textsuperscript{g}Bounded-abstention model: guaranteed coverage with minimal
selective risk, a \emph{population} risk--coverage frontier, not a finite-sample
certificate under covariate shift, and no $\beta$-indexed rate.
\textsuperscript{h}A localized accepted-region second moment governs the upper-bound
variance and the lower-bound hard slices (Remark~\ref{rmk:localized}); its normalized
effective-sample-size reading is heuristic, and a separation theorem at fixed global
ESS is left open.
\textsuperscript{i}Theorem 9 of \citet{almeida2025high} gives high-probability
control under target covariate shift for a \emph{conditional} risk
$\E_{Q}[L(X, Y; \lambda) \mid (X, Y) \in A(\lambda)]$ and instantiates it for the
false discovery rate; taking $A(\lambda)$ to be the accepted set yields exactly the
selection-conditioned ratio $\E[LS_\lambda]/\E[S_\lambda]$, so this capability is
present. What is absent is the rest of the row: no floor is co-certified, no
$\beta$-indexed lower bound, no two-resource law, and complexity is not reported
through a localized accepted-region functional. Its estimated-weight treatment
proceeds under an assumed weight accuracy rather than a finite-sample
weight-estimation-error guarantee, so the nuisance is not priced as an explicit
increment against a matching lower bound.}
\end{table}

\paragraph{What is new here, exactly.}
\looseness=-2 The novelty is not the coverage floor as a population object (the reject-option line already studies guaranteed-coverage selective risk \citep{franc2023optimal}) but the \emph{finite-sample} theory the floor creates under covariate shift: $\beta$-indexed lower bounds, the two-resource (labeled/unlabeled) map, the localized accepted-region complexity, and the B$'$ budgeted-nuisance analysis, \emph{not} weighted testing, ratio linearization, or cascaded-RAG calibration, which we inherit and credit \citep{bates2021distribution, angelopoulos2021learn, angelopoulos2024conformal, laufergoldshtein2023efficiently, jia2026balancerag, wang2025lec, tsybakov2009introduction}.

\section{Problem, Setup, and the Certificate}
\label{sec:setup}

\subsection{Problem, notation, and the certify-or-refuse output}
\looseness=-1 Labeled pairs $(X, Y) \sim P$ carry a bounded loss $L \in [0,1]$ with conditional mean $\eta(x) = \E[L \mid X = x]$; the target supplies unlabeled covariates $X' \sim Q_X$.
We assume covariate shift (the conditional map of $(Y, \text{outputs}, L)$ given $X$ is identical under source and target) and a density ratio $w = dQ_X/dP_X \le B$; write $\cW_B$ for this class.
Since $\int w \, dP_X = 1$, membership forces $P_X(w > 0) \ge 1/B$: positivity holds throughout; no construction here uses unsupported target mass.
A selective policy is a point $\lambda$ on a pre-registered nested-threshold lattice $\lat$ over a frozen router score, with acceptance indicator $\Sacc \in \{0,1\}$ (abstention and escalation mean $\Sacc = 0$). The certified functionals are the selection-conditioned risk $\RQ(\lambda) = \E_{Q_X}[\eta\, \Sacc]/\E_{Q_X}[\Sacc]$ and the coverage $\CQ(\lambda) = \E_{Q_X}[\Sacc]$. We adopt the convention $\RQ(\lambda) := +\infty$ when $\CQ(\lambda) = 0$, so a zero-coverage policy is never certifiable. Reject-all is thus infeasible and the frontier suprema below range over $\{\CQ > 0\}$; escalate-all is a no-automatic-answer guarantee, not a success.
\begin{definition}[Procedure, validity, and power]
\label{def:valid}
From $n$ i.i.d.\ labeled draws from $P$ and $m$ i.i.d.\ unlabeled draws from $Q_X$ (independent), a \emph{procedure} $\cA$ outputs a point of $\lat$ or an honest refusal $\nocert$. It is \emph{valid at $(\alpha, \beta, \delta)$} if in every world $\Prob\big(\cA \ne \nocert \wedge (\RQ(\cA) > \alpha \vee \CQ(\cA) < \beta)\big) \le \delta$, and \emph{powered at a world} if it certifies there with probability $\ge 1/2$.
\end{definition}

\looseness=-1 Table~\ref{tab:notation} collects the notation used throughout the paper.
\begin{table}[t]
\centering
\caption{\textbf{Notation}, grouped by role, with the location of first use
(models A/B/B$'$ are defined in \S\ref{sec:models}). Defining formulas are given
in the cited statements.}
\label{tab:notation}
\small
\begin{tabular}{@{}l l l@{}}
\toprule
Symbol & Meaning & First use \\
\midrule
\multicolumn{3}{@{}l}{\emph{Data, loss, and covariate shift}}\\
$P,\ Q_X$ & source distribution; shifted target covariate distribution & \S\ref{sec:setup}\\
$(X,Y),\ L$ & covariate--label pair; bounded loss $L\in[0,1]$ & \S\ref{sec:setup}\\
$\eta$ & conditional loss mean $\eta(x)=\E[L\mid X=x]$ & \S\ref{sec:setup}\\
$w,\ B$ & density ratio $w=dQ_X/dP_X$, bounded $w\le B$ & \S\ref{sec:setup}\\
$\cW_B$ & bounded-ratio covariate-shift class & \S\ref{sec:setup}\\
$n,\ m$ & labeled-source / unlabeled-target sample sizes & Def.~\ref{def:valid}\\
\addlinespace
\multicolumn{3}{@{}l}{\emph{Policy and certified functionals}}\\
$\lambda,\ \lat$ & selective policy; pre-registered threshold lattice & \S\ref{sec:setup}\\
$\Sacc$ & acceptance indicator $S_\lambda\in\{0,1\}$ ($0$: abstain/escalate) & \S\ref{sec:setup}\\
$\RQ$ & selection-conditioned risk $\E_{Q_X}[\eta\Sacc]/\E_{Q_X}[\Sacc]$ & \S\ref{sec:setup}\\
$\CQ$ & target coverage (the floor functional) $\E_{Q_X}[\Sacc]$ & \S\ref{sec:setup}\\
$\GQ$ & linearized risk functional $\E_P[w\Sacc(L-\alpha)]$ & \S\ref{sec:theory}\\
$\alpha,\ \beta,\ \delta$ & risk level; coverage floor; confidence & \S\ref{sec:setup}\\
$\nocert$ & honest refusal (the certify-or-refuse output) & \S\ref{sec:setup}\\
\addlinespace
\multicolumn{3}{@{}l}{\emph{Frontier, slack, and margins}}\\
$\budget$ & budget functional $\budget(c)=\alpha c-\int_0^c\eta^*_Q(u)\,du$ & Def.~\ref{def:frontier}\\
$\eta^*_Q$ & increasing rearrangement of $\eta$ under $Q_X$ & Def.~\ref{def:frontier}\\
$\betastar,\ \betastar_{\lat}$ & relaxed frontier; lattice frontier ($\betastar_{\lat}\le\betastar$) & Def.~\ref{def:frontier}\\
$s$ & floor slack $s=\betastar-\beta$ (local regime $s\le s_0\wedge c_0\beta$) & Def.~\ref{def:regmargin}\\
$\kappa,\ s_0$ & regular frontier margin: height; radius & Def.~\ref{def:regmargin}\\
$\mathrm{LR}_{\mathrm{margin}}(s)$ & lattice margin condition (assumed for the upper bounds) & Def.~\ref{def:lrmargin}\\
\addlinespace
\multicolumn{3}{@{}l}{\emph{Information models and stratified shift (Model B$'$)}}\\
A / B / B$'$ & oracle-$w$ / unknown-$w$ / $K$-cell estimable-$w$ models & \S\ref{sec:models}\\
$\cK,\ K$ & pre-registered partition $\{\cX_1,\dots,\cX_K\}$ ($w$ cell-constant) & \S\ref{sec:setup}\\
$w_k,\ \hat w_k$ & cell ratio $q_k/p_k$ ($q_k=Q_X(\cX_k),\,p_k=P_X(\cX_k)$); estimate & \S\ref{sec:setup}\\
$\nb$ & localized nuisance budget (accepted-cell weight error) & Alg.~\ref{alg:bprime}\\
$D_w^P,D_r;\,D_w^Q,D_f$ & source / target weight- and test-splits & Alg.~\ref{alg:bprime}\\
\addlinespace
\multicolumn{3}{@{}l}{\emph{Complexity proxy}}\\
$\E_P[w^2\Sacc]$ & localized accepted-region second moment (vs.\ global ESS) & Rmk.~\ref{rmk:localized}\\
\bottomrule
\end{tabular}
\end{table}

\subsection{The feasibility frontier and floor slack}
\begin{definition}[Budget functional and frontier]
\label{def:frontier}
Let $\eta^*_Q(u)$, $u \in [0,1]$, be the increasing rearrangement of $\eta$ under $Q_X$; define the \emph{budget functional} $\budget(c) = \alpha c - \int_0^c \eta^*_Q(u)\, du$ and the \emph{frontier} $\betastar(\alpha, Q) = \sup\{c : \budget(c) \ge 0\}$, the largest certifiable coverage under randomized acceptance (a fractional-knapsack value attained by the increasing-rearrangement threshold rule, the generalized Neyman--Pearson solution \citep{dantzig1951fundamental}; for atomic $Q_X$ the deterministic-selector frontier may be strictly smaller, Appendix~\ref{app:frontier}).
\end{definition}
\begin{figure}[t]
\centering
\includegraphics[width=0.95\textwidth]{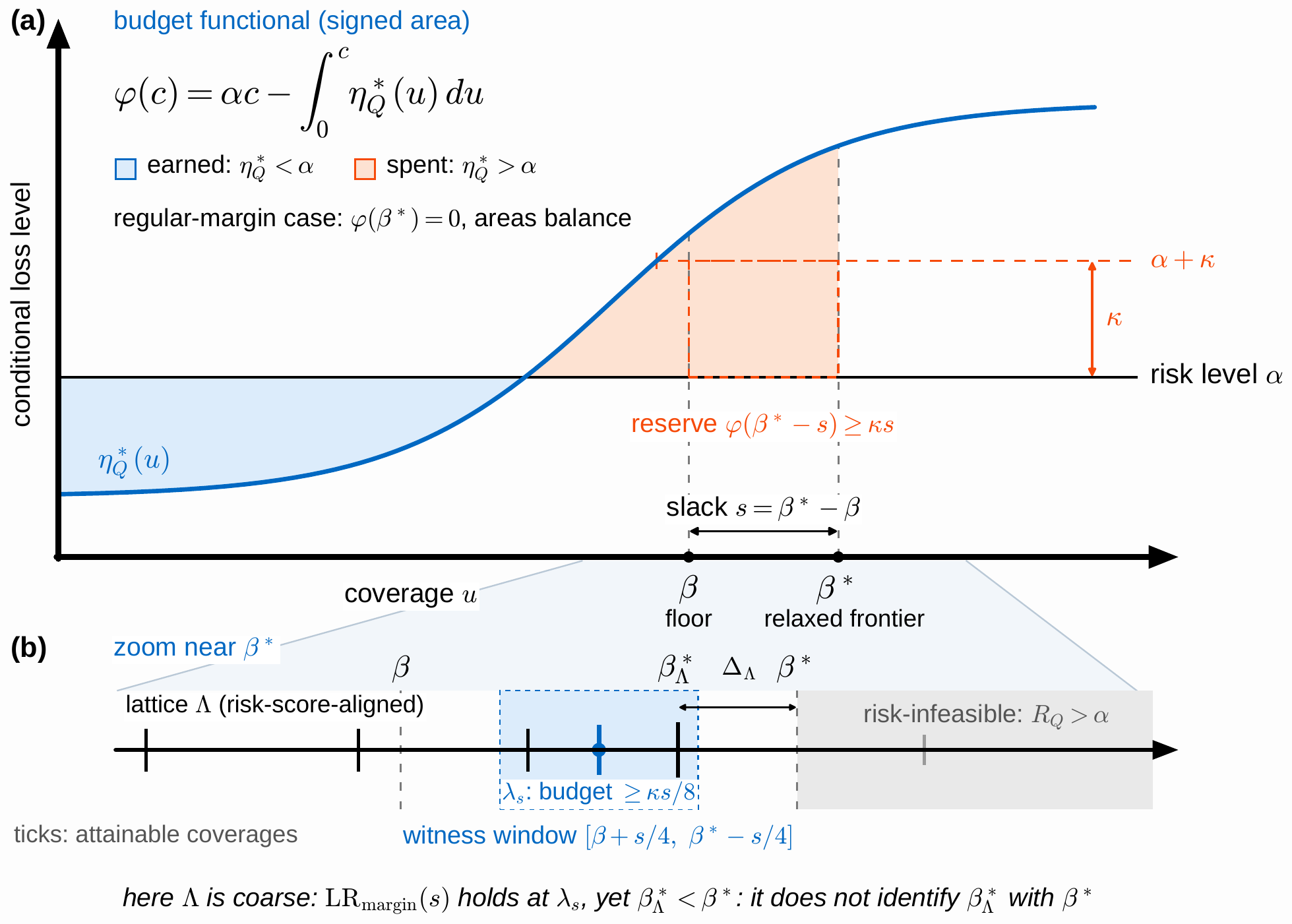}
\caption{\textbf{Frontier geometry and the relaxed-vs-lattice distinction.}
(a)~The budget functional $\budget(c) = \alpha c - \int_0^c \eta^*_Q(u)\,du$ is the signed area between the risk level $\alpha$ and the increasing rearrangement $\eta^*_Q$: earned where $\eta^*_Q < \alpha$, spent where $\eta^*_Q > \alpha$. The relaxed frontier is $\betastar = \sup\{c : \budget(c) \ge 0\}$ (Definition~\ref{def:frontier}); in the depicted regular-margin case the areas balance there, $\budget(\betastar) = 0$. The regular frontier margin (Definition~\ref{def:regmargin}) keeps $\eta^*_Q \ge \alpha + \kappa$ a.e.\ on $[\betastar - s_0, \betastar]$, so a floor $\beta = \betastar - s$ in the local regime retains the linear reserve $\budget(\betastar - s) \ge \kappa s$ (dashed rectangle), the single inequality behind both the bite divergence and the achievability margins.
(b)~A magnified view of the frontier neighborhood on the same coverage axis, with a risk-score-aligned finite lattice $\lat$ (ticks: attainable coverages; under alignment, risk feasibility is exactly $\CQ(\lambda) \le \betastar$). The operative lattice frontier $\betastar_{\lat}$ is the largest risk-feasible attainable coverage; a grid whose last feasible coverage stops short of $\betastar$ leaves the discretization gap $\Delta_{\lat} = \betastar - \betastar_{\lat} > 0$. The condition $\mathrm{LR}_{\mathrm{margin}}(s)$ (Definition~\ref{def:lrmargin}) asks only for an interior witness $\lambda_s$ with coverage in $[\beta + s/4,\, \betastar - s/4]$ and risk budget at least $\kappa s/8$, so it can hold on a coarse lattice with $\betastar_{\lat} < \betastar$: it does not identify the two frontiers. Coverages beyond $\betastar$ are risk-infeasible for every policy. Drawn values are schematic; the population-infeasibility statements quote $\betastar_{\lat}$, the power theorems measure slack against the relaxed $\betastar$ and touch the lattice only through the witness $\lambda_s$, and the lower bound supplies its own compatible lattice per $s$ rather than fixing one (Theorem~\ref{thm:law}).}
\label{fig:frontier_geometry}
\end{figure}

\looseness=-1 On a finite lattice the operative object is the \emph{lattice frontier} $\betastar_{\lat} = \sup\{\CQ(\lambda) : \lambda \in \lat,\, \RQ(\lambda) \le \alpha\} \le \betastar$. We keep the two symbols distinct ($\betastar$ for the relaxed frontier, $\betastar_{\lat}$ for the operative lattice object; Figure~\ref{fig:frontier_geometry}b) and write $\Delta_{\lat} := \betastar - \betastar_{\lat} \ge 0$ for the \emph{discretization gap}, flagging any coarse-lattice construction where it is nonzero (Appendix~\ref{app:frontier}). The map's slack is the relaxed $s = \betastar - \beta$. We do \emph{not} claim a bound uniform over a slack range on a single fixed lattice: the lower bound is existential in $\lat$ (a compatible lattice supplied per $s$; Theorem~\ref{thm:lower}, Corollary~\ref{cor:bprime-minimax}), while a deployed fixed $\lat$ certifies the operative slack $\betastar_{\lat} - \beta = s - \Delta_{\lat}$. The inconsistency theorem is not existential in this sense: it fixes two absolute thresholds and holds on every lattice containing them (Theorem~\ref{thm:modelB-incon}, step R4).

\looseness=-1 The population-infeasibility statements are for the lattice frontier $\betastar_{\lat}$, while the power theorems measure slack against the relaxed frontier, $s = \betastar - \beta$, and touch the finite lattice only through the explicit $\mathrm{LR}_{\mathrm{margin}}(s)$ witness $\lambda_s$ below.

\looseness=-1 The condition $\mathrm{LR}_{\mathrm{margin}}$ should \emph{not} be read as implying $\betastar = \betastar_{\lat}$. The two coincide up to grid resolution only under a sufficient lattice-design condition: a risk-score-aligned lattice whose coverage mesh (the largest gap between attainable coverages) is $< s/8$ \emph{and} whose attainable coverages extend to within that mesh of the frontier $\betastar$ itself. Fineness alone does not suffice, nor does reaching only the interior point $\betastar - s/2$: a grid stopping short of $\betastar$ leaves $\betastar - \betastar_{\lat}$ as large as the terminal gap. By contrast, $\mathrm{LR}_{\mathrm{margin}}(s)$ alone needs only an interior witness near $\betastar - s/2$, so it can hold on a coarse lattice with $\betastar_{\lat} < \betastar$ (Appendix~\ref{app:frontier}).
\begin{definition}[Regular frontier margin]
\label{def:regmargin}
A world has a \emph{regular frontier margin} $(\kappa, s_0)$ if $\budget(\betastar) = 0$, $\kappa > 0$, $0 < s_0 \le \betastar$, and $\eta^*_Q(u) \ge \alpha + \kappa$ for a.e.\ $u \in [\betastar - s_0, \betastar]$.
\end{definition}
\looseness=-1 This margin is generic. With $\betastar < 1$, it fails only in the degenerate tie where $\budget \equiv 0$ on $[0, \betastar]$ (equivalently $\eta^*_Q = \alpha$ a.e.\ there), so any world with $\eta^*_Q < \alpha$ on a set of positive measure has the margin (Lemma~\ref{lem:genericity}). Under the margin, with floor slack $s = \betastar(\alpha, Q) - \beta \in (0,\, s_0 \wedge c_0 \beta]$ (fixed $c_0 \in (0, 1]$, so the local regime has $s \le \beta$), the risk budget is \emph{linear}: $\budget(\betastar - s) \ge \kappa s$ (Figure~\ref{fig:frontier_geometry}a). This single inequality drives both the bite divergence and the achievability margins (Appendix~\ref{app:frontier}, eq.~\ref{eq:budgetlaw_app}).
Genericity is \emph{qualitative}; it asserts a positive margin exists, not that it is large: $\kappa$ and $s_0$ may be arbitrarily small, and since the certificate's sample complexity scales as $\kappa^{-2}$ (Theorem~\ref{thm:modelA}), a small frontier margin makes certification expensive without violating any assumption here.
\begin{definition}[Lattice margin condition $\mathrm{LR}_{\mathrm{margin}}(s)$]
\label{def:lrmargin}
On a finite lattice, $\mathrm{LR}_{\mathrm{margin}}(s)$ holds if some $\lambda_s \in \lat$ has coverage in $[\beta + s/4,\, \betastar - s/4]$ \emph{and} risk budget $\E_Q[\Sacc(\alpha - \eta)] \ge \kappa s/8$.
\end{definition}
\looseness=-1 We additionally assume $\mathrm{LR}_{\mathrm{margin}}(s)$ on a finite lattice for the upper bounds; it is automatic for risk-score-aligned nested-threshold lattices whose coverage grid is finer than $s/8$ and brackets $\betastar - s/2$ (Appendix~\ref{app:frontier}), while on a coarse lattice the certifiable slack plateaus within cells, and zero-margin cells drive the required $n$ to infinity, a lattice-design diagnostic we also observe empirically (Appendices~\ref{app:exponent}, \ref{app:latticediag}).

\subsection{Information models: A, B, and B$'$}
\looseness=-1 \label{sec:models}
The map is stated against a three-way model split, introduced once:
\textbf{A} (oracle weights) receives $w$ exactly: clean upper bound, rate reference;
\textbf{B} (unknown weights) knows only $w \le B$: the lower-bound model, in which the floor is estimable label-free while the risk is not;
\textbf{B$'$} (estimable-ratio structure) knows $w \le B$ plus a pre-registered $K$-cell partition on which $w$ is constant: the implementable upper bound, with an explicit nuisance price.

\subsection{Stratified-shift assumptions and data splits}
Model-B$'$ assumes, on top of $w \le B$, a \emph{shift model}: a pre-registered finite partition $\cK = \{\cX_1, \dots, \cX_K\}$ of $\cX$ on which $w$ is constant: stratified covariate shift; the $K$ numbers $w_k = q_k/p_k$ are learned from finite samples at the explicit nuisance price below ($K$ is a dial, not a disguise: Remark~\ref{rmk:dial}).

\begin{assumption}[Stratified-shift certification setting]
\label{ass:bprime}
\looseness=-1 \emph{Conditions (i), (iii), (iv) are structural (class-defining); (ii) is an algorithm-feasibility requirement on the run's labeled split size $n_w$, not a property of the world.} (i) $\cK$ is fixed before any data used by the run; (ii) $p_k = P_X(\cX_k) \ge p_{\min} \ge 8\log(4K/\delta_w)/n_w$ for all $k$, so $K$ may grow only up to order $n_w/\log$; (iii) each acceptance region $\{\Sacc = 1\}$ is a union of cells ($K$ counts score-refined cells); (iv) $L \in [0,1]$.
\end{assumption}

\subsection{Algorithm: the budgeted floor-aware certificate}
\looseness=-1 Algorithm~\ref{alg:bprime} pairs an upper confidence bound (UCB) on the linearized risk with a lower confidence bound (LCB) on coverage, and certifies only when both clear their margins.
\begin{algorithm}[!ht]
\caption{Budgeted floor-aware certificate (Model-B$'$)}
\label{alg:bprime}
\footnotesize
\begin{algorithmic}[1]
\STATE \textbf{Split:} source $\to (D_w^P, D_r)$; \quad target $\to (D_w^Q, D_f)$ \hfill // four-block splitting
\STATE $\hat w_k \gets \clip(\hat q_k / \hat p_k,\, 0,\, B)$ on $D_w$ \ ($\hat w_k := B$ if $\hat p_k = 0$) \hfill // $K$-cell histogram ratio
\STATE $L_w \gets \log\frac{4(|\lat|+1)}{\delta_w}$ \hfill // simultaneity constant (loop-invariant)
\FOR{$\lambda \in \lat$, one shared budget $\delta = \delta_w + |\lat|\delta_r' + |\lat|\delta_f'$}
    \STATE $\nb \gets$ localized nuisance budget over $\lambda$'s accepted-cell masses (Proposition~\ref{prop:I1}, via $L_w$; plug-in corrected)
    \STATE risk-pass \ iff \ $\UCB_{D_r}\!\big(\hat w \cdot \Sacc \cdot (L - \alpha)\big) \le -\nb$ \hfill // budgeted linear test (empirical Bernstein)
    \STATE floor-pass iff \ $\LCB_{D_f}\!\big(\Sacc\big) \ge \beta$ \ and $> 0$ \hfill // \textbf{weight-free} floor (Bernstein)
\ENDFOR
\STATE \textbf{Output:} the floor-LCB-maximal $\lambda$ passing both (ties broken by larger risk-UCB slack, then earliest lattice index); \ \textbf{else} $\nocert$ with observable test-side attribution (priority and optional diagnoses in \S\ref{sec:notcertify})
\end{algorithmic}
\end{algorithm}

\looseness=-2 Algorithm~\ref{alg:bprime} linearizes the ratio constraint at fixed $\alpha$ \citep{wang2025lec} and tests it with a Maurer--Pontil empirical-Bernstein UCB \citep{maurer2009empirical} at margin $-\nb$, which absorbs the weight-estimation error; the floor LCB never touches weights, and selection among simultaneously valid tests is free in the learn-then-test pattern \citep{angelopoulos2021learn, laufergoldshtein2023efficiently}. The shared confidence budget allocates $\delta_w$ to weight estimation and $\delta_r', \delta_f'$ to each of the $|\lat|$ per-$\lambda$ risk and floor tests.
The displayed budget $\nb$ is the radius the power analysis is proved for; the released implementation computes a different registered closed form, valid at the same $\delta_w$ and used for every reported B$'$ certification frequency and by the SCoRE arm that shares the radius, stated and proved in Remark~\ref{rmk:releasedradius}.

\section{The Floor Certification Map: Guarantees and Limits}
\label{sec:theory}

\subsection{Main result: the map theorem}
\looseness=-1 \emph{The map is a model-split statement, not a single-model law.} Because Model-B is inconsistent over the full unknown-weight class $\cW_B$ (Theorem~\ref{thm:modelB-incon}), the three rates match \emph{across} models (a Model-A oracle-weight upper, the Model-B lower, and a \emph{restricted} Model-B$'$ implementable upper) rather than as a single-model minimax law. We state the three claims first; their precise scope (where the match is cross-model versus within-Model-B$'$, and the exact status of the nuisance necessity) is collected in Remark~\ref{rmk:scope} immediately after the theorem. The three claims refine the headline law~\eqref{eq:law}; per-claim model (Models A/B/B$'$ of \S\ref{sec:models}), assumptions, and status are summarized in Table~\ref{tab:claims}, and Figure~\ref{fig:map_structure} draws the architecture of the match.

\begin{theorem}[The Floor Certification Map]
\label{thm:law}
\looseness=-1 \emph{Standing assumptions:} fix risk level $\alpha \in [0,1]$, confidence $\delta \le 1/8$, bounded ratio $w \le B$, and a regular frontier margin $(\kappa, s_0)$; let the floor $\beta$ have slack $s = \betastar(\alpha, Q) - \beta \in (0,\, s_0 \wedge c_0\beta]$ (the local regime). The upper bounds (Claims 2--3) additionally fix a nested-threshold lattice $\lat$ with $\mathrm{LR}_{\mathrm{margin}}(s)$; the lower bound (Claim 1) is \emph{existential} in the lattice: for each $s$ the hard instance supplies its own compatible $\lat$ (Theorem~\ref{thm:lower}), so the match is not uniform over a single fixed $\lat$ (Corollary~\ref{cor:bprime-minimax}). The Model-B inconsistency theorem is not existential in this sense: it holds on every lattice containing two fixed absolute thresholds (Theorem~\ref{thm:modelB-incon}).

\textbf{Claim 1: lower (Model-B: unknown weights, only $w \le B$).}
On the parameter rectangle $B \ge 2$, $\alpha \in [0.2, 0.9]$, $\beta \in (0, 0.6]$, $\kappa < \tfrac12\min\{\alpha, 1-\alpha\}$ of Theorem~\ref{thm:lower} (the non-vacuous regime; the lower-bound constants are uniform over $B \ge 2$, since a larger ratio bound only enlarges $\cW_B$ and strengthens the impossibility), no valid procedure can certify every slack-$s$ world once
$s \,\lesssim\, \kappa^{-1}\sqrt{\beta/n} + \sqrt{\beta/m}$:
two necessary conditions witnessed inside one construction family per $(n, m)$, whose max form yields the additive display up to a factor of $2$ (each axis separately on the general class).

\textbf{Claim 2: upper (Model-A: oracle weights).}
A floor-aware budgeted certificate (margin-bearing lattice) certifies every slack-$s$ world once
$n \,\gtrsim\, \beta \log(|\lat|/\delta)/(\kappa s)^2$ and $m \,\gtrsim\, \beta \log(|\lat|/\delta)/s^2$:
the same two axes (at fixed bounded-ratio constant $B$, folded into the constants; the explicit $B$-dependent variance and range terms are in Theorem~\ref{thm:modelA}); the A-upper attains the B-lower's rates up to constants and logarithms (a cross-model match).

\textbf{Claim 3: upper (Model-B$'$: estimated weights under a pre-registered $K$-cell stratified shift).}
The same certificate with a clipped histogram ratio and localized nuisance budget $\nb$ certifies once additionally $\nb \lesssim \kappa s$: a sufficient nuisance condition, not a proved-necessary axis (schematically $n_w \wedge m_w \gtrsim B^2 K \cdot \mathrm{mass}(A_\lambda)/(\kappa s)^2$, with the precise per-split $m_w$ and $n_w$ conditions given by Theorem~\ref{thm:U2}).
Validity is conditional on the stratified-shift model, and the nuisance's necessity is only \emph{partially} settled: a $K$-free core necessary, the histogram $B^2 K$ rate only sufficient, the unknown-$\eta$ case open (Remark~\ref{rmk:scope}, Theorem~\ref{thm:t3-target}, Appendix~\ref{app:necessity}).
\end{theorem}

\begin{figure}[t]
\centering
\includegraphics[width=\textwidth]{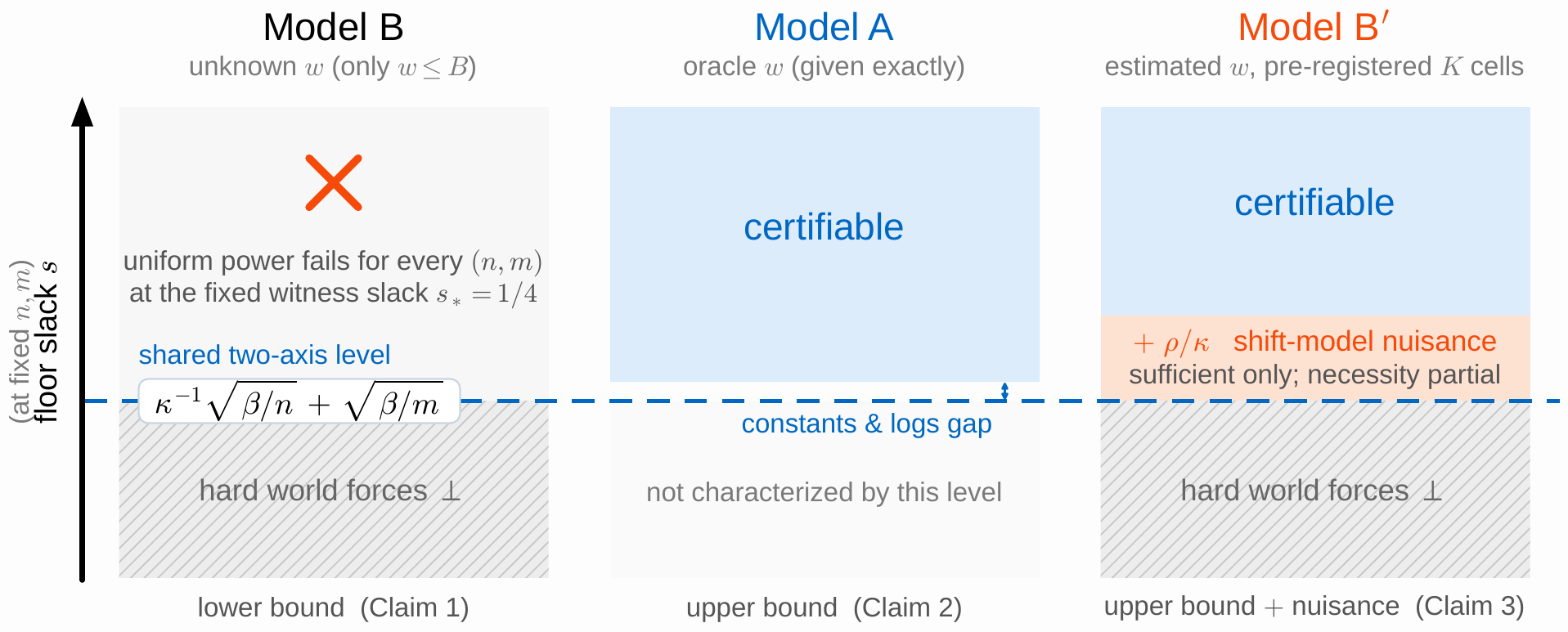}
\caption{\textbf{The Floor Certification Map as a level diagram (Theorem~\ref{thm:law}).}
Vertical axis: floor slack $s$ at fixed $(n, m)$; the dashed level is the two-axis rate $\kappa^{-1}\sqrt{\beta/n} + \sqrt{\beta/m}$ of the template~\eqref{eq:law}, and hatching marks regions with no uniform guarantee: some hard world is refused with probability at least $\tfrac{1}{2}$ there. \emph{Model B} (unknown weights, only $w \le B$): below the level refusal is forced on a hard world (Claim 1, Theorem~\ref{thm:lower}; one hard construction family per $(n, m)$, existential in the lattice with a compatible $\lat$ per $s$); moreover uniform power fails already at the fixed constant witness slack $s_* = \tfrac{1}{4}$ for every $(n, m)$, since Model B is inconsistent over $\cW_B$ at the operating point $\alpha = \beta = \tfrac{1}{2}$ ($B \ge 2$, $\delta \le 1/8$, on every finite nested-threshold lattice containing the theorem's two fixed thresholds; Theorem~\ref{thm:modelB-incon}), so no shrinking uniform certifiable threshold exists and no certifiable region is asserted for this column. \emph{Model A} (oracle weights): the budgeted certificate certifies above the same level up to constants and logarithms (Claim 2, Theorem~\ref{thm:modelA}); this shared level is the cross-model match, while no matching Model-A floor-axis lower bound is claimed and the below-level region is left uncharacterized (Remark~\ref{rmk:msplit}). \emph{Model B$'$} (estimated weights on a pre-registered $K$-cell partition): certifies above the level plus the explicitly priced nuisance increment $\nb/\kappa$ (Claim 3, Theorems~\ref{thm:U1}--\ref{thm:U2}); the lower bound re-runs within B$'$ under a compatible partition and lattice pre-registered for the target $s$ (not uniform in $s$ or over arbitrary partitions), and together with the upper bound it yields a correspondence on the totals when the weight splits do not dominate, pointwise and instance-dependent rather than a uniform minimax (Corollary~\ref{cor:bprime-minimax}); nuisance necessity is only partial: $K$-free core necessary, histogram $B^2 K$ rate only sufficient, unknown-$\eta$ open (Remark~\ref{rmk:scope}, Theorem~\ref{thm:t3-target}). The map is therefore necessarily model-split, and the cross-model match is the correct object rather than a substitute for an in-class theorem. Per-claim assumptions and precise scope are collected in Table~\ref{tab:claims}.}
\label{fig:map_structure}
\end{figure}

\begin{remark}[Scope of the map]
\label{rmk:scope}
\looseness=-1 \emph{Cross-model vs.\ within-model.} The match is \emph{across} models (Model-A upper vs.\ Model-B lower) over the full class. On the two resource axes it additionally holds \emph{within} Model-B$'$, as an instance-dependent (pointwise-on-totals) rate correspondence under a compatible pre-registered partition, not a uniform minimax over the class, and only a componentwise correspondence on the test sub-blocks $(n_r, m_f)$ (Remark~\ref{rmk:bprime-match}, Corollary~\ref{cor:bprime-minimax}), leaving only the histogram nuisance axis cross-model.

\looseness=-1 \emph{Nuisance necessity (partial).} The B$'$ nuisance is the histogram estimator's price, \emph{not} a proved-necessary axis. Feasibility is $p$-independent: the linearized risk $\GQ(\lambda) = \E_P[w \Sacc(L - \alpha)]$ equals $\sum_{k:\,\cX_k \subseteq A_\lambda} q_k(\bar\eta_k - \alpha)$ under the partition ($\bar\eta_k = \E_Q[\eta \mid X \in \cX_k]$), a sum over \emph{accepted} cells with no separate $p_k$, so no lower bound forces a separate \emph{population} source-covariate ($p$-estimation) split. The necessary core is $K$-free (lattice-valued certification, the standing output model of \S\ref{sec:setup}): a target $\Omega(q(A)/s^2)$, the floor $m$-axis localized, proved in the \emph{known-$\eta$} submodel of Theorem~\ref{thm:t3-target}, plus the labeled $n$-axis \emph{varying-$\eta$} reduction. The histogram $B^2 K$ rate is only sufficient; only the unknown-$\eta$ case stays open, with no complete minimax characterization claimed (Appendix~\ref{app:necessity}).
\end{remark}

\begin{table}[t]
\centering
\caption{The three claims of the Floor Certification Map (Theorem~\ref{thm:law}) by information model. \emph{Standing assumptions} (all claims): confidence $\delta \le 1/8$, bounded ratio $w \le B$, a nested-threshold lattice with $\mathrm{LR}_{\mathrm{margin}}$ (pre-registered for the upper bounds; constructed per $s$ for the lower), a regular frontier margin $(\kappa, s_0)$, and floor slack in the local regime $s = \betastar - \beta \in (0,\, s_0 \wedge c_0\beta]$. The match is \emph{across} models (B-lower vs.\ A-upper) over the full class, and additionally \emph{within} Model-B$'$ on the two resource axes (instance-dependent, not a uniform minimax over the class; the nuisance axis stays cross-model); the precise scope is Remark~\ref{rmk:scope} (formally Corollary~\ref{cor:bprime-minimax}). The Model-A side is shown achievable but a matching Model-A floor-axis lower bound is not claimed (Remark~\ref{rmk:msplit}).}
\label{tab:claims}
\footnotesize
\begin{tabular}{@{}p{0.10\textwidth} p{0.16\textwidth} p{0.22\textwidth} p{0.20\textwidth} p{0.18\textwidth}@{}}
\toprule
Claim & Model (weights) & Adds beyond standing & Proves & Status / open gap \\
\midrule
1 (lower) & B: unknown, only $w \le B$ & rectangle $B \ge 2$, $\alpha \in [.2,.9]$, $\beta \in (0,.6]$, $\kappa < \tfrac12\min\{\alpha,1-\alpha\}$ & two-axis impossibility $s \lesssim \kappa^{-1}\sqrt{\beta/n} + \sqrt{\beta/m}$ & no matching Model-B upper (Thm~\ref{thm:modelB-incon}) \\
2 (upper) & A: oracle $w$ known & $w$ supplied exactly & certificate attains Claim-1 rates (constants, logs) & match is cross-model only (Thm~\ref{thm:modelB-incon}) \\
3 (upper) & B$'$: $K$-cell stratified & pre-registered $\cK$; union-of-cells \& $p_{\min}$ (Algorithm~\ref{alg:bprime} feasibility) & implementable certificate; sufficient nuisance condition $\nb \lesssim \kappa s$ & necessity \emph{partial} (Thm~\ref{thm:t3-target}): $K$-free core necessary (lattice-valued, known-$\eta$ submodel), histogram $B^2 K$ not proved necessary, unknown-$\eta$ open; validity cond.\ on $\cK$ \\
\bottomrule
\end{tabular}
\end{table}

\paragraph{Geometry.}
\looseness=-1 The upper-bound variance proxy and the lower-bound hard-slice construction both localize to the accepted region and point to the localized second-moment family $\E_P[w^2 \Sacc]$ (an accepted-region quantity, not global ESS) as the operative complexity proxy in this analysis (Remark~\ref{rmk:localized}; its normalized effective-sample-size reading $\locfun{\Sacc}$ is heuristic, not separately derived here). This identification rests on the matched bound structure and the family-1 experiments (\S\ref{sec:synthetic}); a standalone theorem separating two worlds at fixed global ESS and frontier is left open, as is exact two-sided constant matching (Remark~\ref{rmk:localized}).

\begin{remark}[The map is floor-created]
\label{rmk:floorcreated}
\looseness=-1 As $\beta \to 0$ both lower-bound axes in Claim 1 vanish, each carrying $\sqrt{\beta}$, in the joint local regime $s \le c_0\beta$, so the slack vanishes \emph{with} $\beta$ (a vestigial sample requirement to confirm positive coverage persists). This vanishing is of the certifiable \emph{absolute} slack threshold $s \lesssim \kappa^{-1}\sqrt{\beta/n} + \sqrt{\beta/m}$ at fixed $(n,m)$; holding the \emph{relative} slack $s/\beta$ fixed instead, the requirements $n \gtrsim \beta/(\kappa s)^2$, $m \gtrsim \beta/s^2$ grow as $\Omega(1/\beta)$, so certifying a fixed relative slack does \emph{not} become free as $\beta \to 0$. Meanwhile $\cW_B$ and its overlap structure are unchanged: the positivity non-identifiability of \citet{damour2021overlap} persists at $\beta = 0$, and the $s$-scaled hardness exists only because the floor forces acceptance up to the frontier; Appendix~\ref{app:overlap} draws the contrasting overlap-created boundary.
The map is created by the floor, not by overlap (Remarks~\ref{rmk:beta0}, \ref{rmk:lossdriven}); the registered consistency diagnostic (monotone vanishing of the smallest certifiable slack as $\beta \to 0$) is Figure~\ref{fig:beta_vanishing} (Appendix~\ref{app:beta0}).
\end{remark}

\looseness=-2 Formal versions: Theorem~\ref{thm:lower} (Claim 1), Theorem~\ref{thm:modelA} (Claim 2, phase-diagram corollaries), Theorems~\ref{thm:U1}--\ref{thm:U2}/Proposition~\ref{prop:I1} (Claim 3; the parenthetical nuisance display above is schematic; the per-split form is Theorem~\ref{thm:U2}'s); assembled in Appendix~\ref{app:modelA}, pilot exponent artifact resolved in Appendix~\ref{app:exponent}. Each leg is elaborated in turn below.

\subsection{Model-B two-axis lower bound, and no matching Model-B upper}
\looseness=-1 Claim 1 is the impossibility leg; its formal statement is Theorem~\ref{thm:lower} (Appendix~\ref{app:lower}). Over the full unknown-weight class there is no matching Model-B upper: Model-B is inconsistent over $\cW_B$, with a hard world at every $(n,m)$ at the operating point $\alpha=\beta=\tfrac12$ (Theorem~\ref{thm:modelB-incon}, Appendix~\ref{app:modelB-incon}).
This inconsistency is why the map is \emph{necessarily} model-split, not a defect of the analysis: a single-model minimax statement over $\cW_B$ is provably unattainable (no valid procedure certifies a constant-slack world at any sample size), so the operative question is not the (nonexistent) $\cW_B$ minimax rate but which complexity restriction restores consistency and at what price, answered by Model-A (oracle weights) and Model-B$'$ (a pre-registered partition). The restriction is thus load-bearing, and the cross-model match is the correct object, not a substitute for an in-class theorem that cannot exist (Figure~\ref{fig:map_structure}).

\paragraph{Proof sketch.}
\looseness=-1 \emph{(i) The $n$-axis pair} needs no shift ($w \equiv 1$): raising the loss on an inframarginal slice (mass $\rho = \beta/2$) by $\Delta = 4\kappa s/\beta$ flips every coverage-$\ge\beta$ threshold to infeasible while the unlabeled law is untouched, at per-sample Kullback--Leibler (KL) cost $\lesssim \rho\Delta^2 \asymp (\kappa s)^2/\beta$; so Le~Cam forces refusal once $s \lesssim \kappa^{-1}\sqrt{\beta/n}$.
\emph{(ii) The $m$-axis pair} inverts the channels: shared $(P_X, \eta)$, $Q_X$ moves mass $\varepsilon \asymp s$ out of the safe block at $\chi^2$ cost $\asymp \varepsilon^2/\beta$ per draw, so the $\sqrt{\beta}$ is intrinsic, matching the variance-aware floor-LCB rate $\sqrt{\beta(1-\beta)/m}$ on the achievability side (Appendix~\ref{app:maxis}).
\emph{(iii) Additive combination:} both pairs share the base world, so one family per $(n,m)$ witnesses both axes, and the maximum of the two widths is at least half their sum: the additive display is this max form up to a factor of $2$ (Appendix~\ref{app:additive}).

\subsection{Model-A oracle achievability and the phase diagram}
\looseness=-1 Claim 2's oracle-weight certificate (Theorem~\ref{thm:modelA}, Appendix~\ref{app:modelA}) attains the Claim-1 rates up to constants and logarithms, the cross-model match.

\paragraph{Proof sketch.}
\emph{Model-A achievability:} an empirical-Bernstein test of the linearized risk $\GQ(\lambda) = \E_P[w \Sacc (L - \alpha)] \le 0$ \citep{maurer2009empirical} plus a variance-aware floor LCB, one shared budget over $\lat$ \citep{angelopoulos2021learn}, certifies once $\mathrm{LR}_{\mathrm{margin}}(s)$'s margins ($\kappa s/8$ risk, $s/4$ floor) absorb the deviations.

\begin{corollary}[Phase diagram]
\label{cor:phase}
\looseness=-2 At slack $s$ the \emph{sufficient} guaranteed-certifiable region in the $(n, m)$ plane (the set on which Theorem~\ref{thm:modelA}'s thresholds are met) contains a quadrant up to constants; in Model-A the source-weighted floor branch enlarges it (the labeled $n$-route can clear the floor even when the $m$-axis is starved), so the region is the union of the $m$-route quadrant $\{n \ge N_{\mathrm{risk}},\, m \ge M_{\mathrm{floor}}\}$ and an $n$-only half-strip $\{n \ge \max(N_{\mathrm{risk}}, N_{\mathrm{floor}})\}$ (thresholds in Theorem~\ref{thm:modelA}). Its corner is a \emph{sufficient-guarantee} boundary (schematic up to the unknown $\gtrsim$-constants, not a tight necessity threshold) indicating which purchase \emph{extends the guarantee}: left of it, buy labeled source; below it, buy unlabeled target; when the nuisance axis binds, buy weight-block samples $n_w, m_w$ (the nuisance radius shrinks at Proposition~\ref{prop:I1}'s rate); only shift-model \emph{misspecification} (weights not $\cK$-measurable) is uncurable by samples (sufficient boundary: Corollary~\ref{cor:buy}; feasibility readout: Corollary~\ref{cor:feasibility}).
\end{corollary}

\subsection{Model-B$'$ validity, power, and the nuisance price}
\looseness=-1 The implementable certificate (Algorithm~\ref{alg:bprime}, \S\ref{sec:setup}) realizes Claim 3 under the pre-registered stratified-shift model; we now state its validity and power.

\begin{theorem}[U1: validity over the stratified-shift class]
\label{thm:U1}
\looseness=-1 Under Assumption~\ref{ass:bprime}, for every world in $\cW_B$ with $\cK$-measurable $w$, Algorithm~\ref{alg:bprime} satisfies
$\Prob\big(\text{certify } \hat\lambda \,\wedge\, (\RQ(\hat\lambda) > \alpha \,\vee\, \CQ(\hat\lambda) < \beta)\big) \le \delta$.
\end{theorem}

\looseness=-1 \emph{Proof sketch.}
Three failure events (nuisance budget, risk bounds, floor bounds) share the budget $\delta_w + |\lat|\delta_r' + |\lat|\delta_f' = \delta$; on their complement, the bias identity $\E_P[\hat w \Sacc(\eta - \alpha)] = \GQ(\lambda) + \E_P[(\hat w - w)\Sacc(\eta - \alpha)]$ bounds the weight-estimation bias by $\nb$ (Proposition~\ref{prop:I1}), which the budgeted test's margin absorbs exactly; the floor side is a binomial bound (Appendix~\ref{app:bprime}). \hfill$\square$

\begin{theorem}[U2: power]
\label{thm:U2}
\looseness=-1 Under Assumption~\ref{ass:bprime}, $\mathrm{LR}_{\mathrm{margin}}(s)$, and a regular margin $(\kappa, s_0)$ with $s \le s_0 \wedge c_0\beta$, Algorithm~\ref{alg:bprime} certifies every world with slack exactly $s = \betastar - \beta$ (the world's actual slack, not a lower bound on it) in two forms.
\emph{(i) Sharp (conditional) form.} Condition on $D_w$, and write $Z^\lambda = \hat w \Sacc(L - \alpha)$ for the estimated-weight risk integrand. Let the realized estimated-weight variance proxy be
$\bar V_{\lambda_s} = \Var_P(Z^{\lambda_s} \mid D_w) \le \E_P[\hat w^2 S_{\lambda_s}]$ (random, $D_w$-measurable). Fix any realization $D_w \in E_w$ with $\rho_{\lambda_s} \le \kappa s/32$, and suppose the $D_w$-measurable size event
\[
  n_r \gtrsim \bar V_{\lambda_s}\log(|\lat|/\delta)/(\kappa s)^2 + B\log(|\lat|/\delta)/(\kappa s),
  \qquad
  m_f \gtrsim \beta\log(|\lat|/\delta)/s^2 + \log(|\lat|/\delta)/s
\]
holds. Then the conditional refusal probability satisfies $\Prob\big(\text{Algorithm~\ref{alg:bprime} outputs } \nocert \mid D_w\big) \le \delta_r' + \delta_f'$ (when $\lambda_s$ passes both tests the algorithm certifies its floor-LCB-maximal passing point, not necessarily $\lambda_s$). Marginally,
\[
  \Prob(\nocert) \le \delta_w + \Prob(\rho_{\lambda_s} > \kappa s/32) + \Prob(\text{size event fails}) + \delta_r' + \delta_f';
\]
form (ii) discharges the two middle terms at deterministic split sizes.
\emph{(ii) Deterministic, unconditional.} At the population-mass split sizes $n_w \gtrsim B^2 K_{\lambda_s} p(A_{\lambda_s})/(\kappa s)^2$ and $m_w \gtrsim K_{\lambda_s} q(A_{\lambda_s})/(\kappa s)^2$ (plus $L_w$-terms and the Assumption~\ref{ass:bprime}(ii) feasibility floor $n_w \gtrsim \log(4K/\delta_w)/p_{\min}$), with $n_r \gtrsim B\beta\log(|\lat|/\delta)/(\kappa s)^2 + B\log(|\lat|/\delta)/(\kappa s)$ and $m_f$ as in (i), the Proposition~\ref{prop:I1} good event secures both $\rho_{\lambda_s} \le \kappa s/32$ and the envelope $\bar V_{\lambda_s} \le \tfrac54 B\,\CQ(\lambda_s) \lesssim B\beta$ at no extra confidence budget, so the certificate succeeds with probability $\ge 1 - \delta$ \emph{unconditionally} (the random budgets of (i) folded into $\delta$; Appendix~\ref{app:bprime}).
\end{theorem}

\begin{proposition}[I1: stratified-shift histogram estimator]
\label{prop:I1}
Under Assumption~\ref{ass:bprime}, with probability $\ge 1 - \delta_w$ simultaneously for every region $A \in \{A_\lambda\}_{\lambda \in \lat} \cup \{\cX\}$,
$\E_P[\,|\hat w - w| \mathbf{1}_A\,] \le 2\big[\sqrt{|A|_\cK\, q(A)/m_w} + \sqrt{2L_w/m_w}\big] + 2B\big[\sqrt{|A|_\cK\, p(A)/n_w} + \sqrt{2L_w/n_w}\big]$, and the same bound holds with $(p, q)$ replaced by plug-in corrected empirical masses, making the budgets $\nb$ computable.
\end{proposition}

\looseness=-1 Theorem~\ref{thm:U2}'s power guarantee is \emph{conditional} on $D_w$ (form (i), sharp in the realized \emph{estimated}-weight variance proxy $\bar V_{\lambda_s}$, with the small-nuisance event $\{\rho_{\lambda_s} \le \kappa s/32\}$ and the random size event discharged by form (ii)), and is \emph{unconditional} at the deterministic population-mass split sizes (form (ii)).
Proposition~\ref{prop:I1} with the upper-tail closure of Appendix~\ref{app:bprime} secures the good event with probability $\ge 1 - \delta_w$ at those split sizes, folding it into the overall $\delta$-budget rather than leaving a separate high-probability event uncontrolled.
\looseness=-1 In the regime $s \ge 32\nb/\kappa$, Theorem~\ref{thm:U2} reproduces the Model-B lower-bound orders of Theorem~\ref{thm:law} up to constants, logarithms, and the shift-model dimension $K$ (nuisance split only); the $\nb/\kappa$ nuisance floor is \emph{not} proved information-theoretically necessary in Model-B$'$ (its $K$-free core necessary, the histogram $B^2K$ only sufficient, unknown-$\eta$ open; Remark~\ref{rmk:scope}, Theorem~\ref{thm:t3-target}, Appendix~\ref{app:necessity}).
The match on the two \emph{main} axes is moreover genuinely \emph{within} Model-B$'$, not merely cross-model. The reason: the lower-bound constructions of Theorem~\ref{thm:lower} are weight-simple ($w \equiv 1$ on the $n$-axis, piecewise-constant on a fixed three-cell partition on the $m$-axis), so for any pre-registered $\cK$ refining those pieces they are $\cK$-measurable, hence in the sample-independent Model-B$'$ class (Assumption~\ref{ass:bprime}'s $p_{\min}$ is the certificate's feasibility condition, not a class-membership requirement). Knowing $\cK$ does not defeat the Le~Cam pairs: the $m$-axis differs only in $Q_X$, detectable only once $m \gtrsim \beta/s^2$ (Remark~\ref{rmk:bprime-match}).

\looseness=-1 The upshot, stated formally as Corollary~\ref{cor:bprime-minimax} below, is a two-tier match within Model-B$'$: a \emph{pointwise} correspondence on the \emph{totals} $(n, m)$ (off the weight-split-dominated regime, instance-dependent, not a uniform minimax over the class), and only a componentwise \emph{rate} correspondence on the risk/floor \emph{test} sub-blocks $(n_r, m_f)$; the histogram nuisance $B^2 K$ is the lone unmatched axis, open under unknown $\eta$.
\begin{corollary}[Two-axis rate correspondence within Model-B$'$ (on totals)]
\label{cor:bprime-minimax}
\looseness=-1 \emph{Operating regime.} Work on the Theorem~\ref{thm:lower} rectangle ($B \ge 2$, $\alpha \in [0.2, 0.9]$, $\beta \in (0, 0.6]$, $\kappa < \tfrac12\min\{\alpha, 1-\alpha\}$); fix $\delta \le 1/8$, a regular-margin level $(\kappa, s_0)$, and a target slack $s \le s_0 \wedge c_0\beta$.

\emph{Compatible partition and lattice.} Pre-register a score-refined partition $\cK$ and a risk-score-aligned nested-threshold lattice $\lat$ of mesh $< s/4$ that satisfies $\mathrm{LR}_{\mathrm{margin}}(s)$ and carries the Theorem~\ref{thm:lower} construction blocks as $\cK$-cells and $\lat$-thresholds. One such $(\cK, \lat)$ exists for this $s$ (the blocks are intervals and the aligned grid resolves them, Remark~\ref{rmk:bprime-match}); this is \emph{not} a claim uniform in $s$ or over arbitrary partitions.

\emph{Parameter class.} Let $\cW_{B'}(\kappa, s_0; \cK, \lat)$ be the Model-B$'$ worlds: $w \le B$, $\cK$-measurable, $L \in [0,1]$, regular margin $(\kappa, s_0)$ on $\lat$, a parameter class fixed \emph{independently of any procedure's sample sizes} (Assumption~\ref{ass:bprime}'s $p_{\min}$ and union-of-cells conditions are feasibility requirements on Algorithm~\ref{alg:bprime} below, not on the class). Then, at fixed $B$ and up to constants and logarithms:
\begin{itemize}
\setlength{\itemsep}{1pt}
\item \emph{(lower, in B$'$)} every $(\alpha, \beta, \delta)$-valid procedure refuses with probability $\ge \tfrac12$ on some world of $\cW_{B'}$ of slack \emph{exactly} $s$ (the Theorem~\ref{thm:lower} $n$-axis construction has relaxed frontier $\betastar = \beta + s$ exactly, Appendix~\ref{app:naxis}, so it matches the upper bullet's target slack) once $n \lesssim \beta/(\kappa s)^2$ or $m \lesssim \beta/s^2$ (the Theorem~\ref{thm:lower} pairs, re-run inside B$'$);
\item \emph{(upper, in B$'$)} under $\mathrm{LR}_{\mathrm{margin}}(s)$ and Assumption~\ref{ass:bprime}'s feasibility ($p_{\min}$, union-of-cells), Algorithm~\ref{alg:bprime} certifies every \emph{$p_{\min}$-feasible} slack-exactly-$s$ world of $\cW_{B'}$ with probability $\ge 1-\delta$ once the risk- and floor-test splits obey $n_r \gtrsim \beta\log(|\lat|/\delta)/(\kappa s)^2$ and $m_f \gtrsim \beta\log(|\lat|/\delta)/s^2$, \emph{plus} the weight-estimation splits $n_w, m_w$ that secure $\rho_{\lambda_s} \le \kappa s/32$ on the good event (Theorem~\ref{thm:U2}(ii), deterministic form; carrying $L_w$ and the $B^2K$ factor).
\end{itemize}
The within-B$'$ correspondence on the \emph{totals} $(n, m)$ in the regime where the weight splits do not dominate (so $n \asymp n_r$, $m \asymp m_f$; displayed below) is a \emph{pointwise} rate match: there Algorithm~\ref{alg:bprime}'s total requirement meets the total-sample lower bound up to constants and logs. But the upper split sizes depend on the unknown world ($p(A_{\lambda_s}), q(A_{\lambda_s}), K_{\lambda_s}, p_{\min}$), so this is instance-dependent, not a uniform minimax over the class. A uniform statement would need a fixed subclass with $p_{\min} \ge p_0$ and accepted-mass bounds.

\looseness=-1 On the algorithm's \emph{test-sample} sub-blocks $(n_r, m_f)$ alone the correspondence is weaker: a componentwise \emph{rate} match (the test terms carry the same $\beta/(\kappa s)^2$, $\beta/s^2$ dependence as the total-sample lower bounds), \emph{not} a minimax lower bound on $(n_r, m_f)$ themselves. The reason: an unrestricted B$'$ procedure may reuse the weight-split blocks $D_w^Q$ (resp.\ $D_w^P$) for the floor (resp.\ risk) test, so a lower bound on the total $m$ (resp.\ $n$) does not transfer to the single block $m_f$ (resp.\ $n_r$); e.g.\ a procedure with $m_f = 0$ but $m_w \gg \beta/s^2$ is not bound by the $m_f$ threshold. It is not a single minimax theorem for the nuisance either.

\looseness=-1 The weight splits are the separate, unmatched axis:
\[
  n_w \asymp B^2 K\, p(A_{\lambda_s})/(\kappa s)^2,
  \qquad
  m_w \asymp K\, q(A_{\lambda_s})/(\kappa s)^2,
\]
together with the separate Assumption~\ref{ass:bprime}(ii) all-cell feasibility floor $n_w \gtrsim \log(4K/\delta_w)/p_{\min}$ (which can dominate when a pre-registered cell, even one irrelevant to $A_{\lambda_s}$, has vanishing source mass); the $K$-free core is necessary (lattice-valued; Theorem~\ref{thm:t3-target}) but the full $B^2K$ rate is proved only sufficient (unknown-$\eta$ open). The totals coincide with the test axes ($n \asymp n_r$, $m \asymp m_f$) only when these weight splits do not dominate, which, carrying the additive $L_w$ and feasibility terms \emph{in full} (not as mere log multipliers), reads
\[
  \begin{aligned}
    \text{(source)}\quad & B^2\big(K p(A_{\lambda_s}) + L_w\big)/(\kappa s)^2 + \log(4K/\delta_w)/p_{\min} \lesssim B\beta/(\kappa s)^2, \\
    \text{(target)}\quad & \big(K q(A_{\lambda_s}) + L_w\big)/(\kappa s)^2 \lesssim \beta/s^2.
  \end{aligned}
\]
The bare $L_w/(\kappa s)^2$ and all-cell $p_{\min}$ terms can dominate even when $B K p(A_{\lambda_s}) \lesssim \beta$ and $K q(A_{\lambda_s}) \lesssim \kappa^2\beta$, so these are not ``up to logs'' simplifications.

\looseness=-1 This sample-size comparison is distinct from the accuracy condition $\rho_{\lambda_s} \le \kappa s/32$. One pre-registered $(\cK, \lat)$ serves both bounds, reconciling Theorem~\ref{thm:lower}'s after-$s$ lattice with Theorem~\ref{thm:U2}'s pre-registered one (proof: Appendix~\ref{app:bprime}).
\end{corollary}

\looseness=-1 $\mathrm{LR}_{\mathrm{margin}}$ is likewise a substantive practical assumption: risk-score-aligned lattices inherit it (Appendix~\ref{app:modelA}); a misaligned score can leave the lattice margin-free; then the certificate refuses, and the refusal is observable, though its test-side label does not by itself identify the missing lattice margin as the cause.

\subsection{Nuisance necessity, the open edge, and what B$'$ does not certify}
\label{sec:notcertify}
Validity in Theorem~\ref{thm:U1} is conditional on the pre-registered $K$-cell stratified-shift model, the \emph{same} epistemic status as the covariate-shift assumption itself, with the added pre-registered $\cK$-measurability stated here, and we map its boundary rather than assume it away: if $w$ is not $\cK$-measurable, the guarantee degrades to the $\cK$-projection of $w$ with an additive bias term not estimable from covariates alone (Remark~\ref{rmk:misspec}).
On the real workload the \S\ref{sec:real} sensitivity experiment shows \emph{refusal-status invariance only}, not correctness of the shift model.
Separately, a \emph{deliberately adversarial} synthetic sweep (Appendix~\ref{app:revision}) constructs a within-cell $w$--$\eta$ correlation that stays in $\cW_B$ but, by design, leaves the B$'$ $\cK$-measurable class.
That perturbation is therefore invisible to B$'$'s own certificate, and it nonetheless drives true violations once the receding frontier crosses the locked operating point, while oracle-A stays valid: an adversarial characterization of the conditional guarantee's boundary, \emph{not} an off-partition robustness guarantee, which the paper does not claim.
The nuisance-axis necessity is only \emph{partially} settled (Theorem~\ref{thm:t3-target}, Appendix~\ref{app:necessity}): a $K$-free core necessary (lattice-valued) in the known-$\eta$ submodel (a localized restatement of the $m$-axis target bound plus the existing $n$-axis reduction; the worst-slice $B^2p(A)$ envelope stated, not fully re-derived), histogram $B^2 K$ not proved necessary, unknown-$\eta$ the one open edge, no complete characterization claimed.

\paragraph{Honest-failure semantics.}
\looseness=-1 A refusal $\nocert$ carries a deterministic axis attribution: two labels the certificate itself computes (risk-starved, floor-starved) and two conditional diagnoses (nuisance-limited, frontier-infeasible). For the data-starved labels, Corollary~\ref{cor:phase} reads off what to buy under that reading \citep{kato2023double, yang2024doubly, wang2026weight, laghuvarapu2026kmmcp, wang2025subpopulation}, while when a separate frontier diagnosis supports frontier infeasibility, no $n$ or $m$ purchase fixes the refusal (raise capability or renegotiate the service-level agreement).
Call the \emph{floor-pass set} those $\lambda \in \lat$ whose floor LCB is at least $\beta$ and positive, exactly the floor-pass set Algorithm~\ref{alg:bprime} computes. When several attributions bind at once the report follows a fixed priority: nuisance-limited ($\rho_{\lambda} > \kappa s/32$ at the floor-LCB-maximal candidate) $\succ$ risk-starved (the floor-pass set is nonempty but no member clears the risk UCB) $\succ$ floor-starved (the floor-pass set is empty); frontier-infeasible is reported in place of these only when the separate frontier diagnosis (\S\ref{sec:frontier_diag}) certifies $\betastar < \beta$; this convention fixes only the displayed reason and never alters the certify/refuse decision or its validity.
\looseness=-1 The labels differ in what they require. Risk-starved and floor-starved are read off the certificate's own statistics, since both the floor-pass set and the risk test are computed inside Algorithm~\ref{alg:bprime}. The nuisance-limited comparison additionally involves $(\kappa, s)$, which the algorithm does not receive and which the data does not identify, so it can be reported only where those are supplied as planning values; no run reported here uses it, and absent it a refusal is displayed as risk-starved or floor-starved by the priority above. Frontier-infeasible likewise comes from the separate diagnosis of \S\ref{sec:frontier_diag}, not from the certificate's own output. A risk-starved label can therefore stand in for an undiagnosed nuisance-limited refusal, so the purchase guidance of Corollary~\ref{cor:phase} is conditional on the displayed reading.
Escalate-all rows are reported as a no-automatic-answer guarantee, never as a win.

\section{Experimental Evaluation}
\label{sec:experiments}

\looseness=-1 We evaluate the map on registered synthetic generators (an audit of validity, practical constants, and mechanism consistency; \S\ref{sec:synthetic}, \S\ref{sec:validity}) and then audit the certificate's honest refusal on a real workload (\S\ref{sec:real}).

\subsection{Synthetic mechanisms and rate-shape evidence}
\label{sec:synthetic}
\looseness=-1 This suite is an audit of the map on registered generators for \emph{validity, practical constants, and mechanism consistency}: theorem-carried evidence, not a claim of broad empirical generality for the B$'$ certificate; the scope guard is part of each claim.
\looseness=-1 Every block below is claim-driven, with protocol and registered generators in Appendix~\ref{app:protocol} and number-to-file traceability maintained through the authors' frozen results manifest (reporting pipeline: Appendix~\ref{app:repro}). One reporting discipline holds throughout: $\delta = 0.05$ (certificate confidence, pooled criterion), $\dtol = 0.0625$ (powered per-cell tolerance), and 95\% Clopper--Pearson upper confidence bounds (CP-UCB) on violation rates \citep{clopper1934use} are three separate quantities, never interchanged, with Holm's correction \citep{holm1979simple} for multiplicity. For interval estimates we use ordinary least squares (OLS) standard-error normal approximations for the log--log exponents and bootstrap intervals \citep{efron1979bootstrap} for rank correlations.

\paragraph{Bite divergence near the boundary (Claim C1).}
\label{sec:bite}

\begin{figure}[t]
    \begin{minipage}[c]{0.48\textwidth}
    \centering
    \includegraphics[width=\linewidth]{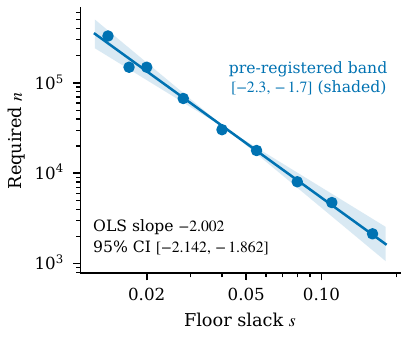}
    \end{minipage}\hfill
    \begin{minipage}[c]{0.48\textwidth}
    \caption{\textbf{Bite divergence in the registered synthetic bite family.}
    Required labeled-sample size $n$ to certify at floor slack $s$ (log--log). The
    log-space OLS fit has slope $-2.002$, 95\% CI $[-2.142, -1.862]$ (normal
    approximation from the OLS standard error), inside the pre-registered acceptance
    band $[-2.3, -1.7]$ (shaded fan, anchored at the fit midpoint); the
    quadratic-curvature CI contains 0. The divergence $n \propto s^{-2}$ near the
    certifiable boundary matches the map's labeled-risk axis. Scope guard: registered
    family; not a broad multi-family empirical proof.}
    \label{fig:bite_divergence}
    \end{minipage}
\end{figure}

\looseness=-1 In the registered synthetic bite family, the labeled-sample requirement diverges quadratically near the certifiable boundary, exactly as the map's labeled axis predicts: required-$n$ vs.\ slack log--log slope $-2.002$, 95\% CI $[-2.142, -1.862] \subset$ pre-registered band $[-2.3, -1.7]$; the curvature CI contains 0 (Figure~\ref{fig:bite_divergence}).
The fit rests on 9 slack design points; a deterministic pairs bootstrap ($20{,}000$ resamples) gives a 95\% slope CI $[-2.11, -1.89]$, consistent with the OLS normal-approximation interval, with log-space residuals $\le 0.10$ in magnitude (raw points and bootstrap script in Appendix~\ref{app:experiments}).
Breadth is now measured, not deferred: the labeled bite holds across three structurally distinct families: continuum (the registered regular-margin family, $-2.002$, in band), kcell ($-1.617$, CI $[-1.76, -1.47]$), and a non-theorem-aligned family ($-1.845$, CI $[-2.02, -1.67]$); the steep, near-quadratic divergence is robust across all three (slopes in $[-2.0, -1.6]$). The exact $-2$ is specific to the regular-margin family. The exponent ordering tracks the accepted-region margin geometry (margin-vs-slack exponent $b \approx 1.10/0.91/0.99$ for continuum/kcell/non-theorem), consistent with the localized-functional thesis (Claim C4); the precise exponent additionally reflects finite-sample range terms and lattice discreteness. The relaxed-crossing/lattice-plateau dichotomy behind this ordering, together with the separation of the empirical margin-vs-slack $b$ from the relaxed budget's linear frontier-margin slope, is given by Proposition~\ref{prop:exponent} with Lemma~\ref{lem:plateau} (Appendix~\ref{app:exponent}); the finite-sample record is in Appendix~\ref{app:experiments}.

\paragraph{Two-axis form and contours: consistency checks, not proof (Claims C2, C3).}
\label{sec:twoaxis}

\looseness=-1 The two-axis functional form itself is carried by Theorem~\ref{thm:law}; the experiments here are consistency checks.
Held-out cross-validation (CV) of candidate $s^*(n, m)$ surfaces gives the additive model the best point estimate (surface CV $R^2 = 0.9202$).
At the registered density (a 41-cell surface) the CI-separation criterion failed (bootstrap CIs overlapped) and the axis-exponent CIs missed containment, so we reported the form as theorem-carried and registered the failure as \emph{sample-limited}; a denser surface bears that out.
Re-running the \emph{same} registered model-comparison procedure (byte-identical criterion; only design density and Monte-Carlo depth raised) at increasing cell count, the probability of strict non-overlap against \emph{every} alternative rises $0.12$ (40-cell base rung, the sweep resolution nearest the registered 41) $\to 0.75$ (134) $\to 1.00$ (635 cells, 800 reps).
At 635 cells the additive form (held-out cross-validation root-mean-square error, CV-RMSE $0.0122$, 95\% CI $[0.0109, 0.0134]$) separates from all four candidate forms (the nearest $2.2\times$ worse), unanimously across all 24 re-splits, with a monotonically growing gap (Figure~\ref{fig:ci_separation}; Appendix~\ref{app:surfacecv}).
The CI-separation criterion is thus met and the 41-cell failure was statistical power, not non-identifiability (the per-axis exponent-containment sub-test is a separate check, not revisited here); the $R^2 = 0.9202$ is the $s^*(n,m)$ surface-CV result, not contour evidence.
Contour transfer is qualitative rate-shape only: two-axis constants fit on the first shift intensity and applied unchanged give in-sample $L1$ correlation $.997$ and held-out $L2$--$L4$ correlations $.981/.963/.890$, but the strict pre-registered $\pm 3\%$ bands failed at all intensities, with held-out median factor errors $\times 1.24/\times 1.91/\times 3.40$ (Figure~\ref{fig:contour_shape}): constants do not transfer; ordering and rate shape do.
The non-transfer is itself systematic, not noise: the factor errors are strictly monotone in shift intensity (rank correlation $1.0$) and grow $\times 3.2$ from L1 (in-sample, $\times 1.07$) to L4 ($\times 3.40$) while the accepted-region variance proxy grows only $\times 1.3$: a predictable, intensity-driven constant inflation with rate-shape preserved (log-correlation $\ge .889$ at every intensity), not an unstructured breakdown (Appendix~\ref{app:contours}; numeric record in Appendix~\ref{app:revision}).

\paragraph{Localized geometry, restricted (Claim C4).}
\label{sec:geometry}

\looseness=-1 The geometry reading of Theorem~\ref{thm:law} (Remark~\ref{rmk:localized}; the normalized form is heuristic, not separately derived) suggests that certification cost tracks the localized accepted-region functional $\locfun{\Sacc}$ rather than global ESS, and the B$'$ arm of family 1 shows this cleanly: across 640 worlds, B$'$ required-$n$ has Spearman $.936$ (95\% bootstrap CI $[.927, .944]$) against the localized functional and $.026$ against global ESS (Figure~\ref{fig:rank_correlation}).
The restriction is part of the claim: empirical support is limited to B$'$ family 1 plus qualified family-2 within-pair evidence (within-pair ratio agreement $.995$ but pooled Spearman $.503$ under $\kappa$/censoring confounds, 162/640 worlds censored), and oracle-arm pooled correlations are not supportive (a range-constant artifact; Appendix~\ref{app:geometry}); no broad-class claim is made.

\begin{figure}[t]
    \centering
    \includegraphics[width=0.95\textwidth]{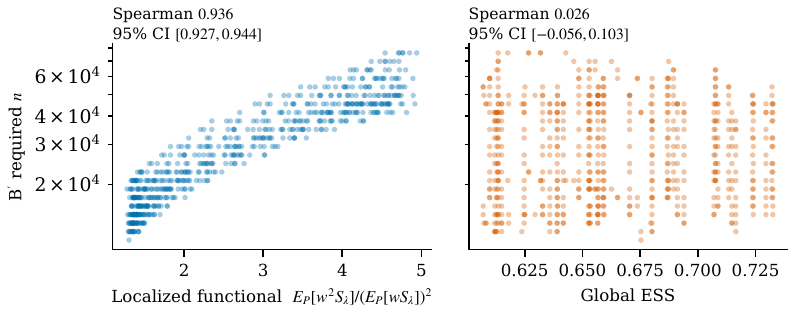}
    \caption{\textbf{The localized accepted-region functional, not global ESS, tracks
    B$'$ certification cost} (family 1, B$'$ arm; 640 worlds = 8 ratio levels
    $\times$ 40 paired draws $\times$ 2 members, at slack $s = .05$ with
    $m = 10^5$ fixed). (a)~B$'$ required $n$ against the localized functional
    $E_P[w^2 S_\lambda]/(E_P[w S_\lambda])^2$: Spearman $.936$, 95\% bootstrap CI
    $[.927, .944]$. (b)~The same required $n$ against global ESS: Spearman $.026$
    $[-.056, .103]$. Family 2 shows within-pair ratio agreement $.995$ but pooled
    Spearman only $.503$ under $\kappa$/censoring confounds (162/640 worlds censored
    at the required-$n$ cap); oracle-arm pooled correlations are not supportive
    (range-constant artifact). Evidence is restricted to B$'$ family 1 plus qualified
    family-2 within-pair structure; no broad-class claim.}
    \label{fig:rank_correlation}
\end{figure}

\subsection{Validity audit and practical constants}
\label{sec:validity}

\looseness=-1 In the full B1 grid audit (4 intensities $\times$ 256 cells, 1{,}000 replications each, with five higher-variance edge cells re-run at 2{,}000--8{,}000), the formal-guarantee arms are valid where they certify: oracle-A certifies 583/1{,}024 cells and B$'$ certifies 8/1{,}024, for 591 arm-cell certifications in total.
Those 591 certifications carry 0 violations, 0 Holm rejections, pooled CP-UCBs $9.0\times10^{-6}$ / $1.8\times10^{-4} \le \delta$, and powered per-cell CP-UCBs $\le \dtol$ (Table~\ref{tab:validity_constants}; Figure~\ref{fig:hero}b). These are the released kernels' runs: Remark~\ref{rmk:releasedradius} states the nuisance radius the B$'$ kernel computes and proves the guarantee for it, and Remark~\ref{rmk:releasedoracle} states the two-family variant the oracle-A kernel runs and proves its validity at the same $\delta$.
A deep-audit subset (6 cells $\times$ 4{,}000 replications, both formal arms) recorded 0 violations in 48{,}000 replications, pooled CP-UCB $6.2\times10^{-5}$.
Two denominators appear in this paper and we label each interval by the one it uses. The \emph{unconditional} convention divides by the replication count and bounds $\Prob(\text{some invalid } \lambda \text{ is certified})$, which is the quantity the $(\alpha, \beta, \delta)$ guarantee caps at $\delta$; the \emph{conditional} convention divides by the number of certifying replications and bounds the violation rate among issued certificates. Every interval in the B1 grid audit and in the deep-audit subset just quoted is unconditional, including the $48{,}000$ above, so none of them is a rate among issued certificates. The one place we report the conditional convention is the targeted existence run below, where the denominator is written explicitly as certifying replications.
As a fairness check on the lower-bound thesis we substitute the risk test inside B$'$'s exact construction: a \emph{SCoRE-inspired} floor-augmented betting arm, in the spirit of \citet{bai2026score}, in which a weighted betting e-value replaces the empirical-Bernstein risk test and every other block is verbatim-identical to B$'$ (Proposition~\ref{prop:score}). This is an equal-budget \emph{risk-test ablation}, not a source-faithful SCoRE implementation: \citet{bai2026score} builds conformal risk-adjusted e-values for individual test instances and thresholds them, whereas this arm keeps B$'$'s weight estimator, nuisance radius, floor LCB, four-block split, and output rule and bets on the calibration block. We therefore do not present it as the strongest available same-target baseline, but as the equal-construction alternative we tested against B$'$'s risk test, which is what the lower-bound thesis needs. It is valid and competitive: across the 1{,}024-cell grid its certify-frequency is $\ge$ B$'$'s in every cell (strictly greater in 8) at the \emph{same} powered-cell count, with 0 violations. In this fixed matched-replication run it therefore weakly dominates B$'$ in \emph{observed} certify-frequency without enlarging the \emph{powered} set; we report no paired test or interval for that ordering, so it is an ordering observed in this Monte Carlo run rather than an estimated power gap. That an \emph{otherwise} verbatim-identical alternative shares B$'$'s conservatism is the evidence we want: the conservatism is a shared price of the matched estimated-weight construction, not specific to the empirical-Bernstein risk test. This is exactly why the certificate's contribution is the lower-bound theory and the priced nuisance, not a finite-sample power edge (full head-to-head and pre-registration: the SCoRE head-to-head record of Appendix~\ref{app:scorehh}; self-contained validity proof: Appendix~\ref{app:score}).
Demonstration arms violate at scale: floor-free certification draws 108 Holm rejections, weighted conformal \citep{tibshirani2019conformal, barber2023conformal} draws 40, and plug-in draws 662; the distributionally robust optimization (DRO)-box baseline certifies 0 cells: an escalate-all, no-automatic-answer guarantee, not validity evidence (Proposition~\ref{prop:separation}, Appendix~\ref{app:separation}).

\looseness=-1 This is also a practical-constants result, and we keep the constants visible: B$'$ is much more conservative than oracle-A: at $\kappa s = 0.042$, oracle-A first certifies at $n_{\mathrm{total}} = 4{,}096$ and fully certifies at $16{,}384$, B$'$ at $524{,}288$ and $1{,}048{,}576$ ($128\times$ at the first-certification onset, $64\times$ at full certification); the registered $K$-sweep prices the shift-model dial directly (Table~\ref{tab:validity_constants}, bottom). These multipliers are measured against the released two-family oracle kernel (Remark~\ref{rmk:releasedoracle}), which differs from B$'$ in its confidence allocation and floor construction as well as in knowing $w$, so they bundle those differences with weight estimation rather than isolating it; read them as the cost of the whole estimated-weight construction, an upper bound on the weight-estimation cost alone, and not as an artifact of the risk test \citep{wang2026weight, kimura2024short}. Isolating weight estimation would need a comparator matched to B$'$ in every block except the weights, which we do not run.

\looseness=-1 These constants are one slice of a characterized envelope, not an isolated harsh point: B$'$ certifies across the partition sweep $K \in \{4, \dots, 64\}$ with an onset premium over oracle-A that scales linearly in the partition dimension ($K^{0.955}$, 95\% CI $[0.755, 1.154] \ni 1$), while the declared bound $B$ is second-order and the B$'$ labeled onset obeys the same $s^{-2}$ bite as oracle-A. This is exactly the reading of Theorem~\ref{thm:t3-target}: this $K$-scaling premium is the histogram estimator's $\ell^1$ weight-recovery price (Proposition~\ref{prop:I1}), \emph{not} a necessary certification cost; the necessary core is $K$-free (lattice-valued) and matched from below (full envelope, with the $B^{0.62}$ and $s^{-2}$ records: Appendices~\ref{app:necessity}, \ref{app:experiments}).

\looseness=-1 A targeted companion run confirms B$'$'s conservatism is not an all-refusal failure mode: on one on-class $K = 4$ grid, B$'$ reaches certify-frequency $1.0$ at $n = m = 262{,}144$ ($s = 0.20$) and $n = m = 1{,}048{,}576$ ($s = 0.12$), with $0/16{,}297$ B$'$ certifying replications violating: positive existence evidence only, not a breadth or real-workload claim (Appendix~\ref{app:experiments}).

\begin{table}[t]
\centering
\caption{\textbf{Full-grid validity audit and practical constants} (B1 grid: 4
shift intensities $\times$ 256 cells $=$ 1{,}024 cells, 1{,}000 replications per
cell; five edge cells re-run at 2{,}000--8{,}000). Three quantities are reported separately and are not interchangeable:
$\delta = 0.05$ (certificate confidence; pooled criterion), $\delta_{\mathrm{tol}}
= 0.0625$ (powered per-cell tolerance, evaluated over cells with $\ge 47$
certifying replications), and 95\% Clopper--Pearson upper confidence bounds (CP-UCB)
on violation rates. Holm correction runs across certifying cells within each arm.
Formal-guarantee arms certify only where valid; demonstration arms violate at scale.
Deep-audit subset (6 grid cells $\times$ 4{,}000 reps, both formal arms): 0 violations in 48{,}000 replications, pooled CP-UCB $6.2\times10^{-5}$ $\le \delta$. Every interval in this table uses the unconditional denominator (replications, not issued certificates), so it bounds the failure probability the guarantee caps at $\delta$.
Bottom: absolute-budget excerpts pricing B$'$'s conservatism (synthetic grid world at
$\kappa s = 0.042$; $K$-price from the registered $K$-sweep; family-2 censoring
at the required-$n$ cap). $^{\ddagger}$DRO-box certifies 0 cells: an escalate-all,
no-automatic-answer guarantee, not validity evidence.
The equal-construction alternative we tested against B$'$'s risk test (a SCoRE-inspired floor-augmented betting arm, Proposition~\ref{prop:score}; a risk-test ablation, not a source-faithful SCoRE implementation) is reported here rather than as a row because its head-to-head runs at \emph{matched} replications (the B$'$ row above uses the registered escalated reps): at matched reps SCoRE's certify-frequency is $\ge$ B$'$'s in every one of the 1{,}024 cells (strictly $>$ in 8), with 0 violations, the same powered-cell count, and a non-trivial-certification count of 10 vs.\ B$'$'s 9 (Appendix~\ref{app:experiments}).}
\label{tab:validity_constants}
\small
\resizebox{0.96\textwidth}{!}{%
\begin{tabular}{llcccccc}
\toprule
Arm & Type & Cert.\ cells & Viol. & Holm rej. & Powered cells &
Max powered-cell CP-UCB & Pooled CP-UCB \\
\midrule
Oracle-A & formal & 583/1{,}024 & 0 & 0 & 432 & 0.061 & $9.0\times10^{-6}$ \\
B$'$ & formal & 8/1{,}024 & 0 & 0 & 8 & 0.013 & $1.8\times10^{-4}$ \\
Floor-free & demo. & 112/1{,}024 & 107{,}157 & 108 & 108 & 1 & 0.824 \\
Weighted-conf. & demo. & 40/1{,}024 & 60{,}469 & 40 & 38 & 1 & 1 \\
Plug-in & demo. & 1{,}024/1{,}024 & 190{,}914 & 662 & 1{,}024 & 0.56 & 0.184 \\
DRO-box$^{\ddagger}$ & demo. & 0/1{,}024 & 0 & 0 & 0 & -- & -- \\
\bottomrule
\end{tabular}}

\vspace{0.6em}
\begin{tabular}{lcc}
\toprule
Absolute-budget excerpt & Oracle-A & B$'$ \\
\midrule
Certification onset (first cert.\ rate $>0$), $n_{\mathrm{total}}$ & 4{,}096 & 524{,}288 (rate 0.145) \\
Full certification (cert.\ rate $=1.0$), $n_{\mathrm{total}}$ & 16{,}384 & 1{,}048{,}576 ($64\times$) \\
$K$-price: B$'$ cert.\ rate at $n_{\mathrm{total}}=2^{20}$, $K=4/8/16/32/64$ & -- & 1.00/0.72/0.00/0.00/0.00 \\
$K$-price: B$'$ cert.\ rate at $n_{\mathrm{total}}=2^{22}$, $K=4/8/16/32/64$ & -- & 1.00/1.00/1.00/1.00/0.00 \\
Family-2 required-$n$ censored at cap (of 640 worlds) & 0 & 162 \\
\bottomrule
\end{tabular}
\end{table}

\subsection{Real-workload negative feasibility audit}
\label{sec:real}
\looseness=-2 This section reports a negative feasibility audit (it exercises the certificate's refusal semantics on real data), not a second validation block, and not a deployment demonstration; the service-level agreement (SLA) is a pre-registered operational scenario, fixed before any data was described (not an external mandate or benchmark), and the scope is a single-corpus domain-shift audit.

\paragraph{Setup.}
\looseness=-2 \label{sec:realsetup}
The workload is extractive question answering under one domain-shift axis: SQuAD as calibration source, NewsQA as shifted target, both via the MRQA distribution \citep{rajpurkar2016squad, trischler2017newsqa, fisch2019mrqa}, $N_{\mathrm{eval}} = 4{,}212$ evaluation items.
Correctness is non-judge (token-level F1 $\ge 0.5$ against gold answers, avoiding model-judge biases \citep{zheng2023judging}), selection scores follow standard uncertainty families \citep{farquhar2024detecting}, the pipeline is frozen and cached, and the SLA scenario was fixed pre-data at $\alpha = .10$, $\beta = .60$, $\delta = .05$.
Two capability legs run on the same corpus: leg 1 Llama-3-8B \citep{grattafiori2024llama} (availability-triggered, recorded deviation), leg 2 DeepSeek-V4-Flash, a conditional capability-ladder re-entry adjudicated under the registered escalation rule, \emph{not} independent validation; an access guard ensured evaluation labels were touched exactly once, by the final evaluation (both recorded deviations are logged in Appendix~\ref{app:realrecords}).

\paragraph{Pre-deployment certificate output: honest refusal.}
\looseness=-1 \label{sec:refusal}
The B$'$ certificate's output in this audit is pre-deployment honest refusal, consuming no evaluation labels: 1{,}000/1{,}000 refusals per leg, 2{,}000/2{,}000 total, 0 violations, each carrying a risk-side axis attribution (Figure~\ref{fig:real_audit}a).

\paragraph{Post-hoc frontier diagnosis, kept separate.}
\looseness=-2 \label{sec:frontier_diag}
Separately from the certificate's output, a post-hoc evaluation-set frontier diagnosis prices the frontier at $\betastarhat(.10) = .050$ (leg 1) and $.209$ (leg 2), both far below the floor $\beta = .60$, which lies outside both bootstrap intervals. Sweeping the risk level $\alpha \in \{.05, .10, .15, .20\}$ traces the full frontier-vs-floor map (Figure~\ref{fig:real_audit}a): the floor $\beta = .60$ is uncrossed (and outside every bootstrap interval) at \emph{every} tested $\alpha$ for both legs (even the stronger leg 2 reaches only $\betastarhat(.20) = .525$, CI $[.464, .597]$).
So the post-hoc evaluation-set frontier estimate stays below the floor across the tested operating plane, not only at the SLA point (per-$\alpha$ 500-resample bootstrap intervals; no simultaneous-coverage claim across the $\alpha$-grid or over threshold selection, and the closest case is leg~2 at $\alpha = .20$, whose interval ends at $.597$), consistent with frontier-infeasibility across the tested plane rather than a population proof of it.
Capability escalation is monotone for $\alpha \ge .10$ (leg 2 above leg 1) but moves the frontier without ever crossing the floor.
We keep the two outputs separate: the certificate refuses pre-deployment, while the frontier number is post-hoc evaluation-set diagnosis; the certificate does not prove the frontier, and we do not claim it does.

\begin{figure}[t]
    \centering
    \begin{subfigure}[b]{0.42\textwidth}
        \centering
        \includegraphics[width=\linewidth]{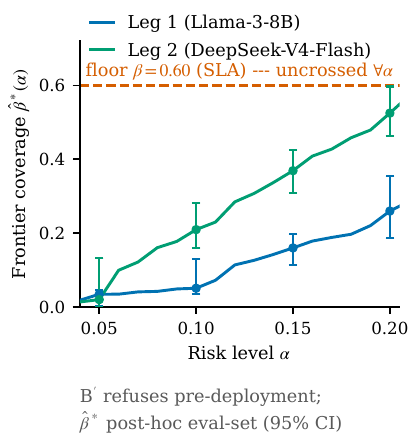}
        \caption{}
    \end{subfigure}
    \hfill
    \begin{subfigure}[b]{0.52\textwidth}
        \centering
        \includegraphics[width=\linewidth]{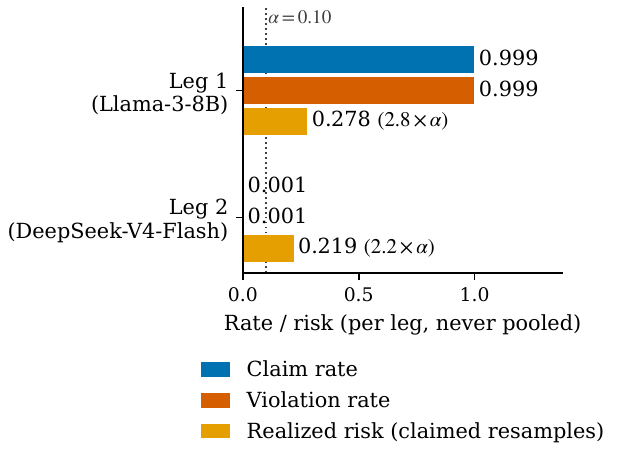}
        \caption{}
    \end{subfigure}
    \caption{\textbf{Real-workload feasibility audit under a pre-registered
    operational SLA scenario} (single-corpus SQuAD$\to$NewsQA domain shift,
    $N_{\mathrm{eval}} = 4{,}212$, F1 $\ge 0.5$ non-judge correctness,
    $\alpha = .10$, $\beta = .60$, $\delta = .05$; two capability legs on the same
    corpus; leg 2 is a conditional capability-ladder re-entry, not independent
    validation). (a)~The Model-B$'$ certificate's pre-deployment output is honest
    refusal, consuming no eval labels: 1{,}000/1{,}000 refusals per leg (2{,}000/2{,}000 total,
    0 violations), attributed to the risk axis. Separately, the post-hoc eval-set
    frontier diagnosis traces $\hat\beta^*(\alpha)$ over $\alpha \in [.05, .20]$
    (curves; 95\% bootstrap CIs at the marked points, 500 resamples over eval items):
    at the SLA $\alpha = .10$, $\hat\beta^* = .050$ $[.035, .130]$ for leg 1
    (Llama-3-8B) and $.209$ $[.160, .281]$ for leg 2 (DeepSeek-V4-Flash); the floor
    $\beta = .60$ is uncrossed (and outside every bootstrap CI) at every tested
    $\alpha$ for both legs (even leg 2 reaches only $\hat\beta^*(.20) = .525$), so
    capability escalation moves the frontier but never across the floor. (b)~Plug-in baseline, reported leg-split only (never pooled): leg 1 claims
    in 99.9\% of resamples and 99.9\% violate, with realized risk on claimed
    resamples $.278 = 2.8\times\alpha$; leg 2 claims in 0.1\% (a single claiming
    resample) and the lone claim violates, realized risk $.219 = 2.2\times\alpha$ at
    realized coverage $.614$. Dotted line: $\alpha = .10$.}
    \label{fig:real_audit}
\end{figure}

\paragraph{Plug-in contrast, leg-split only.}
\looseness=-2 \label{sec:plugin}
The plug-in baseline fails leg-specifically, and we report it leg-split only (Figure~\ref{fig:real_audit}b): leg 1 claims 99.9\%, violates 99.9\%, risk $.278 = 2.8 \times \alpha$; leg 2 claims 0.1\%, and the lone claim violates, risk $.219 = 2.2 \times \alpha$ at coverage $.614$.
The two legs fail in opposite modes (silent violation at scale versus claim collapse): precisely the information a pooled summary would destroy.

\paragraph{Refusal robustness: status only (Claim C6).}
\looseness=-1 \label{sec:sensitivity}
The leg-2 refusal is invariant over registered nuisance choices: $K \in \{8, 16, 32\} \times B \in \{5, 10, 20\}$ plus lock-quantile jitter, 11/11 no-cert, with the risk side failing 64/64 lattice points in every variant (Table~\ref{tab:sensitivity}, Appendix~\ref{app:sensitivity}).
This is refusal-status robustness only, not validation of the stratified-shift model: on the real workload misspecification-to-breakage was not reached, though a synthetic adversarial sweep exhibits the breakage boundary in $\cW_B$ but off the B$'$ $\cK$-measurable class (Appendix~\ref{app:revision}), and the nuisance necessity narrows to the unknown-$\eta$ edge (\S\ref{sec:notcertify}, Theorem~\ref{thm:t3-target}).

\paragraph{What an operator learns.}
\looseness=-2 Honest refusal is the pre-committed first-class outcome here: the attribution says \emph{which test} failed; the \emph{separate} post-hoc frontier diagnosis (\S\ref{sec:frontier_diag}) is what supports reading this SLA scenario as frontier-infeasible, and under that reading no $n$ or $m$ purchase helps (raise capability or renegotiate), while the leg gap prices exactly that.

\section{Discussion}
\label{sec:discussion}

\paragraph{On the assumptions.}
\looseness=-1 The map rests on load-bearing assumptions. Bounded-ratio covariate shift ($w \le B$) is what makes the floor estimable label-free while leaving the risk unidentified in Model-B; the resulting inconsistency over $\cW_B$ (Theorem~\ref{thm:modelB-incon}) is precisely why the map is stated \emph{across} models rather than as a single-model minimax law; the restriction that restores consistency (oracle weights in Model-A, a pre-registered $K$-cell partition in Model-B$'$) is the object of study, not a convenience. The regular frontier margin $(\kappa, s_0)$ is generic (Lemma~\ref{lem:genericity}) and, in its lattice realization $\mathrm{LR}_{\mathrm{margin}}(s)$ for the upper bounds (Definition~\ref{def:lrmargin}), isolates the local regime $s \le s_0 \wedge c_0\beta$ in which the budget law is linear; the stratified-shift partition $\cK$ is a dial, not a disguise (Remark~\ref{rmk:dial}), with its cost priced explicitly through the histogram nuisance.

\paragraph{Conditional validity and misspecification.}
\looseness=-1 Model-B$'$ validity (Theorem~\ref{thm:U1}) is conditional on the pre-registered $K$-cell model: the same epistemic status as the covariate-shift assumption itself. When $w$ is not $\cK$-measurable the guarantee degrades to its $\cK$-projection with an additive bias that covariates alone cannot estimate (Remark~\ref{rmk:misspec}); a \emph{deliberately adversarial} synthetic sweep maps this breakage boundary directly: a perturbation in $\cW_B$ but, by construction, off the B$'$ $\cK$-measurable class, invisible to the certificate, that nonetheless drives true violations once the receding frontier crosses the locked operating point (Appendix~\ref{app:revision}). The nuisance necessity is settled only on its $K$-free necessary core, with the remainder sharply localized: a $K$-free core is \emph{necessary} for lattice-valued certification in the known-$\eta$ submodel (and matched from below), the histogram $B^2K$ rate is proved sufficient but \emph{not proved} necessary, and exactly one edge, the unknown-$\eta$ case, remains open (Theorem~\ref{thm:t3-target}).

\paragraph{Practical reading.}
\looseness=-1 The certificate's first-class output is a certify-or-refuse decision carrying a deterministic, observable test-side attribution: a refusal says \emph{which test} failed, and for the data-starved axes Corollary~\ref{cor:phase} reads off which additional sampling extends the guarantee, while a separate frontier diagnosis signals when no purchase helps and the finer nuisance-limited label requires $(\kappa, s)$ as planning values (\S\ref{sec:notcertify}). The practical constants are harsh and reported in full, as one slice of a characterized envelope rather than an isolated point: B$'$ certifies 8/1{,}024 audit cells against oracle-A's 583/1{,}024, and at the audited slice needs $64\times$ the oracle-A budget to fully certify ($128\times$ at the first-certification onset); across the registered $K$-sweep this onset premium scales as $K^{0.955}$ (95\% CI $[0.755, 1.154]$). This premium is measured against the released two-family oracle kernel (Remark~\ref{rmk:releasedoracle}), which also differs from B$'$ in confidence allocation and floor construction, so it prices the whole estimated-weight construction and upper-bounds, rather than isolates, the histogram estimator's $\ell^1$ weight-recovery cost (Proposition~\ref{prop:I1}). The premium is \emph{not} a proved necessary certification cost (the established necessary core is $K$-free for lattice-valued certification, unknown-$\eta$ still open; Theorem~\ref{thm:t3-target}). Nor is it specific to the empirical-Bernstein risk test: the equal-construction alternative we tested, a SCoRE-inspired floor-augmented betting arm, is valid by Proposition~\ref{prop:score} and, in the matched-replication head-to-head, weakly dominates B$'$ in observed certify-frequency without enlarging the powered set while sharing its conservatism (an ordering observed in that run, with no paired test reported; it is a risk-test ablation rather than a source-faithful SCoRE implementation, \S\ref{sec:validity}). The certificate's contribution is therefore the lower-bound theory, the two-resource map, and the explicitly priced nuisance, not a finite-sample power advantage.

\paragraph{Limitations.}
\looseness=-2 (i) Two-axis support: at the registered 41-cell density the CI-separation criterion failed, but it is \emph{recovered} at higher design density (unanimous strict non-overlap by 635 cells, with the additive form $2.2\times$ better than the nearest alternative), confirming the registered sample-limited reading. The per-axis exponent-containment sub-test remains the finer residual check (Appendix~\ref{app:revision}).
(ii) Contour constants do not transfer across shift intensities (rate-shape and ordering only; the drift is systematic and monotone in intensity, Appendix~\ref{app:revision}).
(iii) Localized-geometry evidence is restricted to B$'$ family 1 plus qualified family-2 within-pair structure.
(iv) B$'$ validity is conditional on the pre-registered shift model and off-partition robustness is not claimed: the real workload shows refusal-status invariance only, and the synthetic breakage sweep above (Appendix~\ref{app:revision}) is a sensitivity characterization, not a guarantee.
(v) The real-workload scope is a single-corpus domain-shift audit with no positive certification, with the post-hoc evaluation-set frontier estimate below the floor across the whole tested operating plane ($\alpha \le .20$, both legs; per-$\alpha$ bootstrap, not a population proof).
(vi) The implementable certificate is far more sample-hungry than oracle weights: those practical constants (reported in full above) are harsh but fully characterized, one slice of an envelope rather than an isolated point (Appendix~\ref{app:revision}).

\section{Conclusion}
\label{sec:conclusion}

\looseness=-1 A hard coverage floor is a statistical object in its own right: it makes the feasibility frontier operational, splits certification cost, in the model-split sense of Theorem~\ref{thm:law}, into labeled-source and unlabeled-target resources, and surfaces a localized accepted-region functional (not global ESS) that drives the upper-bound variance proxy and organizes the lower-bound hard-slice geometry. This correspondence rests on the matched bound structure and the family-1 experiments, and a fixed-ESS separation theorem is left open. Certification with honest refusal is implementable under a pre-registered stratified-shift model: the formal-guarantee arms are valid by theorem, with 0 violations across the 591 cells they certify on the synthetic grid; on the real workload its output is honest refusal (2{,}000/2{,}000), a feasibility verdict, not validity evidence.

\paragraph{Outlook.}
\looseness=-1 The pre-registered empirical program (Appendix~\ref{app:futurework}) is largely complete, and two directions remain. First, beyond the clean B$'$ family-1 evidence, the localized-geometry comparison on family 2 remains censored; an uncensored family-2 rerun with a fixed-range-constant oracle test would test whether the localized accepted-region functional governs the rate beyond B$'$ family~1. Second, the nuisance theory is settled only on one side: the $K$-free necessary core is proved for lattice-valued certification in the \emph{known-$\eta$} submodel, and no lower bound forces a separate source-covariate split for estimating $p$ at the population level (Theorem~\ref{thm:t3-target}); whether \emph{unknown-$\eta$} Model-B$'$ admits a $K$-free certificate, or instead forces $K$ back through the unknown-$\eta$ source$\leftrightarrow$target alignment, is the central open edge.

\label{end:mainbody}

\section*{Acknowledgment}
The research was supported by the Young Doctoral Talent Incubation Program of the Second Affiliated Hospital, Army Medical University (2025YQB034).

\bibliography{references}
\bibliographystyle{plainnat}

\newpage
\appendix
\setcounter{figure}{0}
\renewcommand{\thefigure}{A\arabic{figure}}
\setcounter{table}{0}
\renewcommand{\thetable}{A\arabic{table}}
\section{Frontier Characterization and Regular-Margin Lemmas}
\label{app:frontier}

This appendix proves the frontier-level statements behind \S\ref{sec:setup}: the genericity of the regular margin, the linear budget law, the population frontier characterization, and the lattice bookkeeping packaged by $\mathrm{LR}_{\mathrm{margin}}$.

\paragraph{Notation bridge.}
The main text writes $\betastar(\alpha, Q)$ for the feasibility frontier defined through the budget functional in \S\ref{sec:setup}; this is the \emph{relaxed} (fractional-knapsack) frontier: the supremum of $\CQ$ over all \emph{randomized} acceptance rules $S : \cX \to [0,1]$ satisfying the ratio constraint; maximizing the linear functional $\CQ = \E_Q[S]$ subject to the single linear budget constraint $\E_Q[S(\eta - \alpha)] \le 0$ is a generalized Neyman--Pearson problem, solved by the increasing-rearrangement threshold rule \citep{dantzig1951fundamental}.
It coincides with the supremum over deterministic $\{0,1\}$ selectors when $Q_X$ is nonatomic on the score level sets (no integrality gap); when $Q_X$ has atoms straddling the frontier the deterministic selector frontier can be strictly smaller, since a fractional sliver of an atom is not deterministically attainable (the sharp-threshold construction of Proposition~\ref{prop:overlap} is exactly such an atomic instance).
This relaxation gap never affects the achievability theorems: their coverage slack $s = \betastar - \beta$ is measured against the \emph{relaxed} frontier $\betastar$, but the certificate uses only the margin-bearing lattice \emph{witness} $\lambda_s$ (never the equality $\betastar_\lat = \betastar$), so a residual $\betastar_\lat \le \betastar$ gap is harmless; the witness-only role of $\mathrm{LR}_{\mathrm{margin}}$ is spelled out below.
For a fixed lattice $\lat$ the operative object is the \emph{lattice frontier}
\begin{equation*}
\betastar_{\lat}(\alpha, Q) \;=\; \sup\{\, \CQ(\lambda) : \lambda \in \lat, \; \RQ(\lambda) \le \alpha \,\} \qquad (\sup \emptyset := 0),
\end{equation*}
and $\betastar_{\lat} \le \betastar$ always. Equality, up to grid resolution, is a lattice-design property requiring \emph{two} conditions: the threshold family lies on a score whose order agrees with $\eta$ on the relevant range, and its \emph{coverage} grid is both sufficiently fine \emph{and} positioned to bracket the operative frontier neighborhood. Quantitatively, if the coverage mesh $h_C$ (the largest gap between attainable lattice coverages) satisfies $h_C < s/8$ \emph{and} some attainable coverage lies in $[\betastar - h_C, \betastar]$ (the grid reaches the frontier, not merely the interior neighborhood), then $0 \le \betastar - \betastar_{\lat} \le h_C$; a fine mesh that brackets only $\betastar - s/2$ controls the \emph{witness} (below) but not $\betastar_{\lat}$ itself; a grid whose attainable coverages stop short of $\betastar$ (say at $\betastar - 2h_C$) leaves $\betastar - \betastar_{\lat} \ge 2h_C$ despite $h_C < s/8$.
The power theorems invoke $\mathrm{LR}_{\mathrm{margin}}(s)$ only as the explicit existence of a margin-bearing lattice witness $\lambda_s$; the condition alone does \emph{not} imply $\betastar_{\lat} = \betastar$, and a coarse lattice can satisfy it while still having $\betastar_{\lat} < \betastar$ (e.g.\ a single in-window witness at coverage $\beta + s/4$ with the next feasible point far below $\betastar$).
We keep $\betastar$ (relaxed) and $\betastar_\lat$ (lattice) notationally distinct, write $\Delta_\lat := \betastar - \betastar_\lat \ge 0$ for the discretization gap (so $0 \le \Delta_\lat \le h_C$ under the bracketing condition above, and $\Delta_\lat = 0$ up to grid resolution when both lattice-design conditions hold), and write $s = \betastar - \beta$ for the relaxed coverage slack in the power statements; in any construction that needs the lattice slack (notably the inconsistency lattice of Theorem~\ref{thm:modelB-incon}), we verify the lattice realizes the relevant relaxed frontier explicitly (i.e.\ $\Delta_\lat$ is controlled).

\subsection{Genericity of the regular margin}

\begin{lemma}[Genericity]
\label{lem:genericity}
Suppose $\eta^*_Q(u) < \alpha$ on a set of positive measure and $\betastar < 1$.
Then $\budget$ is strictly positive on an interval $(0, c_\alpha]$ (where $\eta^*_Q < \alpha$), strictly decreasing wherever $\eta^*_Q > \alpha$, and at the zero-crossing $\betastar$ the essential left limit satisfies $\eta^*_Q({\betastar}^{-}) > \alpha$ strictly (equivalently $\eta^*_Q \ge \alpha + \kappa$ a.e.\ on a left-neighborhood of $\betastar$), \emph{unless} $\budget \equiv 0$ on an interval (a degenerate tie $\eta^*_Q \equiv \alpha$).
Hence the regular frontier margin holds except in the degenerate boundary case.
\end{lemma}

\begin{proof}
$\budget'(c) = \alpha - \eta^*_Q(c)$ for a.e.\ $c$.
By assumption $\budget(c_\alpha) = \int_0^{c_\alpha} (\alpha - \eta^*_Q) \, du > 0$ for $c_\alpha$ the supremum of the region where $\eta^*_Q < \alpha$, and $\betastar > c_\alpha$.
Since $\budget(\betastar) = 0$ with $\budget(c_\alpha) > 0$, the derivative must be negative somewhere on $(c_\alpha, \betastar)$, i.e.\ $\eta^*_Q(c) > \alpha$ for some $c$ in that interval; monotonicity of the rearrangement then gives $\eta^*_Q(u) \ge \eta^*_Q(c) > \alpha$ for a.e.\ $u \in (c, \betastar)$. Since $\eta^*_Q$ is defined only up to a.e.\ equivalence, we read the crossing through this a.e.\ left-neighborhood margin (equivalently the essential left limit $\eta^*_Q({\betastar}^{-})$), not a single point value.
Strictness at the crossing fails only if $\eta^*_Q = \alpha$ on a neighborhood of $\betastar$, which is the excluded degenerate tie.
Concretely, fixing such a $c \in (c_\alpha, \betastar)$ with $\eta^*_Q(c) > \alpha$ and setting $\kappa := \eta^*_Q(c) - \alpha > 0$ and $s_0 := \betastar - c > 0$, monotonicity of $\eta^*_Q$ gives $\eta^*_Q(u) \ge \eta^*_Q(c) = \alpha + \kappa$ for a.e.\ $u \in [\betastar - s_0, \betastar)$: an explicit regular margin $(\kappa, s_0)$, obtained from the rearrangement's monotonicity alone (no continuity of $\eta^*_Q$ at $\betastar$ is needed).
\end{proof}

\subsection{The linear budget law}

For a world with regular margin $(\kappa, s_0)$ and $0 \le s \le s_0$,
\begin{equation}
\label{eq:budgetlaw_app}
\budget(\betastar - s) \;=\; \budget(\betastar - s) - \budget(\betastar) \;=\; \int_{\betastar - s}^{\betastar} \big(\eta^*_Q(u) - \alpha\big)\, du \;\ge\; \kappa s,
\end{equation}
using $\budget(\betastar) = 0$ and $\eta^*_Q \ge \alpha + \kappa$ a.e.\ on the integration range; this is the linear budget law quoted in \S\ref{sec:setup}.
The matching upper bound is $\budget(\betastar - s) \le \bar K s$ with $\bar K = \sup_{u \le \betastar} (\eta^*_Q(u) - \alpha)_+ \le 1 - \alpha$.
The risk budget available at coverage slack $s$ is therefore linear in $s$: the single fact that fixes the exponent $-2$ on the labeled axis (Appendix~\ref{app:lower}) and supplies the achievability margins (Appendix~\ref{app:modelA}).
The endpoint-only condition $\eta^*_Q(\betastar) - \alpha \ge \kappa$ would \emph{not} suffice: monotonicity of $\eta^*_Q$ runs the wrong way, so the left-neighborhood form in the definition is load-bearing.

\subsection{The margin-bearing lattice assumption}

The assumption $\mathrm{LR}_{\mathrm{margin}}(s)$ used by Theorems~\ref{thm:law}, \ref{thm:U2}, and~\ref{thm:modelA} reads: there exists $\lambda_s \in \lat$ with
\begin{equation*}
\CQ(\lambda_s) \in [\beta + s/4, \; \betastar - s/4]
\qquad\text{and}\qquad
-\GQ(\lambda_s) \;=\; \E_Q\big[S_{\lambda_s}(\alpha - \eta)\big] \;\ge\; \kappa s / 8 .
\end{equation*}
A coverage-only form ($\exists \lambda$ with the right coverage and $\RQ(\lambda) \le \alpha$) is \emph{not} sufficient for the power proofs: an arbitrary lattice point at that coverage may have risk budget arbitrarily close to $0$, because the budget law~\eqref{eq:budgetlaw_app} bounds the \emph{rearranged-optimal} set's budget, not an arbitrary $\lambda$'s.
$\mathrm{LR}_{\mathrm{margin}}(s)$ holds automatically for risk-score-aligned nested-threshold lattices \emph{under a regular margin $(\kappa, s_0)$ with $s \le s_0$, whose grid is finer than $s/8$ and whose attainable coverages bracket $\betastar - s/2$ on both sides}, namely thresholds on a score whose order agrees with $\eta$ on the relevant range (so each accepting set is a risk \emph{sublevel} set $\{\eta < t_\lambda\}$, equivalently a rearranged-optimal set), because there the lattice's acceptance sets \emph{are} the rearranged-optimal sets and inherit the budget $\budget$ exactly.
Concretely, the aligned threshold whose coverage is nearest $\betastar - s/2$ lands in the window $[\betastar - 3s/4,\, \betastar - s/4] = [\beta + s/4,\, \betastar - s/4]$ (width $s/2$, which exceeds the $<s/8$ coverage mesh, so once the aligned lattice's coverage range brackets $\betastar - s/2$, which is automatic for a grid spanning the operative frontier neighborhood, such a threshold exists); writing its slack $s' = \betastar - \CQ(\lambda_s) \in [s/4,\, 3s/4] \subseteq [0, s_0]$, the budget law~\eqref{eq:budgetlaw_app} gives $-\GQ(\lambda_s) = \budget(\betastar - s') \ge \kappa s' \ge \kappa s/4 \ge \kappa s/8$, so both clauses of $\mathrm{LR}_{\mathrm{margin}}(s)$ hold with room to spare.
For misaligned scores the shortfall is precisely a representation gap, and certification can fail for lattice-design reasons rather than statistical ones.

\subsection{Frontier characterization}

\begin{theorem}[Frontier characterization]
\label{thm:frontier}
Fix a world $(P, Q_X, \eta) \in \cW_B$, $\alpha \in (0,1)$, and $\beta > 0$.
\begin{enumerate}
    \item[(a)] If $\beta > \betastar_{\lat}(\alpha, Q)$, then every $\lambda \in \lat$ violates $\RQ(\lambda) \le \alpha$ or $\CQ(\lambda) \ge \beta$; consequently any $(\alpha, \beta, \delta)$-valid procedure satisfies $\Prob(\cA \ne \nocert) \le \delta$ in that world.
    \item[(b)] If $0 < \beta \le \betastar_{\lat}(\alpha, Q)$ and the supremum is attained by some $\lambda$ with $\CQ(\lambda) > 0$, at least one certifiable $\lambda$ exists.
\end{enumerate}
The hypothesis $\beta > 0$ matches the certify convention of \S\ref{sec:setup}: a coverage-$0$ policy has $\RQ = +\infty$ and is never feasible, so it is excluded from the supremum defining $\betastar_{\lat}$, and part (b)'s positive-coverage attainment hypothesis already covers the $\beta = 0$ boundary.
\end{theorem}

\begin{proof}
(a) Take any $\lambda$ with $\CQ(\lambda) \ge \beta$; since $\beta > 0$ this excludes the $\CQ(\lambda) = 0$ convention case.
Then $\CQ(\lambda) \ge \beta > \betastar_{\lat}$ implies $\lambda$ is not in the feasible set $\{\RQ \le \alpha\}$, by definition of the supremum; hence $\RQ(\lambda) > \alpha$.
So every certified output is incorrect in this world, and validity caps the certification probability at $\delta$.
(b) Immediate from the definitions.
\end{proof}

\begin{remark}[Population vs.\ statistical impossibility]
\label{rmk:popvsstat}
Theorem~\ref{thm:frontier}(a) is population-level: above the frontier there is nothing to certify.
The \emph{statistical} impossibility (that near the frontier no procedure can know which side it is on) is Theorem~\ref{thm:lower}, and the two are deliberately separated.
\end{remark}

\begin{remark}[Discrete lattices break the linear map within cells]
\label{rmk:lattice}
Between adjacent lattice coverages the certifiable slack is piecewise constant, so the budget available to a fixed lattice cell does not scale linearly in $s$ within the cell: positive-budget cells behave as plateaus, and zero-margin cells drive the required $n$ to infinity.
This is a lattice-design diagnostic rather than a pathology of the map; it generates the constant-budget plateau lower bound of Lemma~\ref{lem:plateau} (Appendix~\ref{app:exponent}) and is observed empirically (Appendix~\ref{app:experiments}).
\end{remark}

\section{The Two-Axis Lower Bound}
\label{app:lower}

This appendix proves the formal version of Theorem~\ref{thm:law}, Claim 1: the Model-B two-axis lower bound, each axis witnessed by a Le~Cam two-point pair at bounded ratio, and the additive combination obtained inside a single construction family per $(n, m)$ as two simultaneously witnessed necessary conditions (max form; the additive display follows up to a factor of $2$, Appendix~\ref{app:additive}), not a single joint KL summation.

\begin{theorem}[Two-axis minimax lower bound, Model-B]
\label{thm:lower}
There exist absolute constants $c_1, c_2 > 0$ and, for each $B \ge 2$, $\alpha \in [0.2, 0.9]$, $\beta \in (0, 0.6]$, a family of worlds in $\cW_B$ with regular margin $(\kappa, s_0)$, $\kappa \in (0, \tfrac{1}{2}\min\{\alpha, 1-\alpha\})$ (the $\kappa < \alpha/2$ bound makes the margin $s_0 > 0$, so the family is non-vacuous; the $\kappa \le (1-\alpha)/2$ bound keeps the Bernoulli levels in range), Bernoulli conditional losses $L \mid X \sim \Bern(\eta(X))$, and (for each $s$ in the range below, chosen after $s$), a lattice $\lat$ given by a finite nested grid of score thresholds with mesh finer than $s/4$ containing the construction's block boundaries, such that the following holds in Model-B.
Let $s \in (0, s_0]$ with additionally $s \le c_0 \beta$ and $\beta + s \le 1$ (the local-regime and valid-coverage conditions; for this family $c_0 = s_0/\beta$ depends on $(\alpha, \kappa)$, so $s \le c_0\beta$ is already implied by $s \le s_0$ and is displayed only to match the standing local-regime form of \S\ref{sec:setup}), and suppose a procedure $\cA$ is $(\alpha, \beta, \delta)$-valid on $\cW_B$ with $\delta \le 1/8$.
If
\begin{equation*}
s \;\le\; c_1\, \frac{1}{\kappa} \sqrt{\frac{\beta}{n}} \quad \text{($n$-axis)}
\qquad\text{or}\qquad
s \;\le\; c_2\, \sqrt{\frac{\beta}{m}} \quad \text{($m$-axis)},
\end{equation*}
then there is a world with coverage slack $\ge s$ on which $\Prob(\cA = \nocert) \ge 1/2$; in fact the slack is \emph{exactly} $s$ for the displayed construction families, whose relaxed frontier is $\betastar = \beta + s$ (Appendices~\ref{app:naxis}, \ref{app:maxis}), so a certifiable $\lambda$ exists with room to spare.
Consequently the minimax no-cert window has width at least the maximum of the two expressions (hence at least half their sum), and certifying at slack $s$ requires both $n \gtrsim \beta/(\kappa s)^2$ and $m \gtrsim \beta/s^2$.
\end{theorem}

The $\sqrt{\beta}$ factor on the $m$-axis is intrinsic to the construction family (the perturbable safe block has mass $\asymp \beta$); it makes both axes vanish as $\beta \to 0$ (Remark~\ref{rmk:beta0}) and exactly matches the variance-aware floor rate $\sqrt{\beta(1-\beta)/m}$ on the achievability side (Theorem~\ref{thm:modelA}).

\subsection{The \texorpdfstring{$n$}{n}-axis pair: labels are the only channel}
\label{app:naxis}

\paragraph{Construction.}
Take $\cX = [0,1]$ and $P_X = Q_X = \mathrm{Unif}[0,1]$, so $w \equiv 1 \le B$: the shiftless instance lies inside $\cW_B$, and the risk axis does not even need shift.
Losses are Bernoulli, $L \mid X \sim \Bern(\eta(X))$; the score is the identity and $\lat$ consists of thresholds $\{\text{accept } x \le t\}$.
Fix $g = \beta/2$ and $b = \kappa \le (1-\alpha)/2$, and set the deep-safe margin exactly by the frontier requirement:
\begin{equation*}
a \;=\; b\Big(1 + \frac{2s}{\beta}\Big)
\quad\Big(\text{so } \betastar_1 = g + \frac{a g}{b} = \beta + s\Big),
\qquad
a \le \alpha/2 \;\text{ for } s \le s_0 := \frac{\beta}{2}\Big(\frac{\alpha}{2b} - 1\Big)_{\!+}.
\end{equation*}
All construction constraints ($a \le \alpha/2$, the $m$-axis $\varepsilon \le \varepsilon_0$, the cap $\Delta \le (1-\alpha)/2$, and $\beta + s \le 1$ with the $m$-axis $1 - g' \ge 1/4$) hold once $s \le s_{\mathrm{reg}} := \min\{\beta/8,\ \tfrac{\beta}{2}(\tfrac{\alpha}{2b}-1)_+,\ \beta(1-\alpha)/(8b)\}$ (recall $b = \kappa$), with local-regime constant $c_0 = \min\{1/8,\ s_{\mathrm{reg}}/\beta\} = s_{\mathrm{reg}}/\beta$, the first term never binding since $s_{\mathrm{reg}} \le \beta/8$ (so $\varepsilon \le 2s \le \beta/4$ and $g' = 3/4$ is admissible); we use $s_0 := s_{\mathrm{reg}} \le \beta/8 < 1$ as the regular-margin radius.
World $W_1$:
\begin{equation*}
\eta_1(x) =
\begin{cases}
\alpha - a & x \in [0, g) \quad \text{(deep-safe block $G$)}\\
\alpha + b & x \in [g, 1] \quad \text{(hot block $H$)},
\end{cases}
\end{equation*}
so $\budget_1(c) = a c$ on $[0, g]$, $\budget_1(g + u) = a g - b u$, the frontier is $\betastar_1 = \beta + s$, the regular margin $(\kappa, s_0)$ holds in left-neighborhood form on all of $(g, \betastar_1]$ (there $\eta^* \equiv \alpha + b$), and in this construction
\begin{equation*}
\budget_1(\beta) = g(a - b) = b s = \kappa s,
\qquad
\bar K = \sup_{u \le \betastar_1} (\eta^*_1(u) - \alpha)_+ = b = \kappa \quad \text{(exact)}.
\end{equation*}
All Bernoulli levels stay in $[0.05, 0.95]$ by the constraints $\alpha \in [0.2, 0.9]$, $a \le \alpha/2$, $b \le (1-\alpha)/2$; for the fixed $s$ a finite nested-threshold grid with mesh $< s/4$ containing the block boundaries $\{g,\, \beta,\, \beta + s\}$ realizes the same frontier and budget (the blocks are intervals, so a finer grid only refines the rearrangement), discharging the resolution condition (we display continuum thresholds for brevity); and $\beta + s \le 1$ is required so that $H$ can host the frontier.

World $W_2$ modifies only the conditional risk on the inframarginal slice $U = [0, \rho) \subseteq G$ of mass $\rho = g = \beta/2$ (the whole safe block):
\begin{equation*}
\eta_2 = \eta_1 + \Delta\, \mathbf{1}_U,
\qquad
\Delta = \frac{2\, \budget_1(\beta)}{\rho} = \frac{2 \kappa s}{\beta/2} = \frac{4 \kappa s}{\beta},
\end{equation*}
with the range check $\eta_2|_U = \alpha - a + \Delta \le 1 \iff \Delta \le 1 - \alpha + a$, which holds whenever $s \le \beta(1-\alpha)/(4\kappa)$, folded into $s_0$; we additionally cap $\Delta \le (1-\alpha)/2$ so that all Bernoulli levels lie in $[0.1, 0.95]$, tightening $s_0$ by a constant.

\paragraph{$W_2$ is globally infeasible at $(\alpha, \beta)$.}
For any threshold $\lambda_t$ with $\CQ(\lambda_t) = t \ge \beta > g$, the slice $U$ is fully accepted, so the budget drops uniformly:
$\budget_2(t) = \budget_1(t) - \Delta \rho = \budget_1(t) - 2\budget_1(\beta) \le -\budget_1(\beta) < 0$,
since $\budget_1(t) \le \budget_1(\beta)$ for $t \in [\beta, \betastar_1]$ ($\budget_1$ is nonincreasing on $[g, \betastar_1]$) and $\budget_1(t) < 0 \le \budget_1(\beta)$ for $t > \betastar_1$.
Hence every $t \ge \beta$ has $\RQ(\lambda_t) > \alpha$ under $\eta_2$, and every $t < \beta$ violates the floor: all of $\lat$ is invalid, so validity forces $\Prob_{W_2}(\cA \ne \nocert) \le \delta$.

\paragraph{Indistinguishability.}
The two worlds share $P_X, Q_X$ (hence the unlabeled sample's law is identical) and differ only in $L \mid X \in U$.
Per labeled sample,
\begin{equation*}
\KL(W_1 \,\|\, W_2)
= \E_{P_X}\!\big[\KL\big(\Bern(\eta_1(X)) \,\|\, \Bern(\eta_2(X))\big)\big]
\;\le\; \rho \cdot \frac{\Delta^2}{\big(\eta_2(1-\eta_2)\big)\big|_{\min}}
\;\le\; C_0\, \rho\, \Delta^2,
\end{equation*}
with $C_0$ absolute because all Bernoulli levels lie in $[0.1, 0.95]$ under the constraints above.
Plugging $\rho = \beta/2$ and $\Delta = 4\kappa s/\beta$:
\begin{equation*}
\KL\big(W_1^{\otimes n} \,\|\, W_2^{\otimes n}\big) \;\le\; 8 C_0\, n\, \frac{(\kappa s)^2}{\beta}.
\end{equation*}
If $s \le c_1 \kappa^{-1}\sqrt{\beta/n}$ with $c_1 = (64 C_0)^{-1/2}$, then $\KL \le 8 C_0 c_1^2 = 1/8$, hence by Pinsker $\TV(W_1^{(n,m)}, W_2^{(n,m)}) \le 1/4$.

\paragraph{Le~Cam step.}
Suppose $\Prob_{W_1}(\cA \ne \nocert) > 1/2$.
Then
\begin{equation*}
\Prob_{W_2}(\cA \ne \nocert) \;\ge\; \Prob_{W_1}(\cA \ne \nocert) - \TV \;>\; \tfrac{1}{2} - \tfrac{1}{4} = \tfrac{1}{4} > \delta,
\end{equation*}
contradicting validity in $W_2$, where any certification is invalid.
Hence $\Prob_{W_1}(\cA = \nocert) \ge 1/2$.
\hfill$\square$\;($n$-axis)

The adversary's optimum also fixes the localized functional: the KL cost per unit of frontier flip is minimized by spreading $\Delta$ over the largest available inframarginal \emph{accepted} mass ($\rho \le \beta$), and in the shifted variant by placing the slice where $w$ is large; both choices are properties of the accepted region, which is why the bound's variance proxy is of $\E_P[w^2 \Sacc]$-type rather than global (Remark~\ref{rmk:localized}).

\subsection{The \texorpdfstring{$m$}{m}-axis pair: unlabeled draws are the only channel}
\label{app:maxis}

\paragraph{Construction.}
Fix $\eta = \eta_1$ as above and $P_X = \mathrm{Unif}[0,1]$.
Both worlds share $(P_X, \eta)$ (hence the labeled sample's law is identical) and differ only in $Q_X$:
\begin{itemize}
    \item $Q_1 = P_X$ (so $w_1 \equiv 1$), frontier $\betastar_1 = \beta + s$ as before;
    \item $Q_2$ moves mass $\varepsilon$ from $G$ to a far hot region: $dQ_2/dP_X = 1 - \varepsilon/g$ on $G$, $= 1 + \varepsilon/(1 - g')$ on $H' = [g', 1]$ for a fixed $g' \in (\betastar_1, 1)$ with $1 - g' \ge 1/4$, and $= 1$ elsewhere.
\end{itemize}
Membership in $\cW_B$ is explicit: $w_2 \in [1/B, B]$ requires $\varepsilon \le g(1 - 1/B)$ (lower side, on $G$) and $\varepsilon \le (B-1)(1-g')$ (upper side, on $H'$); both hold for $\varepsilon \le \varepsilon_0 := \min\{g(1 - 1/B),\, (B-1)/4,\, g/2\}$.

Under $Q_2$ the safe mass shrinks and the hot mass grows: the rearranged budget satisfies $\budget_2(c) \le \budget_1(c) - a\varepsilon$ for $c \ge g$ (the lost $G$-mass $\varepsilon$ had margin $a$; mass added beyond $g'$ only hurts).
Choosing $\varepsilon = 2\bar K s/a = 2bs/a \le 2s$ makes $\budget_2(c) < 0$ for all $c \ge \beta$, while $W_1$ has slack $s$.

\paragraph{$W_2$ invalidity, made explicit.}
Under $Q_2$, mass was removed only below $g < \beta$ and added only beyond $g' > \betastar_1 > \beta$; hence for $t < g'$ the $Q_2$-coverage of $[0, t]$ is at most its $Q_1$-coverage, so any threshold $\lambda_t$ with $C_{Q_2}(\lambda_t) \ge \beta$ has $t \ge \beta$.
For all such $t$, $\budget_2(t) < 0$ gives $R_{Q_2}(\lambda_t) > \alpha$; thresholds with $C_{Q_2} < \beta$ violate the floor by definition.
Every $\lambda \in \lat$ is invalid at $(\alpha, \beta)$ under $W_2$, so validity forces $\Prob_{W_2}(\cA \ne \nocert) \le \delta$.

\paragraph{Indistinguishability.}
The labeled channel contributes zero KL (identical laws); in Model-B the procedure does not receive $w$, and this is exactly where Models A and B part ways.
For the unlabeled channel, $dQ_2/dQ_1 \in [1 - \varepsilon/g,\, 1 + 4\varepsilon]$ is bounded below by $1/2$ (since $\varepsilon \le g/2$), so
\begin{align*}
\KL\big(Q_1^{\otimes m} \,\|\, Q_2^{\otimes m}\big)
\;\le\; m\, \chi^2(Q_1 \,\|\, Q_2)
&\;\le\; 2m \Big[\frac{(\varepsilon/g)^2\, g}{1} + \frac{(\varepsilon/(1-g'))^2 (1-g')}{1}\Big]\\
&\;\le\; 2 m \varepsilon^2 \Big(\frac{2}{\beta} + 4\Big)
\;\le\; \frac{C'\, m\, \varepsilon^2}{\beta},
\end{align*}
with $C' \le 2(2 + 4\beta) \le 10$ absolute for $\beta \le 0.6$.
With $\varepsilon \le 2s$ and $s \le c_2 \sqrt{\beta/m}$ for $c_2 = (320)^{-1/2}$, we get $\KL \le 1/8$, $\TV \le 1/4$, and the same Le~Cam step gives $\Prob_{W_1}(\cA = \nocert) \ge 1/2$.
The $\sqrt{\beta}$ in the $m$-axis rate is intrinsic to this family: the perturbable safe block has mass $g = \beta/2$, so the $\chi^2$ cost of moving $\varepsilon$ of target mass scales as $\varepsilon^2/\beta$.
\hfill$\square$\;($m$-axis)

\subsection{Additive combination inside one construction family}
\label{app:additive}

A single procedure facing the class must survive both pairs; both pairs share the base world $W_1$, so the family $\{W_1, W_2^{(n)}, W_2^{(m)}\}$, indexed by $(n, m)$ through the choices of $\Delta$ and $\varepsilon$, witnesses both axes simultaneously.
The two bounds hold simultaneously, so the no-cert window is at least the maximum of the two axis widths, hence at least half their sum:
\begin{align*}
s \;\le\; \tfrac{1}{2}\Big(c_1 \kappa^{-1}\sqrt{\beta/n} + c_2\sqrt{\beta/m}\Big)
\;&\Longrightarrow\;
s \le \max\Big(c_1 \kappa^{-1}\sqrt{\beta/n},\; c_2\sqrt{\beta/m}\Big)\\
&\Longrightarrow\;
\Prob_{W_1}(\cA = \nocert) \ge \tfrac{1}{2},
\end{align*}
which is the additive form of Theorem~\ref{thm:law}, Claim 1.
Each axis also holds separately on the general class, since each pair lies in $\cW_B$ with $w \in [1/B, B]$ bounded by construction and $Q$-support contained in $P$-support everywhere: nothing here touches positivity.

\begin{remark}[The two main axes match \emph{within} Model-B$'$]
\label{rmk:bprime-match}
\looseness=-1 Both lower-bound families are \emph{weight-simple}: the $n$-axis pair has $w \equiv 1$ (Appendix~\ref{app:naxis}) and the $m$-axis pair has $w$ piecewise-constant on the fixed three-piece partition $\{G,\, [g, g'),\, H'\}$ (Appendix~\ref{app:maxis}). Hence for any pre-registered partition $\cK$ refining these pieces every world here is $\cK$-measurable and so lies in the \emph{sample-independent} Model-B$'$ class; a nested-threshold lattice cutting at cell boundaries gives union-of-cells acceptance, and Assumption~\ref{ass:bprime}'s $p_{\min}$ bound (so $\cK$ may not be arbitrarily fine) is the certificate's upper-bound feasibility condition, separate from class membership. Knowing $\cK$ does not defeat the pairs: the $n$-axis pair differs only in the labels and the $m$-axis pair only in $Q_X$, so a Model-B$'$ procedure's per-cell statistics are measurable functions of the \emph{same} $(n, m)$ draws and inherit the same $\TV \le 1/4$ bound below $n \asymp \beta/(\kappa s)^2$ and $m \asymp \beta/s^2$. Therefore, on the two \emph{main} axes the lower bound holds \emph{within} Model-B$'$ (not only cross-model): combined with the achievability of Theorem~\ref{thm:U2}, under its standing hypotheses ($\mathrm{LR}_{\mathrm{margin}}$, a regular margin $(\kappa, s_0)$ with $s$ in the local regime, and the nuisance-budget condition $\rho_{\lambda_s} \le \kappa s/32$), the risk and floor axes are same-model \emph{rate}-matched up to constants and logarithms over the compatible-$\cK$ subclass \emph{as totals} $(n, m)$ (a pointwise, instance-dependent total-rate correspondence, not a uniform minimax over the class), in the regime where the weight splits do not dominate (so $n \asymp n_r$, $m \asymp m_f$; Corollary~\ref{cor:bprime-minimax}); on the algorithm's test sub-blocks $(n_r, m_f)$ alone this is a rate correspondence, not a sub-block lower bound, since a B$'$ procedure may reuse the weight-split draws for testing. The lone unmatched axis is the histogram nuisance $B^2 K$, whose necessity in the unknown-$\eta$ case is the single open edge (Theorem~\ref{thm:t3-target}, Appendix~\ref{app:necessity}). This narrows the cross-model gap of Theorem~\ref{thm:law} to the nuisance axis (on the totals, in that regime), not the two resource axes.
\end{remark}

\subsection{No matching Model-B upper bound: inconsistency over \texorpdfstring{$\cW_B$}{W\_B}}
\label{app:modelB-incon}

Theorem~\ref{thm:lower} lower-bounds Model-B at the same two-axis rate that the oracle-weight certificate of Theorem~\ref{thm:modelA} attains, so the \emph{rate} matches across the models. It is natural to ask whether the match is also \emph{within} Model-B: is there an unknown-weight procedure (knowing only $w \le B$, estimating everything from $n$ source labels and $m$ unlabeled target covariates) that certifies at the Model-A rate uniformly over $\cW_B$? Theorem~\ref{thm:modelB-incon} answers in the negative, and decisively: over the full bounded-ratio class, Model-B is not merely rate-separated but \emph{inconsistent} at the operating point $\alpha=\beta=1/2$; its minimax certifiable slack there does not vanish at any sample size. This is why the implementable certificate of \S\ref{sec:setup} restricts to a pre-registered finite shift model (Model-B$'$): a complexity restriction on the weights is necessary for consistency, not a convenience.

\begin{theorem}[Model-B inconsistency]
\label{thm:modelB-incon}
\looseness=-1 Fix $\alpha = \beta = 1/2$, $B \ge 2$, $\gamma = 1/2$. There are an absolute constant $s_* = 1/4$ and two nested thresholds, $\lambda_\tau$ accepting $[0, 9/16]$ and $\lambda_{t^\star}$ accepting $[0, 3/4]$, both independent of $b$, $n$, and $m$, such that for every $b \in (0, 1/4]$, every pair $(n, m)$ of sample sizes, and every finite nested-threshold lattice $\lat \supseteq \{\lambda_\tau, \lambda_{t^\star}\}$, there exist two priors $\pi^+, \pi^-$ supported on worlds of $\cW_B$ with regular margin $(\kappa, s_{\mathrm{reg}})$, $\kappa = b$, such that: every world in $\mathrm{supp}(\pi^-)$ has frontier slack $\ge s_*$ and admits a \emph{certifiable} lattice point (coverage $\ge \beta$, risk margin $\ge b\gamma/2$); in every world of $\mathrm{supp}(\pi^+)$ every lattice point with coverage $\ge \beta$ is risk-infeasible; and the laws of the observed data $(D_n^{\mathrm{src}}, U_m^{\mathrm{tgt}})$ under the two priors satisfy $\TV \le 1/4$. Consequently no $(\alpha, \beta, \delta)$-valid Model-B procedure with $\delta \le 1/8$ certifies the feasible $\pi^-$ worlds with probability $> 1/2$. Hence on any such lattice, fixed and pre-registered in advance, no \emph{valid} Model-B procedure certifies all frontier-slack-$s_*$ worlds of $\cW_B$ at any $(n, m)$: Model-B is inconsistent over $\cW_B$, in contrast to Model-A (Theorem~\ref{thm:modelA}) and Model-B$'$ (Theorem~\ref{thm:U2}).
\end{theorem}

\paragraph{Construction.}
Let $\gamma = 1/2$, $g = \beta/2 = 1/4$, $\zeta = 1/16$, $\tau = g + h$ with $h = g + \zeta = 5/16$ the mass of the gadget band, and a deep-safe margin $d = b\gamma/4$. Labels are Bernoulli, $L \mid X \sim \Bern(\eta(X))$ as in Appendix~\ref{app:naxis}. Take $K$ even (chosen below), cell mass $p = h/K$ with $p \le \gamma g / (4(1 + \gamma^2))$. On $\cX = [0,1]$ with $P_X = \mathrm{Unif}$:
\begin{itemize}
\item $G = [0, g)$: $w \equiv 1$, $\eta = \alpha - d$ (a deep-safe block that forces the floor; its margin $d < b\gamma$ is deliberately \emph{smaller} than the gadget amplitude: this is the crux of global infeasibility, and is where this construction departs from the surplus-anchored $W_1$ of Appendix~\ref{app:naxis}).
\item $\Phi = [g, \tau)$: partitioned into $K$ cells of mass $p$ grouped into $K/2$ adjacent \emph{pairs}. Per pair $r$ draw $Z_r \in \{\pm 1\}$ i.i.d.\ uniform; the two cells of pair $r$ carry loss-signs $v = (Z_r, -Z_r)$, i.e.\ $\eta = \alpha + b v$, and weight-signs $u$ with $w = 1 + \gamma u$ (so $w \in \{1/2, 3/2\} \subset [1/B, B]$ for $B \ge 2$). Under $\pi^+$ set $u = v$; under $\pi^-$ set $u = -v$. Pairing forces $\sum_{i \in \Phi} u_i = \sum_{i \in \Phi} v_i = 0$ exactly.
\item $H = [\tau, 1]$: $w \equiv 1$, $\eta = \alpha + b$ (hot).
\end{itemize}
Thresholds are the nested policies $\lambda_t = \{$accept $[0,t]\}$; the theorem fixes two of them, the $\Phi$--$H$ boundary $\lambda_\tau$ and the hot-block frontier threshold $\lambda_{t^\star}$ of (R2), and any lattice containing this pair may be used (R4). Because the per-pair weight-signs sum to zero, the $Q$-mass of each pair equals its $P$-mass $2p$ in both priors; hence $C_Q(\lambda_\tau) = g + h = \beta + \zeta$ identically, and the pair-occupancy law is the same under $\pi^+$ and $\pi^-$.

\paragraph{(R1) Lattice infeasibility under $\pi^+$.}
For one aligned pair (cells $(\eta, w) = (\alpha + b, 1 + \gamma)$ and $(\alpha - b, 1 - \gamma)$, each of $P$-mass $p$), the $Q$-budget contribution is
\begin{equation*}
p(1+\gamma)\big(\alpha - (\alpha + b)\big) + p(1-\gamma)\big(\alpha - (\alpha - b)\big) = pb\big[-(1+\gamma) + (1-\gamma)\big] = -2pb\gamma = -b\gamma \cdot (\text{pair mass}).
\end{equation*}
Consider any threshold $t$ with $C_Q(\lambda_t) \ge \beta$. Thresholds inside $G$ have coverage at most $g = \beta/2 < \beta$, so $t$ reaches into $\Phi$; write its accepted $\Phi$-part as matched pair portions of total $P$-mass $\ell$ plus one residual cell portion of mass $\le p$ (worst case a low-weight safe cell, contributing $\le b(1-\gamma)p$ to the budget). Cutting inside a pair is covered: if the two cells of the active pair contribute accepted masses $a$ and $c$, the matched portions of mass $\min(a,c)$ each enter $\ell$ and the residue $|a-c| \le p$ is the residual portion. Coverage $\ge \beta = 2g$ forces the accepted $\Phi$ $Q$-mass $\ge g$, and since the partial cell adds $Q$-mass $\le (1+\gamma)p$, we have $\ell + (1+\gamma)p \ge g$. Therefore
\begin{equation*}
\budget_Q(\lambda_t) \;\le\; \underbrace{dg}_{G} \;-\; b\gamma \ell \;+\; b(1-\gamma)p
\;\le\; (d - b\gamma)g + b(1 + \gamma^2)p
\;\le\; -\tfrac{1}{2}b\gamma g \;<\; 0,
\end{equation*}
using $d = b\gamma/4$ and $p \le \gamma g/(4(1+\gamma^2))$. At $t = \tau$, $\budget_Q(\lambda_\tau) = dg - b\gamma h < 0$, and accepting into $H$ only subtracts more. So under $\pi^+$ every lattice point with coverage $\ge\beta$ has $\RQ > \alpha$; any non-refusal is therefore invalid, and validity forces $\Prob(\cA \ne \nocert) \le \delta$.

\paragraph{(R2) Feasibility with constant slack under $\pi^-$.}
For one anti-aligned pair (cells $(\alpha + b, 1 - \gamma)$ and $(\alpha - b, 1 + \gamma)$) the $Q$-budget contribution is $+2pb\gamma$. Hence at $\lambda_\tau$ (accept $G \cup \Phi$), $C_Q(\lambda_\tau) = \beta + \zeta \ge \beta$ and
\begin{equation*}
\budget_Q(\lambda_\tau) = dg + b\gamma h > 0,
\qquad\text{so}\qquad
\alpha - \RQ(\lambda_\tau) = \frac{\budget_Q(\lambda_\tau)}{\beta + \zeta} = \frac{b\gamma(h + g/4)}{2g + \zeta} \;\ge\; \tfrac{1}{2}b\gamma,
\end{equation*}
a constant risk margin (numerically $b/3$). The frontier sits at $\betastar = g + (1+\gamma)h + (d/b)g$, so the frontier slack is $s_* = \betastar - \beta = \gamma g + (1+\gamma)\zeta + (d/b)g = 1/4$, a fixed positive constant independent of $b, n, m$. \emph{Lattice realization of the frontier.} The second threshold the theorem fixes is the hot-block threshold $\lambda_{t^\star}$ at $t^\star = \betastar = 3/4$: since $w \equiv 1$ on $H$, $C_Q(\lambda_{t^\star}) = t^\star = \betastar$ and $\budget_Q(\lambda_{t^\star}) = \budget_Q(\lambda_\tau) - b\,(t^\star - \tau) = \tfrac{3b}{16} - b\cdot\tfrac{3}{16} = 0$, i.e.\ $\RQ(\lambda_{t^\star}) = \alpha$ (feasible at equality). Hence $\betastar_\lat = \betastar = 3/4$ on every admissible lattice, since $\betastar_\lat \le \betastar$ always, and $s_* = 1/4$ is a genuine \emph{lattice}-frontier slack, not merely the relaxed one. Requiring $\lambda_{t^\star}$ is exactly what excludes a lattice leaving $H$ uncut, on which $\betastar_\lat$ would instead collapse to the integer coverage $\beta + \zeta = 9/16$, a lattice-design artifact. Its presence leaves R1 and R3 untouched: under $\pi^+$ it accepts hot mass and is only more infeasible, and the data law depends on $(P, Q, \eta)$, not on $\lat$; the displayed certifiable lattice point with strictly positive margin remains $\lambda_\tau$ (coverage $9/16$, margin $b/3 \ge b\gamma/2$), the point certified in the Le~Cam step. The interval just below the frontier lies in the hot level $\eta = \alpha + b$, so the regular margin $(\kappa, s_{\mathrm{reg}})$ holds with $\kappa = b$. Thus the $\pi^-$ worlds are genuinely certifiable at constant slack. (Under $\pi^+$ the same hot left-neighborhood gives regular margin $\kappa = b$ as well; there the frontier sits at $\betastar = g + (1-\gamma)h + (d/b)g = 7/16 < \beta$, so the world is infeasible, consistent with R1.)

\paragraph{(R3) Indistinguishability.}
$G$ and $H$ are identical under $\pi^+$ and $\pi^-$ ($w \equiv 1$, fixed $\eta$), contributing likelihood ratio $1$; only $\Phi$'s pairs carry the $\pi^+$-versus-$\pi^-$ signal. The observations are: for each source draw, its cell and label; for each target draw, its cell. Given the pair assignment (which pair each point falls in, \emph{not} the within-pair cell identity, which carries the $\pi^+$-vs-$\pi^-$ signal) the prior factorizes over pairs (the $Z_r$ are independent), and the pair-occupancy law is $\pi$-independent (R1 paragraph), so writing $N_r, M_r$ for the source and target occupancies of pair $r$ and $\chi^2_r(N_r, M_r)$ for the conditional pairwise $\chi^2$ between the two priors,
\begin{equation}
\label{eq:incon-product}
1 + \chi^2(P_+ \,\|\, P_-) \;=\; \E_{\mathrm{assign}}\Big[\,\textstyle\prod_r \big(1 + \chi^2_r(N_r, M_r)\big)\Big].
\end{equation}

\emph{Per-pair bound.} The pairwise factor is controlled by the following self-contained lemma; the construction enters only through its hypotheses.

\begin{lemma}[Pairwise correlation $\chi^2$]
\label{lem:pair-chi2}
Fix $N, M \in \mathbb{N}_0$ and $\theta, \gamma \in [0, \tfrac12]$, and let $Z \sim \mathrm{Unif}\{\pm 1\}$. Conditional on $Z$, let $X_1, \dots, X_N \in \{\pm 1\}$ be i.i.d.\ with $\E[X_i \mid Z] = \theta Z$, and independently let $Y_1, \dots, Y_M \in \{\pm 1\}$ be i.i.d.\ with $\E[Y_j \mid Z] = \sigma \gamma Z$, where $\sigma \in \{\pm 1\}$ indexes the two laws $P_\sigma$ of $(X_{1:N}, Y_{1:M})$ (each marginalizing over $Z$). Then for an absolute constant $c_{\chi}$,
\begin{equation*}
1 + \chi^2(P_+ \,\|\, P_-) \;\le\; \exp\!\big(c_{\chi}\, \theta^2 \gamma^2\, N M\big),
\end{equation*}
and $\chi^2(P_+ \,\|\, P_-) = 0$ whenever $\theta = 0$ or $\gamma = 0$.
\end{lemma}

\begin{proof}
The signed sums $x = \sum_i X_i$, $y = \sum_j Y_j$ are sufficient for $Z$. Let $q_{\pm}, r_{\pm}$ be the laws of $x, y$ given $Z = \pm 1$, and $q = \tfrac12(q_+ + q_-)$, $r = \tfrac12(r_+ + r_-)$ the marginals; define posterior-mean coordinates $m(x) = \frac{q_+(x) - q_-(x)}{q_+(x) + q_-(x)}$ and $n(y) = \frac{r_+(y) - r_-(y)}{r_+(y) + r_-(y)}$, both in $(-1, 1)$ for $\theta, \gamma < 1$, so $q_{\pm} = q(1 \pm m)$ and $r_{\pm} = r(1 \pm n)$. Conditional independence of $x, y$ given $Z$ gives, under $P_\sigma$,
\begin{equation*}
g^{\sigma}(x, y) = \tfrac12\!\sum_{z = \pm 1} q_z(x)\, r_{\sigma z}(y) = q(x) r(y)\big(1 + \sigma\, m(x) n(y)\big),
\end{equation*}
the linear-in-$z$ term cancelling. Write $U = m(x) n(y)$ under $q \otimes r$; then $|U| < 1$, and $U \overset{d}{=} -U$ (the map $x \mapsto -x$ exchanges $q_+ \leftrightarrow q_-$, flipping the sign of $m$). Hence
\begin{equation*}
1 + \chi^2(P_+ \| P_-) = \E_{q \otimes r}\!\Big[\tfrac{(1+U)^2}{1 - U}\Big]
= \E\!\Big[\tfrac{1 + 3U^2}{1 - U^2}\Big]
= 1 + 4\,\E\!\Big[\tfrac{U^2}{1 - U^2}\Big],
\end{equation*}
the middle step by symmetrizing the integrand under $U \mapsto -U$; this vanishes when $\theta = 0$ or $\gamma = 0$ (then $m \equiv 0$ or $n \equiv 0$, so $U \equiv 0$). Two bounds on $\E[U^2/(1 - U^2)]$ follow. Since $u \mapsto u/(1-u)$ is increasing and $m^2 n^2 \le m^2$, the \emph{data-processing} bound $\E[U^2/(1-U^2)] \le \E_q[m^2/(1 - m^2)]$ holds (and symmetrically $\le \E_r[n^2/(1-n^2)]$). And from $\tfrac{m^2 n^2}{1 - m^2 n^2} \le \tfrac{m^2}{1 - m^2}\,\tfrac{n^2}{1 - n^2}$ (cross-multiplying, $(1 - m^2)(1 - n^2) \le 1 - m^2 n^2$) with independence, the \emph{product} bound $\E[U^2/(1-U^2)] \le \E_q[\tfrac{m^2}{1-m^2}]\,\E_r[\tfrac{n^2}{1-n^2}]$.

The single-channel moment has a closed form. With $\psi = \big(\tfrac{1+\theta}{1-\theta}\big)^{x}$ one has $q_+/q_- = \psi$, so $m = \tfrac{\psi - 1}{\psi + 1}$ and $\tfrac{m^2}{1 - m^2} = \tfrac{(\psi - 1)^2}{4\psi} = \tfrac14(\psi - 2 + \psi^{-1})$. Averaging over $q = \tfrac12(q_+ + q_-)$ and using the per-bit identity $\E_{q_+}[\psi^{\pm 1}] $: under $q_+$, $\E[\psi] = \big(\tfrac{1+3\theta^2}{1-\theta^2}\big)^{N} =: C_\theta^N$ (each bit contributes $\tfrac{1+\theta}{2}\tfrac{1+\theta}{1-\theta} + \tfrac{1-\theta}{2}\tfrac{1-\theta}{1+\theta} = \tfrac{1 + 3\theta^2}{1 - \theta^2}$) while $\E_{q_-}[\psi] = 1$, and symmetrically for $\psi^{-1}$, gives
\begin{equation*}
\E_q\!\Big[\tfrac{m^2}{1 - m^2}\Big] = \tfrac14\big(C_\theta^N - 1\big),
\qquad
\E_r\!\Big[\tfrac{n^2}{1 - n^2}\Big] = \tfrac14\big(C_\gamma^M - 1\big),
\quad C_\gamma = \tfrac{1 + 3\gamma^2}{1 - \gamma^2}.
\end{equation*}
For $\theta, \gamma \le \tfrac12$, $\log C_\theta \le C_\theta - 1 = \tfrac{4\theta^2}{1 - \theta^2} \le \tfrac{16}{3}\theta^2$, so $C_\theta^N \le e^{(16/3)\theta^2 N}$. Two regimes finish the proof. If $\theta^2 N \le 1$ and $\gamma^2 M \le 1$, then $C_\theta^N - 1 \le \tfrac{16}{3}\theta^2 N\, e^{(16/3)\theta^2 N} \le \tfrac{16}{3} e^{16/3}\,\theta^2 N$ (using $e^t - 1 \le t e^t$), and the product bound gives $\chi^2 \le \tfrac14 (\tfrac{16}{3})^2 e^{32/3}\,\theta^2\gamma^2 NM$, whence $1 + \chi^2 \le \exp(c\,\theta^2\gamma^2 NM)$. If instead $\theta^2 N > 1$, the data-processing bound gives $1 + \chi^2 \le C_\gamma^M \le e^{(16/3)\gamma^2 M} \le e^{(16/3)\theta^2\gamma^2 NM}$, and symmetrically if $\gamma^2 M > 1$. Taking $c_{\chi}$ the larger of the two absolute constants proves the claim.
\end{proof}

\noindent Within a pair, conditional on $(N_r, M_r)$: each source point falls in one of the pair's two cells with an observed cell-sign $c \in \{\pm1\}$, and (loss-signs $v = (Z_r, -Z_r)$ across the pair) its label has mean $\alpha + c\,b\,Z_r$. The cell-occupancy law is identical under $\pi^+, \pi^-$ (R1 paragraph), so $c$ is an ancillary, prior-independent label of the datum; the deterministic map $(c, L) \mapsto (c, X)$ with $X = c\,(2L - 1)$ is therefore a bijection of the observed pair that fixes $\chi^2(P_+\|P_-)$ (the ancillary $c$ has a prior-independent law and factors out), and at $\alpha = 1/2$ the bit $X \in \{\pm1\}$ has $\E[X \mid Z_r] = 2b\,Z_r$ in \emph{both} cells. Thus the $N_r$ source data reduce to i.i.d.\ bits of mean $\theta Z_r$, $\theta = 2b$. Symmetrically, each target point's cell-identity within the pair is a $\{\pm1\}$ bit whose bias is set by the weight-signs $u = \sigma v$, of mean $\gamma\sigma Z_r$ ($\sigma = +1$ under $\pi^+$, $\sigma = -1$ under $\pi^-$). Lemma~\ref{lem:pair-chi2} applied with these $(\theta, \gamma, N_r, M_r)$ then yields, for an absolute constant $c_{\chi}$,
\begin{equation}
\label{eq:incon-perpair}
1 + \chi^2_r(N_r, M_r) \;\le\; \exp\!\big(c_{\chi}\, b^2 \gamma^2\, N_r M_r\big)
\qquad\text{for all } N_r, M_r.
\end{equation}

\emph{Collision-count moment generating function.} Substituting~\eqref{eq:incon-perpair} into~\eqref{eq:incon-product} and writing $t = c_{\chi} b^2 \gamma^2$ and $W = \sum_r N_r M_r$ (the source--target collision count over pairs),
\begin{equation*}
1 + \chi^2(P_+ \,\|\, P_-) \;\le\; \E\big[e^{t W}\big].
\end{equation*}
Because $W$ is integer-valued with $0 \le W \le nm$, and the source and target occupancies are independent with $\E[W] = \sum_r \E[N_r]\E[M_r] = (K/2)(n \cdot 2h/K)(m \cdot 2h/K) = 2h^2 nm / K$,
\begin{equation}
\label{eq:incon-mgf}
\E\big[e^{tW}\big]
\;\le\; 1 + (e^{t nm} - 1)\,\Prob(W \ge 1)
\;\le\; 1 + (e^{t nm} - 1)\,\E[W]
\;=\; 1 + (e^{t nm} - 1)\,\frac{2 h^2 nm}{K}.
\end{equation}
(The MGF is dominated by the rare high-collision events, so~\eqref{eq:incon-mgf} is the honest bound; it is exponential in $nm$, but only the \emph{existence} of a finite $K$ is needed.) Choosing any even
\begin{equation*}
K \;\ge\; K_*(n, m) \;:=\; \max\Big\{\, \tfrac{4h(1+\gamma^2)}{\gamma g},\ \ \frac{(e^{t nm} - 1)\, 2 h^2 nm}{\log(5/4)} \,\Big\}
\end{equation*}
(finite for every $n, m$; the first term enforces the cell-size constraint $p \le \gamma g/(4(1+\gamma^2))$) yields $\E[e^{tW}] \le 5/4$, hence $\chi^2(P_+\|P_-) \le 1/4$ and $\TV(P_+, P_-) \le \tfrac12\sqrt{\chi^2} \le 1/4$.

\paragraph{(R4) The lattice is fixed, not adversarial.}
Neither the cell boundaries nor the cell count $K$ enters R1--R3 through $\lat$. R1 bounds $\budget_Q(\lambda_t)$ for \emph{every} threshold $t$ with $C_Q(\lambda_t) \ge \beta$, so it applies verbatim to any nested-threshold lattice, however refined: additional thresholds only supply further points that R1 has already ruled infeasible under $\pi^+$. R2 exhibits the certifiable point as $\lambda_\tau$ (coverage $\beta + \zeta = 9/16$ in both priors, each pair's $Q$-mass equalling its $P$-mass, with risk margin $\ge b\gamma/2$) and the frontier realization as $\lambda_{t^\star}$ (coverage $t^\star = \betastar = 3/4$, since $w \equiv 1$ on $H$); $\tau = g + h$ is the $\Phi$--$H$ boundary and $t^\star$ lies inside $H$, so both are absolute constants, unmoved by $K$, $n$, or $m$. R3 is lattice-free by construction, the data law depending on $(P, Q, \eta)$ alone. The construction is therefore existential in the \emph{worlds}, whose cell count $K_*(n, m)$ does grow with the sample sizes, as this construction requires, but \emph{not} in the lattice, which may be fixed and pre-registered before any data is seen, as Definition~\ref{def:valid} requires. In this respect Theorem~\ref{thm:modelB-incon} is stronger than Theorem~\ref{thm:lower}, whose hard instance supplies its own compatible lattice per slack.

\paragraph{Le~Cam conclusion.}
A valid procedure certifies with probability $\le \delta \le 1/8$ under every $\pi^+$ world (R1), hence under the $\pi^+$ data-mixture. If it certified each feasible $\pi^-$ world with probability $> 1/2$ it would certify under the $\pi^-$ mixture with probability $> 1/2$; since $\TV(P_+, P_-) \le 1/4$, it would certify under the $\pi^+$ mixture with probability $> 1/2 - 1/4 = 1/4 > 1/8$, a contradiction. Hence some feasible $\pi^-$ world is refused with probability $\ge 1/2$. As $K_*(n, m)$ is finite at every $(n, m)$ and $\cW_B$ admits arbitrarily fine finite weight oscillation (with weights confined to $\{1/2, 3/2\}$), a hard instance exists at every sample size, and the certifiable slack stays $\ge s_*$. \hfill$\square$

\begin{remark}[What separates Model-B from Model-A, exactly]
\label{rmk:modelB-mechanism}
\looseness=-1 The barrier is the \emph{bilinear} functional $\E_P[w\,\Sacc\,(\eta - \alpha)]$: its sign is fixed by the alignment between the weight-signs $u$ (seen only through the target covariate channel, which observes $Q_X = wP_X$) and the loss-signs $v$ (seen only through the source label channel). Neither marginal channel reveals the alignment $\sum_i p_i u_i v_i$; the two priors share both marginal laws and differ only in the sign-coupling, which is detectable only on cells occupied by \emph{both} channels: a two-sample correlation-detection problem whose signal scales as the collision count, and which the adversary defeats with weight complexity $K$ beyond any fixed $(n, m)$. Model-A escapes because exact $w$ reveals $u$, collapsing the risk test to the unbiased weighted labeled-source mean of Theorem~\ref{thm:modelA}; Model-B$'$ escapes because a pre-registered $K$-cell partition caps the oscillation, so the per-split nuisance becomes estimable at the finite sample sizes of Theorem~\ref{thm:U2}. The inconsistency is thus the information-theoretic shadow of the nuisance dimension that B$'$ pays for explicitly; it does not assert a polynomial sample-complexity rate, only that no $(n, m)$-uniform guarantee exists over the unrestricted class. This estimability boundary is the certification counterpart of the learnability/estimability lines for arbitrary covariate shift: reject-to-learn via a reliable-learner oracle \citep{kalai21a} and structural recovery of $\E_Q[f]$ from source labels and target covariates \citep{adil26a}, with the pre-registered $K$-cell partition of Model-B$'$ our analogue of the structure those works impose to regain tractability.
\end{remark}

\subsection{Non-triviality remarks}

\begin{remark}[$\beta \to 0$ kills the impossibility; positivity persists]
\label{rmk:beta0}
As $\beta \to 0$: (i) the above-frontier region $\{\beta > \betastar\}$ of Theorem~\ref{thm:frontier} becomes empty; (ii) the $n$-axis window $c_1 \kappa^{-1}\sqrt{\beta/n} \to 0$; (iii) the $m$-axis window $c_2\sqrt{\beta/m} \to 0$ as well; both axes carry $\sqrt{\beta}$ and vanish with the floor.
Meanwhile the class $\cW_B$ and its overlap structure are unchanged: nothing about positivity moves.
In Model-A the deep-safe block is then certifiable from labeled data alone at $n = O(\log(1/\delta))$ with no dependence on the frontier slack; in Model-B a vestigial unlabeled requirement remains, confirming $\CQ(\lambda) > 0$ needs $O(1/\CQ(\lambda))$ draws, but it is $s$-independent and does not scale with the frontier.
The bite threshold is a property of $(\beta, \budget)$, not of overlap.
\end{remark}

\begin{remark}[Impossibility is loss-profile-driven, not support-driven]
\label{rmk:lossdriven}
In $\cW_B$ all target mass is identified (full overlap): the unidentified mass is zero in every construction, identically.
Yet feasibility boundaries differ across worlds through $\budget$, a functional of the loss profile $\eta$ and $Q$: the $n$-axis pair has identical $(P_X, Q_X)$ (identical overlap geometry) and different feasibility verdicts at $(\alpha, \beta)$.
The impossibility here is therefore structurally disjoint from positivity results such as \citet{damour2021overlap}.
\end{remark}

\begin{remark}[The floor is the $m$-axis; the model split is load-bearing]
\label{rmk:msplit}
Remove the floor ($\beta = 0$): certification reduces to exhibiting $\lambda$ with $\RQ(\lambda) \le \alpha$ and $\CQ(\lambda) > 0$.
In the $m$-axis pair, any threshold inside $G$ has $\RQ = \alpha - a$ under every $Q \in \cW_B$ sharing $\eta$ (coverage changes; risk on $G$ does not), so in Model-A (where $\CQ(\lambda) = \E_P[w\, \mathbf{1}_{[0,t]}]$ is computable from the known $w$), the deep-safe block is certifiable from labeled data alone with no dependence on $s$.
The $s$-scaled $m$-axis hardness exists only because the floor forces acceptance up to the frontier.
In Model-A the floor information moreover pools across channels (a labeled-sample route to $\CQ$ opens), so the clean $m$-only floor axis is a Model-B phenomenon; a matching Model-A floor-axis lower bound would require a pooled-rate analysis that we do not claim; its constants are conjecture-grade and we leave it open.
\end{remark}

\subsection{The overlap-created boundary, for contrast}
\label{app:overlap}

For completeness we record what \emph{overlap-created} impossibility looks like, outside the bounded-ratio class: it is governed by a sharp threshold in the unsupported target mass, with no role for sample sizes.

\begin{proposition}[Sharp threshold under unsupported target mass]
\label{prop:overlap}
Drop the bounded-ratio assumption.
Let $A \subseteq \cX$ be measurable with $P_X(A) = 0$ and $Q_X(A) = q_A$, and let the loss class be $\cF_A = \{\eta : \eta|_{A^c} = \eta_0 \text{ (fixed, identifiable)},\ \eta|_A \text{ arbitrary}\}$.
A procedure sees source labels $D_n \sim P^n$ and unlabeled target covariates $U_m \sim Q_X^m$, and outputs either a rule $S$ with the certificate ``$\RQ(S) \le \alpha \wedge \E_Q[S] \ge \beta$'' or $\nocert$.
\begin{enumerate}
    \item[(a)] (Lower bound.) If $q_A > 1 - (1 - \alpha)\beta$, then for every procedure there is a member of $\cF_A$ on which, whenever the procedure certifies, the certificate is false; consequently any $(\alpha, \beta, \delta)$-valid procedure certifies with probability at most $\delta$ there (it must refuse with probability $\ge 1 - \delta$).
    \item[(b)] (Population/oracle achievability under a spare off-$A$ budget.) Assume in addition that $Q|_A$ is nonatomic, or that acceptance may be randomized on $A$, so that an accepted sub-mass of any size $\le q_A$ is attainable. Suppose there is an identifiable subset $S_0 \subseteq A^c$ with target mass $D_0 = \E_Q[S_0]$ and (world-independent) off-$A$ risk numerator $N_0 = \E_Q[\eta_0 S_0]$, and set $a = \max\{0,\, \beta - D_0\}$. If $a \le q_A$ and $N_0 + a \le \alpha(D_0 + a)$, then an explicit rule certifies validly over $\cF_A$. This is a \emph{population/oracle} achievability statement: the rule uses the true masses $D_0, N_0, q_A$ and an exactly sized accepted sliver of $A$, so a finite-sample certificate from $(D_n, U_m)$ would additionally require these estimated with confidence margins and a strict (non-equality) spare-budget slack. (Without an atomless or randomized $A$ the threshold can shift: an indivisible $A$-atom of mass $> a$ cannot supply exactly $a$, so the deterministic feasible set is smaller; the gap is the integrality gap of Appendix~\ref{app:frontier}.)
    In the risk-free special case $N_0 = 0$, $D_0 = 1 - q_A$ this condition reduces to $q_A \le 1 - (1-\alpha)\beta$, the exact complement of the lower bound (a); so $q_A^\star = 1 - (1-\alpha)\beta$ is the tight threshold in that regime, while a positive off-$A$ risk $N_0$ tightens it.
\end{enumerate}
\end{proposition}

\begin{proof}
Write $D = \E_Q[S]$ for accepted mass, $a_A = \E_Q[S \mathbf{1}_A]$ for accepted target mass in $A$, and $N_0 = \E_Q[L S \mathbf{1}_{A^c}]$ for the off-$A$ numerator, which is world-independent since $\eta_0$ is fixed.
Consider two worlds $\eta^0$ ($\eta \equiv 0$ on $A$) and $\eta^1$ ($\eta \equiv 1$ on $A$), both in $\cF_A$ and identical off $A$.

\emph{Indistinguishability.}
Labels in $D_n$ are observed only at the $X_i$; since $P_X(A) = 0$, almost surely no $X_i \in A$, so no label from $A$ is ever seen; $U_m$ carries covariates only; the off-$A$ law is shared.
Hence the joint law of the observed data is identical under $\eta^0$ and $\eta^1$, and any (possibly randomized) procedure emits the same certificate/rule law under both.

\emph{Ratio inflation.}
$\RQ^{(j)}(S) = (\E_Q[L S \mathbf{1}_{A^c}] + \E_Q[L S \mathbf{1}_A])/D$; the first term equals $N_0$ in both worlds, while $\E_Q[L S \mathbf{1}_A] = \E_Q[\eta S \mathbf{1}_A]$ equals $0$ under $\eta^0$ and $a_A$ under $\eta^1$.
Subtracting, $\RQ^{1} - \RQ^{0} = a_A/D$.

\emph{Sharp threshold.}
Suppose the procedure certifies with rule $S$ (the same $S$ in both worlds, by indistinguishability) and the certificate is valid in both.
Validity gives $D \ge \beta$ and $\RQ^1(S) \le \alpha$.
Since $\E_Q[S \mathbf{1}_{A^c}] \le Q_X(A^c) = 1 - q_A$, we have $a_A = D - \E_Q[S \mathbf{1}_{A^c}] \ge D - (1 - q_A)$, and with $N_0 \ge 0$,
\begin{equation*}
\RQ^1(S) = \frac{N_0 + a_A}{D} \;\ge\; \frac{a_A}{D} \;\ge\; 1 - \frac{1 - q_A}{D}.
\end{equation*}
This lower bound is increasing in $D$, so it is minimized at the smallest feasible coverage $D = \beta$:
$\RQ^1(S) \ge 1 - (1 - q_A)/\beta$.
Forced violation $\RQ^1 > \alpha$ holds whenever $1 - (1 - q_A)/\beta > \alpha$, i.e.\ exactly when $q_A > 1 - (1 - \alpha)\beta$.
Thus above the threshold, every certifying procedure is false in world $\eta^1$; since $\{\text{certify}\} \subseteq \{\text{invalid certificate}\}$ there, $(\alpha, \beta, \delta)$-validity forces $\Prob(\cA \ne \nocert) \le \delta$ in world $\eta^1$.

\emph{Achievability.}
Suppose the spare-budget condition of (b) holds: an identifiable $S_0 \subseteq A^c$ with $D_0 = \E_Q[S_0]$, $N_0 = \E_Q[\eta_0 S_0]$, and $a = \max\{0, \beta - D_0\} \le q_A$ satisfy $N_0 + a \le \alpha(D_0 + a)$.
Take $S$ to accept $S_0$ plus a sliver of $A$-mass $a$ (available since $a \le q_A = Q_X(A)$; we assume $Q|_A$ is nonatomic, or equivalently allow randomized acceptance on $A$, so that an accepted sub-mass of exactly $a$ exists).
Then $D = D_0 + a \ge \beta$ (either $D_0 \ge \beta$ with $a = 0$, or $a = \beta - D_0$ and $D = \beta$), and the accepted-$A$ risk numerator is $\E_Q[\eta S \mathbf 1_A] \le a$ for \emph{every} $\eta|_A$ (loss $\le 1$ on accepted mass $a$); hence, uniformly over $\cF_A$,
\begin{equation*}
\RQ(S) = \frac{N_0 + \E_Q[\eta S \mathbf 1_A]}{D} \;\le\; \frac{N_0 + a}{D_0 + a} \;\le\; \alpha,
\end{equation*}
the last step by the spare-budget hypothesis. This is the correct worst case: the original ``$N_0 = 0$'' worst case is valid only when $\eta_0$ carries no risk on the accepted $A^c$-mass; in general the off-$A$ numerator $N_0$ adds to the ratio (e.g.\ $\alpha = \beta = \tfrac12$, $q_A = 0.6 < 0.75$ with $\eta_0 = \tfrac12$ on $A^c$ forces $N_0 = \tfrac12 D_0$ and $\RQ \ge 0.6 > \alpha$ for the unconditional rule), which is exactly what the condition $N_0 + a \le \alpha(D_0 + a)$ guards against.
In the risk-free regime $N_0 = 0$, $D_0 = 1 - q_A$: when $q_A > 1 - \beta$ one has $a = \beta - D_0 > 0$, $D = \beta$, and $\RQ = a/\beta = 1 - (1-q_A)/\beta$; when $q_A \le 1 - \beta$ one has $a = 0$, $D = D_0$, and $\RQ = 0$. Compactly $\RQ = a/(D_0 + a) = \max\{0,\, 1 - (1-q_A)/\beta\} \le \alpha$ exactly when $q_A \le 1 - (1-\alpha)\beta$, recovering the tight threshold of (a).
A finite-sample version with $P_X(A) > 0$ but $n P_X(A) \ll 1$ holds up to a $1 - n P_X(A)$ event, by the same argument conditioned on no source point landing in $A$.
\end{proof}

The contrast with Theorem~\ref{thm:lower} is the point: Proposition~\ref{prop:overlap}'s impossibility is created by unsupported mass and is insensitive to $(n, m)$, while the law's impossibility lives entirely inside the full-overlap class, scales with $(n, m)$, and vanishes as $\beta \to 0$.
Within $\cW_B$ the overlap mechanism cannot occur; the floor mechanism is what remains, and it is the subject of this paper.

\section{Model-A Achievability and the Phase Diagram}
\label{app:modelA}

This appendix proves the formal version of Theorem~\ref{thm:law}, Claim 2, together with the operational corollaries: the one-sided feasibility estimator and the which-data-to-buy boundary.
Throughout, Model-A means the procedure receives the ratio function $w$ exactly.
Standing parameter domain (with Theorem~\ref{thm:law}): $\delta \le 1/8$, $B \ge 2$, $s > 0$, $\kappa > 0$, $0 < c_0 \le 1$, $|\lat| \ge 1$, and all empirical-Bernstein sample sizes $n, m \ge 2$ (the Maurer--Pontil endpoints carry $N - 1$ denominators; set $\LCB_m := 0$ when $m < 2$, the labeled $n$-route then carrying the floor). The router score and lattice $\lat$ are pre-registered (fixed independently of $D_r, D_f$), so the fixed-function empirical-Bernstein bounds below apply (\S\ref{sec:setup}, Assumption~\ref{ass:bprime}(i)).

\begin{theorem}[Model-A achievability]
\label{thm:modelA}
In Model-A (known $w \le B$), under $\mathrm{LR}_{\mathrm{margin}}(s)$ and a regular margin $(\kappa, s_0)$ with $s \le s_0 \wedge c_0\beta$ (the standing local-regime condition of Theorem~\ref{thm:law}), the procedure below is $(\alpha, \beta, \delta)$-valid on all of $\cW_B$ and certifies with probability $\ge 1 - \delta$ on every world with slack exactly $s = \betastar - \beta$ (the standing definition; the world's actual slack, not a lower bound on it), provided
\begin{gather*}
n \;\ge\; C\Big(\frac{\bar V_{\lambda_s}\,\log(|\lat|/\delta)}{(\kappa s)^2} + \frac{B\,\log(|\lat|/\delta)}{\kappa s}\Big)
\qquad\text{and}\\
\Big(m \;\ge\; C\,\frac{\beta\,\log(|\lat|/\delta)}{s^2}
\;\;\text{or}\;\;
n \;\ge\; C\Big(\frac{B\beta\,\log(|\lat|/\delta)}{s^2} + \frac{B\,\log(|\lat|/\delta)}{s}\Big)\Big),
\end{gather*}
where $\bar V_{\lambda_s} := \Var_P(w S_{\lambda_s} (L - \alpha)) \le \E_P[w^2 S_{\lambda_s}]$ is the \emph{population} (deterministic) accepted-region weighted variance proxy at the $\mathrm{LR}_{\mathrm{margin}}(s)$ witness $\lambda_s$ ($\betastar$ the relaxed frontier of \S\ref{sec:setup}, $s = \betastar - \beta$; if several witnesses exist, read the condition at the one minimizing $\bar V_{\lambda_s}$), the $B \log(\cdot)/(\kappa s)$-type terms are the empirical-Bernstein range terms, and the floor condition uses a variance-aware lower confidence bound at coverage level $\approx \beta$, matching the $m$-axis lower bound of Theorem~\ref{thm:lower} including the $\beta$ factor. (Under the standing $\delta \le 1/8$ the per-bound radii carry $\log(2/\delta') = \log(8|\lat|/\delta) \le 2\log(|\lat|/\delta)$, absorbed into $C$, so the displayed $\log(|\lat|/\delta)$ is correct up to the constant.)
\end{theorem}

\paragraph{Procedure (floor-aware weighted learn-then-test).}
For each $\lambda \in \lat$, with one shared per-bound budget $\delta' = \delta/(4|\lat|)$ (the power proof's four-tail union for the witness needs $4\delta' = \delta/|\lat| \le \delta$; validity, unioning the three bound families, then costs at most $3\delta/4 \le \delta$):
\begin{enumerate}
    \item \emph{Risk test}: a Maurer--Pontil empirical-Bernstein one-sided upper confidence bound \citep{maurer2009empirical} for the linear functional $\GQ(\lambda) = \E_P[w \Sacc (L - \alpha)]$ from the $n$ labeled samples (range $\le B$); pass iff $\UCB \le 0$.
    Note $\RQ(\lambda) \le \alpha \iff \GQ(\lambda) \le 0$ whenever $\CQ(\lambda) > 0$: the test is linear, with no fractional program anywhere.
    \item \emph{Floor test (pooled)}: pass iff $\LCB(\lambda) := \max\{\LCB_m, \LCB_n\} \ge \beta$, where $\LCB_m$ is a variance-aware empirical-Bernstein lower bound on $\CQ(\lambda)$ from the $m$ unlabeled draws at level $\delta'$ (deviation $\asymp \sqrt{\CQ(1 - \CQ)\log(1/\delta')/m} + \log(1/\delta')/m$, the $\sqrt{\beta}$-scaled rate), and $\LCB_n$ is an empirical-Bernstein lower bound on $\E_P[w \Sacc]$ from the labeled draws at $\delta'$ (Model-A only; range $B$). (We use the empirical-Bernstein floor LCB throughout, the route the power proof's Lemma~\ref{lem:lcbdeficit} bounds; a Clopper--Pearson LCB is also valid but its deficit constants differ and we do not invoke that route.)
    Additionally require $\LCB > 0$.
    \item Output the passing $\lambda$ of maximal $\LCB$ coverage (any tie-break); else $\nocert$.
\end{enumerate}

\begin{proof}
\emph{Validity.}
For fixed $\lambda$, $\Prob(\UCB < \GQ(\lambda)) \le \delta'$ and $\Prob(\LCB_\cdot > \CQ(\lambda)) \le \delta'$ each, by the one-sided empirical-Bernstein and binomial bounds (bounded ranges, i.i.d.\ samples).
Union over $\lat$ and the three bound families: with probability $\ge 1 - \delta$, every passing $\lambda$ truly satisfies $\GQ(\lambda) \le 0$ and $\CQ(\lambda) \ge \beta$, and selection among simultaneously valid tests adds no error, by the learn-then-test argument \citep{angelopoulos2021learn, laufergoldshtein2023efficiently}: $\Prob(\exists \text{ certified invalid } \lambda) \le \delta$.
The requirement $\LCB > 0$ ensures a certified $\hat\lambda$ has $\CQ(\hat\lambda) > 0$, so the ratio is well-defined and the $\beta = 0$ degenerate edge is excluded by the convention of Theorem~\ref{thm:frontier}(b).

\emph{Power at slack $s$.}
By $\mathrm{LR}_{\mathrm{margin}}(s)$ pick $\lambda_s$ with $\CQ(\lambda_s) \in [\beta + s/4,\, \betastar - s/4]$ and $-\GQ(\lambda_s) \ge \kappa s/8$; for risk-score-aligned threshold lattices the risk margin is inherited from the budget law~\eqref{eq:budgetlaw_app} applied to the rearranged-optimal sets, which the lattice's acceptance sets are.
Then:
\begin{itemize}
    \item \emph{Risk margin.} By Lemma~\ref{lem:overshoot} applied to the bounded variable $w S_{\lambda_s}(L-\alpha)$ (range $\le B$; here unconditionally, as $w$ is known), with probability $\ge 1 - 2\delta'$ the UCB overshoots $\GQ$ by at most $\mathrm{Rad}(\lambda_s) = 3\sqrt{2 \bar V_{\lambda_s}\log(2/\delta')/n} + 11B\log(2/\delta')/(n-1)$. This radius couples the two empirical-Bernstein root terms (an upper-confidence endpoint overshoots the mean by at most \emph{twice} its own radius, $\UCB - \GQ \le 2[\sqrt{2\hat V\log(2/\delta')/n} + 7B\log(2/\delta')/(3(n-1))]$) with the empirical-to-population variance concentration $\hat V \le 2\bar V_{\lambda_s} + 4B^2\log(2/\delta')/(n-1)$. Under the stated $n$ the population-variance and range terms are each $\le \kappa s/16$ (for $C$ large), so $\mathrm{Rad}(\lambda_s) \le \kappa s/8$ and the risk test passes.
    \item \emph{Floor margin.} $\CQ(\lambda_s) - \beta \ge s/4$, and by Lemma~\ref{lem:lcbdeficit} (applied to $Y = S_{\lambda_s} \in [0,1]$, $R = 1$, $V = \CQ(1-\CQ)$) the $\LCB_m$ deficit is at most $3\sqrt{2 \CQ(1 - \CQ)\log(2/\delta')/m} + 11\log(2/\delta')/(m-1) \le s/4$ (each term $\le s/8$, the deficit being twice the radius) under the stated $m \gtrsim \beta \log(|\lat|/\delta)/s^2$, using $\CQ(\lambda_s) \le \betastar = \beta + s \le 2\beta$ from the hypothesis $s \le c_0\beta$; this is where the variance-aware bound buys the matching $\beta$ factor.
    Alternatively (the \emph{$n$-route}, available in Model-A because $w$ is known) the same floor functional is reachable from labeled data: $S_{\lambda_s}$ is a function of the frozen router score, hence of $X$, so for every $\lambda$ the covariate-shift change of measure gives $\E_P[w S_\lambda] = \E_{Q_X}[S_\lambda] = \CQ(\lambda)$, and $\LCB_n$ (an empirical-Bernstein lower bound on $\E_P[w S_{\lambda_s}]$) is therefore a valid lower confidence bound on $\CQ(\lambda_s)$ (this identity also routes the $\LCB_n$ branch of \emph{Validity} above, since $\LCB_n \ge \beta$ then forces $\CQ \ge \beta$). For the bounded variable $Z = w S_{\lambda_s} \in [0, B]$, Lemma~\ref{lem:lcbdeficit} (with $R = B$, $V = \Var_P(Z)$) bounds the $\LCB_n$ deficit by $3\sqrt{2\Var_P(Z)\log(2/\delta')/n} + 11B\log(2/\delta')/(n-1)$, with $\Var_P(Z) \le \E_P[w^2 S_{\lambda_s}] \le B\,\E_P[w S_{\lambda_s}] = B\,\CQ(\lambda_s) \le B(\beta + s) \le 2B\beta$ (using $S_{\lambda_s}^2 = S_{\lambda_s}$, $0 \le w \le B$, and $\CQ(\lambda_s) \le \betastar = \beta + s \le 2\beta$); hence the deficit is $\le 6\sqrt{B\beta\log(2/\delta')/n} + 11B\log(2/\delta')/(n-1)$, each summand $\le s/8$ under the stated $n \gtrsim B\beta\log(|\lat|/\delta)/s^2 + B\log(|\lat|/\delta)/s$, so $\LCB_n \ge \CQ(\lambda_s) - s/4 \ge \beta$. The extra $B$ under the root (range $B$ here, versus range $1$ for the unlabeled $\LCB_m$) is exactly the $\sqrt{B\beta/n}$-versus-$\sqrt{\beta/m}$ price of the source-weighted floor recorded in Corollary~\ref{cor:buy}.
\end{itemize}
The power proof for $\lambda_s$ unions four tails for the single witness: the UCB mean overshoot and its variance concentration (Lemma~\ref{lem:overshoot}), and the floor-LCB mean deficit and its variance concentration (Lemma~\ref{lem:lcbdeficit}; one floor route suffices). The procedure therefore certifies with probability $\ge 1 - 4\delta' \ge 1 - \delta$ using the shared per-tail budget $\delta' = \delta/(4|\lat|)$, which gives $4\delta' = \delta/|\lat| \le \delta$ for every $|\lat| \ge 1$ (no $|\lat| \ge 2$ assumption is needed: $\mathrm{LR}_{\mathrm{margin}}(s)$ does not guarantee a feasible lattice point distinct from the witness, as the witness may itself attain $\betastar_\lat$), the budget absorbing all four tails.
\end{proof}

\begin{remark}[Released oracle kernel: the evaluated two-family variant]
\label{rmk:releasedoracle}
The procedure displayed above allocates one shared per-bound budget $\delta' = \delta/(4|\lat|)$ across \emph{three} bound families and takes the pooled floor statistic $\max\{\LCB_m, \LCB_n\}$.
The released implementation runs a different registered variant, which we call the \emph{two-family oracle kernel} and which produces every reported oracle-A number.
It differs in exactly three registered respects.
\emph{(i)~Two families.} It runs only the risk UCB and the target-sample floor LCB, at the common per-bound level $\delta'' := \delta/(2|\lat|)$.
\emph{(ii)~Target-only floor.} It does not construct the labeled-source branch $\LCB_n$; its floor test is $\LCB_m \ge \beta$ together with $\LCB_m > 0$.
\emph{(iii)~Realized range.} Its empirical-Bernstein range constant is the realized $B_{\mathrm{eff}} := \max_k w_k$ of the world rather than the declared bound $B$; this is Model-A information, since $w$ is supplied exactly, and it only shrinks the range term.

\emph{Validity.}
Two bound families at level $\delta''$, unioned over $\lat$, cost $2|\lat|\,\delta'' = \delta$; the argument of \emph{Validity} above is otherwise unchanged, and the $\LCB_n$ branch it routes is simply not exercised.
So the two-family kernel is $(\alpha, \beta, \delta)$-valid on all of $\cW_B$, at the same $\delta$.

\emph{Power.}
The witness argument unions the same four tails, now at $\delta''$, for a total of $4\delta'' = 2\delta/|\lat| \le \delta$ whenever $|\lat| \ge 2$ (both evaluated lattices have $|\lat| \in \{50, 64\}$).
Dropping the $n$-route removes an \emph{alternative} floor route, so the $m$-route sample-size condition of Theorem~\ref{thm:modelA} still suffices verbatim; the theorem's disjunctive floor condition is what the displayed procedure buys and the kernel does not.

\emph{Neither variant dominates the other.}
The kernel's per-bound level is looser than the displayed one ($\log(2/\delta'') = \log(4|\lat|/\delta)$ against $\log(8|\lat|/\delta)$, that is $\log 4{,}000$ against $\log 8{,}000$ at the audited $\delta = 0.05$, $|\lat| = 50$), which makes it certify more often; the displayed procedure carries the extra $n$-route floor, which the kernel lacks.
We therefore claim no pointwise domination in either direction, and we report the oracle numbers as records of the two-family kernel rather than of the displayed procedure.
Both are $(\alpha, \beta, \delta)$-valid, so the $0$-violation tallies are statements about the arm that was run; the certification counts, onsets, and the oracle-versus-B$'$ premium are likewise properties of this kernel.
\end{remark}

\begin{corollary}[One-sided feasibility estimator]
\label{cor:feasibility}
$\hat\beta^*_{\LCB} := \max\big\{0,\ \max\{\LCB(\lambda) : \UCB_{\mathrm{risk}}(\lambda) \le 0\}\big\}$ (clipped at $0$, with the inner $\max\emptyset := 0$; consistent with refusal and nonnegative coverage) satisfies $\Prob(\hat\beta^*_{\LCB} > \betastar) \le \delta$ and, at slack $s$ and the sample sizes of Theorem~\ref{thm:modelA}, $\hat\beta^*_{\LCB} \ge \beta$ with probability $\ge 1 - \delta$: an operational ``is $(\alpha, \beta)$ certifiable here'' test.
\end{corollary}

\begin{proof}
Immediate from the same events as in the proof of Theorem~\ref{thm:modelA}.
\end{proof}

\begin{corollary}[Which-data-to-buy boundary]
\label{cor:buy}
Define the \emph{deterministic} radii
\begin{align*}
s_n &= C'\Big(\kappa^{-1}\sqrt{\bar V_{\lambda_s} \log(|\lat|/\delta)/n} + B\kappa^{-1}\log(|\lat|/\delta)/n\Big),\\
s_m &= C'\Big(\sqrt{\beta \log(|\lat|/\delta)/m} + \log(|\lat|/\delta)/m\Big)
\end{align*}
in the population variance proxy $\bar V_{\lambda_s}$ of Theorem~\ref{thm:modelA} (the Model-B floor route; the Model-A source-weighted floor term carries the analogous $\sqrt{B\beta\log(\cdot)/n} + B\log(\cdot)/n$ form).
Certification at slack $s$ is \emph{guaranteed} once $s \gtrsim s_n + s_m$, a sufficient condition holding with probability $\ge 1 - \delta$ (the deficits of Theorem~\ref{thm:modelA} then fit inside the margins $\kappa s/8$ and $s/4$, all radii being deterministic in $\bar V_{\lambda_s}$); the matching necessity is the Model-B construction of Theorem~\ref{thm:lower}, and in Model-A the source-weighted floor route can certify from labeled data even when the $m$-axis is starved, so this is a comparison-of-radii diagnostic rather than a necessary-and-sufficient theorem for a fixed world. The binding axis is $n$ when $s_n > s_m$; in the variance-dominated regime (neglecting the empirical-Bernstein range terms $B\log(\cdot)/(\kappa n)$ and $\log(\cdot)/m$) this reads $m \gtrsim \beta n \kappa^2/\bar V_{\lambda_s}$, while in general the binding axis is governed by the full inequality $s_n > s_m$ (e.g.\ at $\bar V_{\lambda_s} = 0$ the range term keeps $s_n > 0$, so the displayed ratio form does not apply); replacing $\bar V_{\lambda_s}$ by its empirical plug-in $\hat V$ gives an \emph{estimable} readout of this boundary: a data-driven diagnostic computed from the run, not itself a certification threshold (the guarantee above is the deterministic one in $\bar V_{\lambda_s}$).
This matches Theorem~\ref{thm:lower}'s two axes up to constants and $\log(|\lat|)$ factors.
\end{corollary}

\begin{proof}
Immediate from the deficit bounds in the power proof of Theorem~\ref{thm:modelA}: certification is guaranteed when both deficits fit inside the margins $\kappa s/8$ and $s/4$, the displayed sufficient condition up to constants; comparing $s_n$ with $s_m$ gives the diagnostic boundary.
\end{proof}

\begin{proof}[Proof of Corollary~\ref{cor:phase}]
The guaranteed-certification region is read directly off Theorem~\ref{thm:modelA}'s sufficient conditions: the risk condition $n \gtrsim N_{\mathrm{risk}} := \bar V_{\lambda_s}\log(|\lat|/\delta)/(\kappa s)^2 + B\log(|\lat|/\delta)/(\kappa s)$ (variance \emph{and} empirical-Bernstein range term), together with the floor \emph{disjunction} $m \gtrsim M_{\mathrm{floor}} := \beta\log(|\lat|/\delta)/s^2$ \emph{or} (the source-weighted $n$-route, available in Model-A because $w$ is known) $n \gtrsim N_{\mathrm{floor}} := B\beta\log(|\lat|/\delta)/s^2 + B\log(|\lat|/\delta)/s$. The region is therefore $\{n \gtrsim N_{\mathrm{risk}}\} \cap (\{m \gtrsim M_{\mathrm{floor}}\} \cup \{n \gtrsim N_{\mathrm{floor}}\})$, i.e.\ the union of the $m$-route quadrant $\{n \gtrsim N_{\mathrm{risk}},\ m \gtrsim M_{\mathrm{floor}}\}$ and the $n$-only half-strip $\{n \gtrsim N_{\mathrm{risk}} \vee N_{\mathrm{floor}}\}$ (no constraint on $m$), the claimed shape; the half-strip is the Model-A enlargement, absent when $w$ is unknown (then the floor is reachable only through the unlabeled $m$-route). When $N_{\mathrm{floor}} \ge N_{\mathrm{risk}}$ the union has a genuine step-corner at $(N_{\mathrm{floor}},\ M_{\mathrm{floor}})$, the break-even between the unlabeled $m$-route and the source-weighted labeled $n$-route for the \emph{floor}: for $n \in [N_{\mathrm{risk}}, N_{\mathrm{floor}})$ the \emph{sufficient} $n$-route guarantee does not yet cover the floor, so the available guarantee uses unlabeled target $m$; once $n \gtrsim N_{\mathrm{floor}}$ the $n$-route's sufficient threshold clears it and no $m$ is needed. This corner is a sufficient-region boundary, not a per-world necessity threshold: the unknown $\gtrsim$-constants make its exact location schematic. (When instead $N_{\mathrm{risk}} > N_{\mathrm{floor}}$, the $n$-route clears the floor already within the risk requirement, so the region is the half-strip $\{n \gtrsim N_{\mathrm{risk}}\}$ alone and the $m$-route quadrant is redundant; the Model-A enlargement is then total.) The separate formal sufficient boundary comparing the binding \emph{radii} (the labeled \emph{risk} radius $s_n$ versus the unlabeled \emph{floor} radius $s_m$) is Corollary~\ref{cor:buy} (binding axis $n$ when $s_n > s_m$), and the feasibility readout is Corollary~\ref{cor:feasibility}; when instead the nuisance budget binds in Model-B$'$ ($\nb \le \kappa s/32$ fails, Theorem~\ref{thm:U2}), buying weight-block samples $n_w, m_w$ \emph{does} shrink the nuisance radius $\nb$ (at Proposition~\ref{prop:I1}'s rate), so the attribution is ``buy weight-block samples''; only an irreducible shift-model misspecification (weights not $\cK$-measurable) is uncurable by any sample purchase.
\end{proof}

\begin{remark}[The variance functional is localized on both sides]
\label{rmk:localized}
On the lower-bound side, the adversary's KL cost per unit of frontier flip is $\asymp (\kappa s)^2/(\rho\, \bar w_U^2 \cdots)$ minimized over slices $U$ inside the accepted region, so the binding quantity is an accepted-region weighted mass ($\rho \le \beta$, weights $w|_U$); a shifted variant places the slice where $w$ is large, and the channel scales like localized second moments of $\E_P[w^2 S]$-type.
On the upper-bound side, the empirical-Bernstein variance in Theorem~\ref{thm:modelA} is $\Var_P(w \Sacc (L - \alpha)) \le \E_P[w^2 \Sacc]$, the same localized family.
Global $\E_P[w^2]$ (global ESS) appears on neither side.
Exact constant matching between the two sides is open: the lower-bound constant involves the slice geometry, the upper the full accepted-set variance; we flag this and do not claim it.
\end{remark}

\begin{remark}[Relation to the floor-free minimax bound]
\label{rmk:vsnonmonotone}
The matching single-axis lower bound for conformal risk control under non-monotone losses \citep{aldirawi2026conformal} concerns excess risk when selecting among a finite grid of tuning-parameter values, with no coverage constraint: one resource axis, no coverage-floor parameter.
Our lower bound differs in (i) the $m$-axis existing at all (it is the floor's shadow, Remark~\ref{rmk:msplit}), (ii) $\beta$ appearing inside the $n$-axis constant through the inframarginal mass, and (iii) the target being a certification (decide-or-refuse) task.
The techniques are neighbors in the same minimax toolkit \citep{tsybakov2009introduction}, theirs a Fano argument over the tuning-parameter grid and ours Le~Cam two-point pairs, one per axis, but neither statement implies the other.
\end{remark}

\paragraph{Summary of the three claims.}
Table~\ref{tab:claims} (\S\ref{sec:theory}, beside the map) states, for each claim of Theorem~\ref{thm:law}, its information model, the assumptions it adds beyond the standing ones, what it proves, and the gap it leaves open. This makes the cross-model (not within-model) structure of the match scannable.

\paragraph{Assembly: proof of Theorem~\ref{thm:law}.}
Claim 1 is Theorem~\ref{thm:lower} with the additive single-family combination of Appendix~\ref{app:additive}.
Claim 2 is Theorem~\ref{thm:modelA}: in the worst case over the class, the $\mathrm{LR}_{\mathrm{margin}}$ witness has $\bar V_{\lambda_s} \le \E_P[w^2 S_{\lambda_s}] \le B \E_P[w S_{\lambda_s}] = B\, \CQ(\lambda_s) \le B(\beta + s) \asymp B\beta$ in the local regime $s \le c_0\beta$; the empirical-Bernstein \emph{range} term is dominated via the witness margin $\kappa s/8 \le -\GQ(\lambda_s) \le \CQ(\lambda_s) \le (1+c_0)\beta$, giving $B/(\kappa s) \le 8(1+c_0)\,B\beta/(\kappa s)^2$, so with $B$ treated as a constant the displayed conditions reduce to $n \gtrsim \beta \log(|\lat|/\delta)/(\kappa s)^2$ and $m \gtrsim \beta \log(|\lat|/\delta)/s^2$ (the oracle Model-A upper attains the same two-axis rates as the Model-B lower bound, up to constants and logarithms, while no distribution-free unknown-weight matching upper is claimed), and indeed none can attain the Model-A rate over $\cW_B$ (Theorem~\ref{thm:modelB-incon}); the formal statement's pooled floor option (the $n$-route) only weakens the requirement.
Claim 3 is Theorems~\ref{thm:U1}--\ref{thm:U2} with Proposition~\ref{prop:I1} (proofs in Appendix~\ref{app:bprime}); the geometry sentence is Remark~\ref{rmk:localized}; the phase-diagram sentence is Corollary~\ref{cor:phase} (proved in this appendix from Theorem~\ref{thm:modelA}'s disjunctive sample condition and Corollaries~\ref{cor:feasibility}--\ref{cor:buy}); the within-Model-B$'$ two-resource-axis match asserted in the headline is Corollary~\ref{cor:bprime-minimax}; and the nuisance-necessity clause is Theorem~\ref{thm:t3-target} (Appendix~\ref{app:necessity}): the $K$-free core necessary (lattice-valued) in the known-$\eta$ submodel, the histogram $B^2 K$ split not proved necessary, unknown-$\eta$ open.
\hfill$\square$

\paragraph{Why no matching Model-B upper exists.}
In Model-B the floor side is estimable label-free but the risk side is not: $\GQ$ depends on $w$, unidentified beyond $w \le B$ (\S\ref{sec:models}).
The routes one might try either import estimable structure (B$'$'s pre-registered partition, at Proposition~\ref{prop:I1}'s nuisance price) or guard against the worst weight in the box, and Proposition~\ref{prop:separation} exhibits an instance where that envelope certifies no population-feasible point (refusing with probability $\ge 1 - |\lat|\delta_f'$) while B$'$ certifies at finite samples.
This is not an accident of those two routes: Theorem~\ref{thm:modelB-incon} shows that \emph{no} Model-B procedure attains the Model-A guarantee over $\cW_B$: over the full unknown-weight class the problem is not consistently solvable at any $(n,m)$, because certifying the risk requires detecting the alignment between weight-signs (seen only through target covariates) and loss-signs (seen only through source labels), a two-sample correlation-detection task the adversary defeats with weight complexity beyond any fixed sample size (Remark~\ref{rmk:modelB-mechanism}). The cross-model gap is therefore fundamental, and the B$'$ complexity restriction is necessary for consistency, not a convenience.
Positive certification thus requires \emph{either} oracle weights (Model-A) \emph{or} an explicit shift-complexity restriction (Model-B$'$'s pre-registered finite $\cK$): the inconsistency is a statement about the \emph{unrestricted} bounded-ratio class; it does not assert that a fixed-complexity unknown-weight model is inconsistent, and B$'$ is exactly such a fixed-complexity restriction, certifying at finite samples at its nuisance price.

\section{Relaxed Linear Crossings versus Discrete-Lattice Plateaus}
\label{app:exponent}

The exponent $-2$ in the bite divergence is tied to the \emph{linear} budget law~\eqref{eq:budgetlaw_app}, which is generic (Lemma~\ref{lem:genericity}).
Under bounded loss the relaxed budget in fact admits \emph{no other} nondegenerate exponent: every \emph{nondegenerate} frontier crossing is asymptotically linear, the only alternative being a degenerate zero-local-budget tie (Proposition~\ref{prop:exponent}).
The flat ``plateau'' behavior an early coarse-lattice pilot suggested is therefore not a property of the relaxed budget (a $\gamma = \infty$ relaxed exponent is a category error for a concave Lipschitz budget) but of a \emph{fixed finite lattice}, which we record separately as a constant-budget lower-bound lemma (Lemma~\ref{lem:plateau}).
We then explain the pilot's intermediate fitted slope as a mixture of the two.

\begin{proposition}[Relaxed crossings are linear: a dichotomy]
\label{prop:exponent}
Let $\budget$ be the relaxed budget of a world in $\cW_B$ with $\alpha \in [0,1]$, and assume a \emph{crossing} frontier: $0 < \betastar \le 1$ with $\budget(\betastar) = 0$ (the regular-margin setting; the non-crossing case $\betastar = 1$, $\budget(1) > 0$ and the empty-frontier case $\betastar = 0$ are excluded and treated as boundary cases below).
Since $\eta^*_Q \in [0,1]$, $\budget$ is concave and $1$-Lipschitz on $[0, \betastar]$ (its a.e.\ derivative $\alpha - \eta^*_Q \in [-1,1]$) with $\budget \ge 0$ there.
Exactly one of the following holds as $s \downarrow 0$:
\begin{enumerate}
    \item[(i)] \emph{(nondegenerate crossing)} $\budget(\betastar - s) > 0$ for some $s \in (0, \betastar]$; then $\budget(\betastar - s) = \kappa_* s + o(s)$ with $\kappa_* = -\budget'_-(\betastar) \in (0, 1]$. A regular margin $(\kappa, s_0)$, $\eta^*_Q \ge \alpha + \kappa$ a.e.\ on $[\betastar - s_0, \betastar]$, then holds for every fixed $\kappa \in (0, \kappa_*)$ once $s_0$ is small enough (in general only $\kappa \le \kappa_*$, with the attainable $\kappa \uparrow \kappa_*$ as $s_0 \downarrow 0$; the lower-bound constant $\kappa$ need \emph{not} equal the limiting slope $\kappa_*$). The labeled-sample scale $n \gtrsim \beta/(\kappa s)^2$ in this regime is \emph{not} a corollary of the dichotomy but Theorem~\ref{thm:lower}'s two-point bound under its own hypotheses (Remark~\ref{rmk:exp-rate}): a necessary condition for that construction, not a matching $\asymp$;
    \item[(ii)] \emph{(degenerate tie)} $\budget \equiv 0$ on $[0, \betastar]$, i.e.\ $\eta^*_Q = \alpha$ a.e.\ there (the zero-local-budget case, left open).
\end{enumerate}
No relaxed exponent $\gamma \neq 1$ arises: $\gamma < 1$ would force $\budget(\betastar - s)/s \to \infty$, violating Lipschitzness, while $\gamma > 1$ (a fortiori $\gamma = \infty$) would force $\budget(\betastar - s) = o(s)$, i.e.\ case (ii).
\end{proposition}

\begin{proof}
Set $f(s) := \budget(\betastar - s)$ for $s \in [0, \betastar]$: $f$ is concave, $1$-Lipschitz, $f(0) = 0$, and $f \ge 0$.
By concavity and $f(0) = 0$, for $0 < s < t \le \betastar$ we have $f(s) \ge \tfrac{s}{t} f(t)$, so $s \mapsto f(s)/s$ is nonincreasing on $(0, \betastar]$; the $1$-Lipschitz bound gives $0 \le f(s)/s \le 1$.
Hence $\kappa_* := \lim_{s\downarrow0} f(s)/s = f'_+(0)$ exists in $[0,1]$, with $\kappa_* = -\budget'_-(\betastar)$.
If $f(t) > 0$ for some $t \in (0, \betastar]$, then $\kappa_* \ge f(t)/t > 0$, giving (i) with $f(s) = \kappa_* s + o(s)$.
For the regular-margin claim of (i), write $f(s) = \int_0^s (\eta^*_Q(\betastar - v) - \alpha)\, dv$ (since $f'(v) = \eta^*_Q(\betastar - v) - \alpha$ a.e.\ and $f(0) = 0$); the integrand is nonincreasing in $v$ (the rearrangement $\eta^*_Q$ is nondecreasing) and tends to $\kappa_*$ as $v \downarrow 0$, so for every fixed $\kappa < \kappa_*$ it is $\ge \kappa$ a.e.\ on $(0, s_0)$ for $s_0$ small enough, a regular margin $(\kappa, s_0)$, while any regular-margin $\kappa$ gives $f(s)/s \ge \kappa$, hence $\kappa \le \kappa_*$.
Otherwise $f(t) = 0$ for all small $t > 0$; a nonnegative concave $f$ with $f(0) = 0$ that vanishes on an interval at the origin vanishes on all of $[0, \betastar]$ (else concavity $f(\lambda u) \ge \lambda f(u)$ is violated), so $f \equiv 0$, i.e.\ $\eta^*_Q = \alpha$ a.e.\ on $[0, \betastar]$, which is (ii).
The cases are exhaustive and mutually exclusive.
The exponent claim follows since $f(s) \asymp s^\gamma$ forces $f(s)/s \to \infty$ ($\gamma < 1$) or $f(s)/s \to 0$ ($\gamma > 1$), the latter making the nonincreasing ratio identically $0$, i.e.\ case (ii).
\end{proof}

\begin{remark}[The labeled-axis rate is imported from Theorem~\ref{thm:lower}, not the dichotomy]
\label{rmk:exp-rate}
Proposition~\ref{prop:exponent}(i) fixes only the \emph{geometry} of the relaxed budget, $\budget(\betastar - s) = \kappa_* s + o(s)$. The labeled-sample scale $n \gtrsim \beta/(\kappa s)^2$ is the separate two-point (Le~Cam) lower bound of Theorem~\ref{thm:lower} (Appendix~\ref{app:naxis}) and inherits \emph{all} of its hypotheses: a world in $\cW_B$ with regular margin $(\kappa, s_0)$, the local regime $s \in (0,\, s_0 \wedge c_0 \beta]$, valid coverage $\beta + s \le 1$, confidence $\delta \le 1/8$, and the construction's interior Bernoulli levels in $[0.1, 0.95]$. It is a necessary condition for that construction family, not a matching $\asymp$, and does \emph{not} follow from concavity or the regular margin alone; at a boundary Bernoulli level the per-sample KL can be linear, not quadratic, in the perturbation.
\end{remark}

A lower bound on the required sample size that does \emph{not} carry the linear-crossing $s^{-2}$ factor is thus a phenomenon of \emph{coarse-lattice} certification (where, between adjacent lattice coverages, the operative budget does not shrink with the within-cell floor position; Remark~\ref{rmk:lattice}) rather than of the relaxed budget; we record a per-floor lower bound exhibiting it directly.

\begin{lemma}[Finite-lattice constant-budget plateau]
\label{lem:plateau}
Fix $\alpha \in [0.2, 0.9]$, $\delta \le 1/8$, and a constant $\budget_0 > 0$.
Fix a floor cell $[\beta_-, \beta_+] \subseteq (0, 0.6]$ with $\budget_0 \le \tfrac18 \beta_- \min\{\alpha, 1-\alpha\}$.
For \emph{each} floor $\beta \in [\beta_-, \beta_+]$ there are worlds $W_{1,\beta}, W_{2,\beta} \in \cW_B$ and a finite nested-threshold lattice containing the coverage-$\beta$ threshold, whose accepted block carries risk budget $\budget_0$, on which no $(\alpha, \beta, \delta)$-valid lattice procedure certifies the feasible world $W_{1,\beta}$ with probability $\ge 1/2$ once $n \lesssim \beta/\budget_0^{2}$.
This per-$\beta$ two-point lower bound is $\Omega(\beta/\budget_0^{2})$ and \emph{contains no $s^{-2}$ factor}: holding the accepted block's budget fixed at $\budget_0$ (instead of letting it shrink to $\kappa s$ as on a refining grid) removes the $s^{-2}$ growth of the linear-crossing scale of Proposition~\ref{prop:exponent}(i). The relaxed slack of each construction is itself bounded below by a constant ($\betastar = 3\beta/2$ below, so $s = \beta/2$), so the absence of an $s^{-2}$ factor is genuine, not a vanishing-$s$ artifact. (We do \emph{not} claim a single fixed world or one lattice point serving the whole cell $[\beta_-, \beta_+]$; each $\beta$ uses its own $W_{1,\beta}$ and coverage-$\beta$ threshold.)
\end{lemma}

\begin{proof}
Fix $\beta \in [\beta_-, \beta_+]$ and instantiate the $n$-axis pair of Appendix~\ref{app:naxis} with the constant budget $\budget_0$ in place of $\kappa s$.
Take $\cX = [0,1]$, $P_X = Q_X = \mathrm{Unif}[0,1]$ (so $w \equiv 1$) and thresholds accepting $[0,t]$; both worlds have Bernoulli conditional losses $L \mid X \sim \Bern(\eta_j(X))$ and share $(P_X, Q_X)$. World $W_1$ sets $\eta_1 = \alpha - a$ on the deep-safe block $G = [0, g)$ ($g = \beta/2$) and $\eta_1 = \alpha + b$ on the hot block $H = [g,1]$, with $b = 2\budget_0/\beta$, $a = 4\budget_0/\beta$, so $\budget_1(\beta) = g(a - b) = \budget_0$; the restriction $\budget_0 \le \tfrac18\beta\min\{\alpha, 1-\alpha\}$ (which holds since $\beta \ge \beta_-$) gives $a \le \alpha/2$, $b \le (1-\alpha)/2$.
World $W_2$ differs from $W_1$ \emph{only} in the conditional loss on the slice $U = [0, g) = G$, namely $\eta_2 = \eta_1 + \Delta\,\mathbf 1_U$ (every other component identical), so $\rho = P_X(U) = Q_X(U) = g$ and $\Delta = 2\budget_0/\rho = 4\budget_0/\beta$.
Using $\alpha \in [0.2, 0.9]$:
\emph{(range)} $\eta_2|_U = \alpha - a + \Delta = \alpha$ and all Bernoulli levels lie in $[0.1, 0.95]$ ($\alpha - a \ge \alpha/2 \ge 0.1$, $\alpha + b \le 0.95$);
\emph{(global infeasibility)} for any threshold of coverage $t \ge \beta > g$ the slice is fully accepted, so $\budget_2(t) = \budget_1(t) - \Delta\rho = \budget_1(t) - 2\budget_0 \le \budget_0 - 2\budget_0 < 0$: every floor-clearing threshold is risk-infeasible in $W_2$;
\emph{(KL)} only the mass-$g$ slice changes and $w \equiv 1$, so the per-sample $\KL \asymp g\Delta^2 \asymp \budget_0^2/\beta$ and the $n$-sample $\KL \asymp n\budget_0^2/\beta$;
\emph{(Le~Cam)} $W_1$ is feasible at $\beta$ ($\budget_1(\beta) = \budget_0 > 0$, coverage $\ge \beta$); if a valid procedure certified it with probability $\ge 1/2$, then since validity forces refusal in $W_2$ with probability $\ge 1 - \delta \ge 7/8$, Pinsker ($\TV \le \sqrt{\KL_n/2} \le 1/4$ for $\KL_n \le 1/8$) yields a contradiction once $n \le c\,\beta/\budget_0^2$ for an absolute $c$ (indeed $P_1(\mathrm{cert}) \le P_2(\mathrm{cert}) + \TV \le \tfrac18 + \tfrac14 = \tfrac38 < \tfrac12$, ruling out equality too).
The pair is built per $\beta$: its $a, b, g, \Delta$ scale with $\beta$, and its relaxed frontier is $\betastar = g + ag/b = 3\beta/2$ (a constant relaxed slack $s = \beta/2$, never $\downarrow 0$). The conclusion is a \emph{per-floor} lower bound: for each $\beta \in [\beta_-, \beta_+]$ this two-point bound is $\Omega(\beta/\budget_0^2)$ with no $s^{-2}$ factor; we make no single-fixed-world or single-lattice-point claim across the cell.
\end{proof}

\paragraph{Resolution of the coarse-grid pilot exponent.}
Sweeping $\beta$ across the cells of a coarse lattice mixes \emph{linear} relaxed segments (Proposition~\ref{prop:exponent}(i), Le~Cam lower-bound scale $\gtrsim s^{-2}$) with \emph{constant-budget plateau} segments (Lemma~\ref{lem:plateau}, per-floor two-point lower-bound scale $\Omega(\beta/\budget_0^2)$ with no $s^{-2}$ factor), so a log--log fit across such a sweep can return a fitted slope strictly between $-2$ and $0$.
We regard this as a \emph{plausible} discreteness explanation for the intermediate exponent seen in an early coarse-lattice pilot (not a theorem about the regression coefficient, whose fitted value also depends on cell widths, sampling weights, and censoring), and the coarse-lattice required-$n$ behavior driving it is the archived pilot record (Appendix~\ref{app:experiments}: zero-margin cells with required-$n = \infty$).
The attribution was registered as a prediction before the refined run: on a continuum-threshold profile with regular margin the fitted exponent was predicted to approach $-2$; on a coarse lattice, to flatten within cells.
The refined registered family then returned slope $-2.002$ with 95\% CI $[-2.142, -1.862]$ inside the pre-registered band (\S\ref{sec:bite}), consistent with Proposition~\ref{prop:exponent}(i); we treat the intermediate-exponent attribution as a registered, plausible discreteness explanation rather than proven.
Zero-margin cells (required-$n = \infty$) are the extreme of the plateau phenomenon and are observed empirically (Appendix~\ref{app:experiments}).

\section{The Non-Derivability Analysis behind Table~\ref{tab:theory_comparison}}
\label{app:nonderiv}

Table~\ref{tab:theory_comparison} compresses a line-by-line analysis of whether each nearest verified prior line yields, by direct adaptation, any of the paper's five objects: the hard floor as a \emph{finite-sample} certified object under covariate shift, $\beta$-indexed lower bounds, the two-resource $(n, m)$ map, the localized accepted-region functional, and the nuisance-priced implementable upper bound.
This appendix expands each row.

\paragraph{SCoRE \citep{bai2026score}.}
SCoRE controls general selection-conditioned risks through weighted e-values, including a covariate-shift treatment with an estimated-weight robustness analysis.
Adding a floor test to SCoRE plausibly recovers a U1-style validity statement (this is the concession that opens \S\ref{sec:related}), but nothing in the framework produces a lower bound, a feasibility frontier, or rate theory: the robustness analysis prices estimated weights into the e-value, not into a certifiable-slack threshold, and no quantity indexed by $\beta$ appears.
The two-resource split cannot be read off either, because no result separates the labeled and unlabeled sample roles.

\paragraph{LEC \citep{wang2025lec} and BalanceRAG \citep{jia2026balancerag}.}
LEC owns the ratio linearization this paper inherits ($\RQ \le \alpha \iff \E[S(L - \alpha)] \le 0$); BalanceRAG owns cascaded two-dimensional lattice calibration with multiplicity accounting.
Both are exchangeable: no shift, hence no weights, no frontier in $Q$, and no $\beta$-indexed statement.
BalanceRAG's $\alpha$-infeasibility declaration (refusing when no lattice point passes) is an honest-refusal mechanism but not a coverage floor: it certifies nothing about $\CQ$, and no lower bound accompanies it.

\paragraph{Weighted and marginal conformal risk control under shift \citep{tibshirani2019conformal, zecchin2025generalization}.}
This family owns the weight machinery: weighted exchangeability, weighted CRC generalization bounds, and the estimated-weight repair toolkit \citep{yang2024doubly, kato2023double, wang2026weight, laghuvarapu2026kmmcp}.
The certified estimand is marginal risk or coverage (not the selection-conditioned ratio), and no floor constraint co-exists with the risk statement.
The complexity quantities that appear are global ESS-type functionals of $w$; the localized accepted-region functional cannot be recovered because no accepted set is distinguished by the analysis.
No lower bound appears in the cited works of this line (we make no claim about the broader literature).

\paragraph{High-probability risk control under shift \citep{almeida2025high}.}
We separate this work from the marginal line above, because its estimand reaches further than that grouping suggests.
Theorem 9 gives high-probability control under target covariate shift for a \emph{conditional} risk $\E_{Q}[L(X, Y; \lambda) \mid (X, Y) \in A(\lambda)]$, and instantiates the result for the false discovery rate; taking $A(\lambda)$ to be the accepted set makes that object exactly the selection-conditioned ratio $\E[L S_\lambda]/\E[S_\lambda]$ we certify.
So the ratio estimand under shift is within its reach, and we mark that capability present in Table~\ref{tab:theory_comparison}.
What does not follow is the rest of the map: the risk statement stands alone, with no hard coverage floor certified jointly against it, hence no second resource and no $\beta$-indexed lower bound; complexity is not reported through a localized accepted-region functional; and the estimated-weight treatment proceeds under an assumed weight accuracy rather than a finite-sample weight-estimation-error guarantee, so the nuisance is not priced as an explicit increment matched from below.
The non-derivability claim for this line is therefore about the floor-indexed map, not about the ratio estimand.

\paragraph{Non-monotone CRC \citep{aldirawi2026conformal}.}
The nearest line in distribution-free lower-bound theory proves a matching minimax lower bound for conformal risk control under non-monotone losses, using a Fano argument over a finite tuning-parameter grid.
It is single-axis (excess risk in $n$) and floor-free; the analysis is extended to distribution shift by importance weighting, but no coverage-floor parameter appears, so there is no frontier, no second resource, and no unlabeled axis.
Remark~\ref{rmk:vsnonmonotone} details why neither statement implies the other despite the shared minimax toolkit.

\paragraph{Reject-option classifiers \citep{franc2023optimal}; selective QA under shift \citep{kamath2020selective}.}
The reject-option line poses the floor as a \emph{population} object: the bounded-abstention model fixes a coverage floor and minimizes selective risk, the bounded-improvement model fixes risk and maximizes coverage, and the optimal strategy is characterized (a Bayes classifier with a randomized selection function).
It supplies no finite-sample certificate, no covariate-shift weighting, no $\beta$-indexed lower bound, and no labeled/unlabeled resource split; it is the population precedent for the \emph{object}, not for any of the five finite-sample contributions here.
Selective QA under domain shift \citep{kamath2020selective} studies exactly the deployment regime of our real-workload audit (abstain under shift to keep accuracy high), but with a learned calibrator and no validity guarantee or floor certificate, so it motivates the audit rather than competing with its certificate.

\paragraph{Inherited machinery, credited.}
We inherit and do not claim: weighted e-values and learn-then-test multiplicity \citep{angelopoulos2021learn, laufergoldshtein2023efficiently}, cascade multiplicity accounting \citep{jia2026balancerag}, the ratio linearization \citep{wang2025lec}, empirical-Bernstein and betting-based concentration \citep{maurer2009empirical, waudbysmith2023estimating}, and the Le~Cam two-point technique \citep{tsybakov2009introduction}.
What is new is not the floor as a \emph{population} object (the reject-option line \citep{franc2023optimal} already poses guaranteed-coverage selective risk) but the floor as a \emph{finite-sample} certified object under covariate shift, and everything the floor then creates: the frontier, the $\beta$-indexed two-resource map, the localized functional, and the budgeted-nuisance implementable upper bound.

\section{Proofs for the Model-\texorpdfstring{B$'$}{B'} Certificate}
\label{app:bprime}

This appendix proves Theorems~\ref{thm:U1} and~\ref{thm:U2} and Proposition~\ref{prop:I1}, and records the separation of the budgeted certificate from the trivial worst-case-box baseline.
The setting is Assumption~\ref{ass:bprime} with the four-block splitting of Algorithm~\ref{alg:bprime}: the labeled source sample splits into $D_w^P$ (size $n_w$) and $D_r$ (size $n_r$); the unlabeled target sample splits into $D_w^Q$ (size $m_w$) and $D_f$ (size $m_f$); all four blocks are mutually independent, and only covariates are used from $D_w^P$.
Throughout, $A_\lambda$ denotes the union of cells with $\Sacc = 1$, $K_\lambda = |A_\lambda|_\cK$ the number of accepted cells, and $Z_i^\lambda = \hat w(X_i)\, \Sacc(X_i)\, (L_i - \alpha)$ the risk-test variable on $D_r$, with range contained in $[-\alpha B, (1-\alpha)B]$ of width $\le B$.
Standing parameter domain: $\delta \le 1/8$ (the standing confidence of Theorem~\ref{thm:law}), $B \ge 2$, $\alpha \in [0,1]$ (and $0 < \alpha < 1$ wherever a betting denominator appears, Proposition~\ref{prop:score}), $s, \kappa > 0$, $0 < c_0 \le 1$, and all empirical-Bernstein sample sizes $n_r, m_f, n_w, m_w \ge 2$ (the Maurer--Pontil endpoints carry $N - 1$ denominators). The power-side radii carry $\log(2/\delta_r'') = \log(4/\delta_f') = \log(16|\lat|/\delta)$; under $\delta \le 1/8$ (so $\log(|\lat|/\delta) \ge \log 8$) these are $\le \tfrac{7}{3}\log(|\lat|/\delta)$, absorbed into the displayed absolute constant $C$, so conditions written with $\log(|\lat|/\delta)$ are correct up to $C$.
The budgets are
\begin{equation}
\label{eq:nuisancebudget}
\nb \;=\; 2\Big[\sqrt{\tfrac{K_\lambda\, \tilde q(A_\lambda)}{m_w}} + \sqrt{\tfrac{2 L_w}{m_w}}\Big]
\;+\; 2B\Big[\sqrt{\tfrac{K_\lambda\, \tilde p(A_\lambda)}{n_w}} + \sqrt{\tfrac{2 L_w}{n_w}}\Big],
\qquad
L_w := \log\tfrac{4(|\lat| + 1)}{\delta_w},
\end{equation}
with the plug-in corrections $\tilde q(A) := 2\hat q(A) + 4L_w/m_w$ and $\tilde p(A) := 2\hat p(A) + 4L_w/n_w$; Proposition~\ref{prop:I1}(iii) guarantees $q(A) \le \tilde q(A)$ and $p(A) \le \tilde p(A)$ on the same event, and the simultaneity log-factor $L_w$ is carried in full.
We write $r_\lambda := \E_P[|\hat w - w|\,\Sacc]$ for the \emph{true} localized weight error and keep $\nb$ (the right-hand side of~\eqref{eq:nuisancebudget}) for the \emph{computable radius}; the good event is $E_w = \{\forall\lambda:\ r_\lambda \le \nb\}$, and every computable risk test uses the margin $-\nb$, never the unknown true error $r_\lambda$.

\subsection{Proof of Proposition~\ref{prop:I1} (stratified-shift histogram estimator)}

\begin{propositionrestate}[I1, restated]
Under Assumption~\ref{ass:bprime}, with $\hat p, \hat q$ the empirical cell frequencies from $n_w$ source and $m_w$ target draws and $\hat w_k = \clip(\hat q_k/\hat p_k, 0, B)$ (with $\hat w_k := B$ when $\hat p_k = 0$, as in Algorithm~\ref{alg:bprime}, so $\hat w$ is defined on every outcome): with probability $\ge 1 - \delta_w$, simultaneously for every region $A \in \{A_\lambda\}_{\lambda \in \lat} \cup \{\cX\}$,
\begin{equation*}
\E_P\big[\,|\hat w - w|\, \mathbf{1}_A\,\big]
\;=\; \sum_{k :\, \cX_k \subseteq A} p_k\, |\hat w_k - w_k|
\;\le\; 2\Big[\sqrt{\tfrac{|A|_\cK\, q(A)}{m_w}} + \sqrt{\tfrac{2 L_w}{m_w}}\Big]
+ 2B\Big[\sqrt{\tfrac{|A|_\cK\, p(A)}{n_w}} + \sqrt{\tfrac{2 L_w}{n_w}}\Big],
\end{equation*}
and (iii) the same bound holds with $(p, q)$ replaced by $(\tilde p, \tilde q)$, making the budgets~\eqref{eq:nuisancebudget} computable.
\end{propositionrestate}

\begin{proof}
\emph{Step 1 (good event).}
By multiplicative Chernoff, $\Prob(\hat p_k < p_k/2) \le \exp(-n_w p_k/8) \le \delta_w/(4K)$ under Assumption~\ref{ass:bprime}(ii); a union over $k$ gives $\hat p_k \ge p_k/2$ for all cells with probability $\ge 1 - \delta_w/4$.

\emph{Step 2 (decomposition).}
On that event, for each $k$, since clipping to $[0, B] \ni w_k$ can only help,
\begin{equation*}
p_k |\hat w_k - w_k|
\;\le\; p_k \Big|\frac{\hat q_k}{\hat p_k} - \frac{q_k}{p_k}\Big|
\;=\; \frac{p_k}{\hat p_k} \Big|(\hat q_k - q_k) - w_k (\hat p_k - p_k)\Big|
\;\le\; 2 |\hat q_k - q_k| + 2B |\hat p_k - p_k|.
\end{equation*}

\emph{Step 3 (simultaneous localized $\ell^1$ concentration).}
For a fixed union-of-cells region $A$ (the bound below is then made simultaneous over the stated family), write $D_A = \sum_{k \in A} |\hat q_k - q_k|$.
By Jensen and Cauchy--Schwarz, $\E D_A \le \sum_{k \in A} \sqrt{q_k/m_w} \le \sqrt{|A|_\cK\, q(A)/m_w}$.
$D_A$ has bounded differences $2/m_w$ in each target draw, so McDiarmid's inequality gives $D_A \le \E D_A + \sqrt{2 \log(1/\varepsilon)/m_w}$ with probability $\ge 1 - \varepsilon$ (we set $\varepsilon = \delta_w/(4(|\lat|+1))$ per region below).
To keep both localization and simultaneity rigorous, apply McDiarmid separately to the $|\lat| + 1$ regions $\{A_\lambda\}_{\lambda \in \lat} \cup \{\cX\}$ at level $\delta_w/(4(|\lat| + 1))$ each: the deviation term becomes $\sqrt{2 L_w/m_w}$ with $L_w = \log(4(|\lat|+1)/\delta_w)$, which is exactly the form carried in~\eqref{eq:nuisancebudget}.
The same argument applies to the source frequencies with prefactor $2B$ and denominator $n_w$.

\emph{Step 4 (combine).}
Sum the Step-2 bound over $k \in A$ and apply the Step-3 bounds to the two $\ell^1$ terms.

\emph{(iii) Plug-in correction.}
For $Y = m_w \hat q(A) \sim \mathrm{Bin}(m_w, q(A))$, the event $\{q(A) > 2\hat q(A) + 4L_w/m_w\}$ equals $\{Y < m_w q(A)/2 - 2L_w\}$.
If $m_w q(A) < 4L_w$ this event is empty (its threshold $m_w q(A)/2 - 2L_w < 0$ while $Y \ge 0$), so assume $m_w q(A) \ge 4L_w$, whence $u := m_w q/2 + 2L_w \le m_w q$ lies in the lower-tail domain.
By the Bernstein--Chernoff lower tail $\Prob(Y \le m_w q - u) \le \exp(-u^2/(2 m_w q))$, and AM--GM giving $u^2 \ge 4 \cdot (m_w q/2)(2 L_w) = 4 m_w q L_w$, the probability is $\le e^{-2L_w} \le (\delta_w/(4(|\lat|+1)))^2$, comfortably within the per-region allocation; likewise for $p(A)$ with $n_w$ (the event empty when $n_w p(A) < 4L_w$).
Call $G_w$ the intersection of the full collection of good events (Chernoff good event, per-region McDiarmid for two marginals, per-region plug-in for two marginals); its complement stays within $\delta_w$, i.e.\ $\Prob(G_w) \ge 1 - \delta_w$, and on $G_w$ \emph{all} of the Step-1, Step-3, and part-(iii) deviation bounds (in particular the two-sided cellwise controls $|\hat q(A) - q(A)| \le D^q_A$, $|\hat p(A) - p(A)| \le D^p_A$ used by the power-side bridge below) hold simultaneously.
Substituting $\tilde q(A) \ge q(A)$ and $\tilde p(A) \ge p(A)$ into the Step-3/4 bound yields computable budgets dominating the population bound on the same event.
\end{proof}

\emph{Rate remark.}
Globally ($A = \cX$) the computable radius scales as $O\big(\sqrt{(K + L_w)/m_w} + B\sqrt{(K + L_w)/n_w}\big)$ in the regime $m_w, n_w \gtrsim \min\{K, L_w\}$ (otherwise the plug-in cross terms $\sqrt{K L_w}/m_w$ and $B\sqrt{K L_w}/n_w$, from $\tilde q, \tilde p$ at $A = \cX$, add; an upper bound, with the $L_w$ simultaneity term shown explicitly; we reserve $\asymp$ for a matched two-sided rate), the $\sqrt{K/N}$ rate of a $K$-parameter shift model up to the confidence term; the localized budgets shrink with the accepted masses $q(A_\lambda), p(A_\lambda)$: the nuisance, like the variance, is an accepted-region quantity. The \emph{true} error $r_\lambda$ can be smaller (even zero for matched histograms); only the radius $\nb$ is rate-controlled.

\begin{lemma}[Empirical-Bernstein endpoint with the biased variance]
\label{lem:ebbiased}
Let $Z, Z_1, \dots, Z_n$ be i.i.d.\ in $[0,1]$ with $n \ge 2$, let $\gamma \in (0,1)$, put $L := \log(2/\gamma)$, and let $\hat V := \tfrac1n \sum_{i}(Z_i - \bar Z)^2$ be the \emph{biased} empirical variance. Then with probability at least $1 - \gamma$,
\begin{equation}
\label{eq:ebbiased}
\E[Z] \;\le\; \bar Z + \sqrt{\tfrac{2 \hat V L}{n}} + \tfrac{7 L}{3(n-1)} .
\end{equation}
\end{lemma}

\begin{proof}
If $7L/(3(n-1)) \ge 1$ the right-hand side is at least $\bar Z + 1 \ge \E[Z]$ and there is nothing to prove; assume $7L/(3(n-1)) < 1$ and set $x := L/n < L/(n-1) < 3/7$.
Split the confidence \emph{asymmetrically}: $\delta_A := e^{-L + x}$ and $\delta_B := e^{-L}(2 - e^{x})$, both positive since $x < 3/7 < \log 2$, with $\delta_A + \delta_B = 2e^{-L} = \gamma$.
Write $a := \log(1/\delta_A) = L(n-1)/n$ and $b := \log(1/\delta_B) = L - \log(2 - e^{x})$.
Bernstein's inequality for variables of range $1$ at level $\delta_A$ gives $\E[Z] - \bar Z \le \sqrt{2\sigma^2 a/n} + a/(3n)$ with $\sigma^2 = \Var(Z)$; the variance bound of \citet{maurer2009empirical} at level $\delta_B$ gives $\sigma \le \sqrt{V} + \sqrt{2b/(n-1)}$ with $V := \tfrac{1}{n-1}\sum_i (Z_i - \bar Z)^2 = \tfrac{n}{n-1}\hat V$. A union bound costs $\delta_A + \delta_B = \gamma$, and on the intersection
\begin{equation*}
\E[Z] - \bar Z \;\le\; \sqrt{\tfrac{2 V a}{n}} + 2\sqrt{\tfrac{ab}{n(n-1)}} + \tfrac{a}{3n},
\qquad \text{where} \qquad
\tfrac{2 V a}{n} \;=\; \tfrac{2}{n}\cdot\tfrac{n \hat V}{n-1}\cdot\tfrac{L(n-1)}{n} \;=\; \tfrac{2\hat V L}{n},
\end{equation*}
so the root term is exactly the one in~\eqref{eq:ebbiased}: the asymmetric split is chosen precisely to make $a$ cancel the $n/(n-1)$ of the unbiased variance.
It remains to bound the two additive terms by $7L/(3(n-1))$, i.e.\ to show $C := 2\tfrac{n-1}{n}\sqrt{b/L} + \tfrac{(n-1)^2}{3n^2} \le \tfrac73$.
Since $-\log(2 - e^{x}) \le 2x$ on $[0, 3/7]$ we have $b \le L + 2x = L(1 + 2/n)$, hence $\sqrt{b/L} \le \sqrt{1 + 2/n} \le 1 + 1/n$ and
\begin{equation*}
C \;\le\; 2\tfrac{n-1}{n}\big(1 + \tfrac1n\big) + \tfrac{(n-1)^2}{3n^2}
\;=\; \tfrac{7n^2 - 2n - 5}{3n^2} \;=\; \tfrac73 - \tfrac{2}{3n} - \tfrac{5}{3n^2} \;<\; \tfrac73 . \qedhere
\end{equation*}
\end{proof}

\begin{remark}[Released kernel: the registered per-cell radius]
\label{rmk:releasedradius}
The radius analyzed above, \eqref{eq:nuisancebudget}, unions over the $|\lat| + 1$ \emph{regions} $\{A_\lambda\}_{\lambda \in \lat} \cup \{\cX\}$, two McDiarmid events per region (one per marginal), which is the source of its $L_w = \log(4(|\lat| + 1)/\delta_w)$.
The released implementation computes a different registered closed form, written $\rho^{\vtwo}_\lambda$, which unions over \emph{cells} instead.
With per-event confidence $\gamma_w = \delta_w/(4K)$ and $L_w^{\vtwo} := \log(2/\gamma_w) = \log(8K/\delta_w)$, set the empirical-Bernstein mass bands
\begin{equation*}
e^P_k = \sqrt{\tfrac{2 \hat p_k (1 - \hat p_k) L_w^{\vtwo}}{n_w}} + \tfrac{7 L_w^{\vtwo}}{3(n_w - 1)},
\qquad
e^Q_k = \sqrt{\tfrac{2 \hat q_k (1 - \hat q_k) L_w^{\vtwo}}{m_w}} + \tfrac{7 L_w^{\vtwo}}{3(m_w - 1)},
\end{equation*}
put $\tilde p_k := \max\{\hat p_k - e^P_k,\, 0\}$ and let the per-cell ratio error be $\mathrm{err}_k := B$ when $\tilde p_k = 0$ and $\mathrm{err}_k := \min\{B,\ (e^Q_k + \hat w_k e^P_k)/\tilde p_k\}$ otherwise; then
\begin{equation}
\label{eq:nuisancebudgetv2}
\rho^{\vtwo}_\lambda \;=\; \sum_{k \,:\, \cX_k \cap A_\lambda \neq \emptyset} (\hat p_k + e^P_k)\, \mathrm{err}_k .
\end{equation}

\emph{The good event.}
Let $G^{\vtwo}_w$ be the event that all $4K$ one-sided mass deviations obey their bands ($K$ cells $\times$ two marginals $\times$ two sides, at level $\gamma_w$ each, so the union costs $4K\gamma_w = \delta_w$).
Cell indicators are $[0,1]$-valued and the bands above are exactly the endpoint of Lemma~\ref{lem:ebbiased} at $\gamma_w$, evaluated at the biased empirical variance $\hat p_k(1 - \hat p_k)$ that the released code uses; that lemma is stated for this normalization, so per cell and by the union, $\Prob(G^{\vtwo}_w) \ge 1 - \delta_w$.
The same lemma covers the released risk UCB and floor LCB, which use the same $(1/n)$ convention.

\emph{Domination on $G^{\vtwo}_w$.}
Fix a cell $k$ and write $u_k = \hat q_k/\hat p_k$ when $\hat p_k > 0$.
If $\tilde p_k = 0$ (in particular when $\hat p_k = 0$, where the algorithm sets $\hat w_k = B$) the branch $\mathrm{err}_k = B$ applies and $|\hat w_k - w_k| \le B$ because $w_k, \hat w_k \in [0, B]$.
Otherwise $\hat p_k > 0$, and there are two cases.
If $u_k \le B$ then $\hat w_k = u_k$, and the identity $\hat q_k p_k - q_k \hat p_k = \hat p_k(\hat q_k - q_k) + \hat q_k(p_k - \hat p_k)$ gives $|\hat w_k - w_k| \le (e^Q_k + \hat w_k e^P_k)/p_k$.
If $u_k > B$ then $\hat w_k = B \ge w_k$ (Model B$'$ assumes $w \le B$) and $B\hat p_k - \hat q_k < 0$, so $p_k(B - w_k) = Bp_k - q_k \le B|p_k - \hat p_k| + |\hat q_k - q_k| \le Be^P_k + e^Q_k$, giving the same bound.
In both cases $p_k \ge \tilde p_k$ on $G^{\vtwo}_w$, so $|\hat w_k - w_k| \le (e^Q_k + \hat w_k e^P_k)/\tilde p_k$, and since $|\hat w_k - w_k| \le B$ always, $|\hat w_k - w_k| \le \mathrm{err}_k$ for every cell.
Because $\hat w$ and $w$ are constant on cells, and $P(\cX_k \cap A_\lambda) \le p_k \le \hat p_k + e^P_k$ on $G^{\vtwo}_w$,
\begin{equation*}
r_\lambda \;=\; \sum_{k \,:\, \cX_k \cap A_\lambda \neq \emptyset} P(\cX_k \cap A_\lambda)\, |\hat w_k - w_k| \;\le\; \rho^{\vtwo}_\lambda .
\end{equation*}
This covers the union-of-cells acceptance regions Assumption~\ref{ass:bprime}(iii) requires and, as an extension beyond that assumption, lattices whose $A_\lambda$ cuts cell interiors: each intersecting cell is charged its full upper mass, which over-counts and is therefore conservative.

\emph{What transfers, and what does not.}
Theorem~\ref{thm:U1} uses the radius only through the domination event $E_w = \{\forall \lambda:\ r_\lambda \le \nb\}$ at budget $\delta_w$, the power-side bridge of Theorem~\ref{thm:U2} being the step that needs the fuller good event $G_w$; since $G^{\vtwo}_w$ supplies exactly $E_w$ at the same budget, the validity statement holds verbatim with $\nb$ replaced by $\rho^{\vtwo}_\lambda$, and the violation counts we report for the B$'$ arm, and for the SCoRE arm that shares the radius, are counts for the released radius.
Because $G^{\vtwo}_w$ is \emph{cellwise}, this domination is simultaneous over every union of cells at once, hence over all of $\lat$ with no per-region union; that is why $L_w^{\vtwo}$ is indexed by $K$ rather than by $|\lat|$.
Neither radius dominates the other in general: in the well-resolved unsaturated regime the cellwise root terms of $\rho^{\vtwo}_\lambda$ carry an extra $\sqrt{L_w^{\vtwo}}$ relative to~\eqref{eq:nuisancebudget}, while at fixed $B$ and coarse, saturated, well-sampled configurations (for instance $K = 1$, $\hat p = \hat q = 1$, $n_w = m_w$ large, where $\rho^{\vtwo}_\lambda = O(L_w^{\vtwo}/n_w)$ against $\Theta(n_w^{-1/2})$) it is the smaller of the two.
Accordingly the power and sample-complexity statements (Theorem~\ref{thm:U2}, Corollary~\ref{cor:bprime-minimax}, and the rate remark above) are proved for~\eqref{eq:nuisancebudget}: a reader reproducing the rate analysis should use it, and a reader running the released code will obtain~\eqref{eq:nuisancebudgetv2}.
\end{remark}

\subsection{Proof of Theorem~\ref{thm:U1} (validity)}

\begin{theoremrestate}[U1, restated]
Under Assumption~\ref{ass:bprime}, for every world in $\cW_B$ with $\cK$-measurable $w$, Algorithm~\ref{alg:bprime} satisfies
$\Prob\big(\text{certify } \hat\lambda \wedge (\RQ(\hat\lambda) > \alpha \vee \CQ(\hat\lambda) < \beta)\big) \le \delta$.
\end{theoremrestate}

\begin{proof}
Define three failure events:
\begin{itemize}
    \item $E_w^c$ (the complement of the good event $E_w$): the nuisance guarantee fails.
    By Proposition~\ref{prop:I1}, $\Prob(E_w^c) \le \delta_w$, where $E_w = \{\forall \lambda :\ r_\lambda \le \nb\}$ (with $r_\lambda = \E_P[\,|\hat w - w|\, \Sacc\,]$ the true error and $\nb$ the computable radius~\eqref{eq:nuisancebudget}) is the $L^1$-bias \emph{consequence} of the full good event $G_w$ of the proposition's proof (so $G_w \subseteq E_w$ and $\Prob(E_w) \ge 1 - \delta_w$); validity uses only $E_w$, while the power-side bridge of Theorem~\ref{thm:U2} uses the full $G_w$ at the same budget. The budgets~\eqref{eq:nuisancebudget} are exactly Proposition~\ref{prop:I1}'s high-probability bounds.
    \item $E_r^c$: some risk UCB undershoots.
    Conditional on $D_w$, the function $\hat w$, the budgets $\nb$, and the scan set are all fixed; the $Z_i^\lambda$ are i.i.d.\ and bounded, and the Maurer--Pontil empirical-Bernstein bound \citep{maurer2009empirical} gives $\Prob(\UCB(\lambda) < \E_P[Z^\lambda] \mid D_w) \le \delta_r'$ per $\lambda$; a union gives $\Prob(E_r^c) \le |\lat| \delta_r'$.
    The bound is stated with the empirical variance, applied per $\lambda$ and unioned, hence valid under post-scan selection; the random threshold $-\nb(D_w)$ is harmless because the concentration event concerns $\E_P[Z^\lambda] \le \UCB(\lambda)$ conditional on $D_w$, and both sides of the pass test are $D_w$-measurable functions composed with $D_r$-randomness.
    \item $E_f^c$: some floor LCB overshoots: $\Prob(\LCB(\lambda) > \CQ(\lambda)) \le \delta_f'$ per $\lambda$ (one-sided binomial bound); union $\le |\lat| \delta_f'$.
\end{itemize}
On $E_w \cap E_r \cap E_f$, of probability $\ge 1 - \delta_w - |\lat|\delta_r' - |\lat|\delta_f' \ge 1 - \delta$, suppose $\hat\lambda$ is certified.
Floor: $\CQ(\hat\lambda) \ge \LCB(\hat\lambda) \ge \beta > 0$, which also rules out the $\CQ = 0$ convention case.
Risk: the bias identity (tower property, $\E[L - \alpha \mid X] = \eta(X) - \alpha$, splits independent) gives
\begin{equation*}
\E_P[Z^{\hat\lambda}]
\;=\; \E_P[\hat w\, S_{\hat\lambda}\, (\eta - \alpha)]
\;=\; \underbrace{\E_P[w\, S_{\hat\lambda}\, (\eta - \alpha)]}_{=\, \GQ(\hat\lambda)}
\;+\; \E_P[(\hat w - w)\, S_{\hat\lambda}\, (\eta - \alpha)],
\end{equation*}
and, with $r_{\hat\lambda} = \E_P[|\hat w - w| S_{\hat\lambda}]$ the true error, $|\E_P[(\hat w - w) S_{\hat\lambda}(\eta - \alpha)]| \le r_{\hat\lambda} \cdot \max(\alpha, 1 - \alpha) \le r_{\hat\lambda} \le \rho_{\hat\lambda}$ on $E_w$, using $\eta, \alpha \in [0, 1]$ (so the bias factor $\max(\alpha, 1 - \alpha) \le 1$); here $\rho_{\hat\lambda} = \nb|_{\lambda = \hat\lambda}$ is the computable radius.
The risk-pass condition gives $\E_P[Z^{\hat\lambda}] \le \UCB(\hat\lambda) \le -\rho_{\hat\lambda}$ on $E_r$ (the test margin is the computable radius, not the unknown true error $r_{\hat\lambda}$), hence $\GQ(\hat\lambda) \le \E_P[Z^{\hat\lambda}] + r_{\hat\lambda} \le -\rho_{\hat\lambda} + r_{\hat\lambda} \le 0$ on $E_w$ (as $r_{\hat\lambda} \le \rho_{\hat\lambda}$), i.e.\ $\RQ(\hat\lambda) \le \alpha$ (well-defined since $\CQ(\hat\lambda) > 0$).
\end{proof}

\subsection{Proof of Theorem~\ref{thm:U2} (power)}

\begin{lemma}[Power-side overshoot]
\label{lem:overshoot}
Conditional on $D_w$, for each $\lambda$, with probability $\ge 1 - 2\delta_r''$:
$\UCB(\lambda) \le \E_P[Z^\lambda] + \mathrm{Rad}(\lambda)$ with
$\mathrm{Rad}(\lambda) = 3\sqrt{2 \bar V_\lambda \log(2/\delta_r'')/n_r} + 11 B \log(2/\delta_r'')/(n_r - 1)$,
where $\bar V_\lambda = \Var_P(Z^\lambda)$.
\end{lemma}

\begin{proof}
Two events: (i) the lower Maurer--Pontil tail on the mean, applied to $-Z$: $\hat G(\lambda) \le \E_P[Z^\lambda] + \sqrt{2 \hat V \log(2/\delta_r'')/n_r} + 7B\log(2/\delta_r'')/(3(n_r - 1))$ with probability $\ge 1 - \delta_r''$; (ii) empirical-to-population variance concentration for bounded variables: $\hat V(\lambda) \le 2\bar V_\lambda + 4B^2 \log(2/\delta_r'')/(n_r - 1)$ with probability $\ge 1 - \delta_r''$ (the Maurer--Pontil sample-variance bound carries $n_r - 1$).
Substituting (ii) into the UCB's own radius term and collecting constants gives the display; the conversion of the two empirical-variance root terms to $\bar V_\lambda$ costs the enlarged linear constant $11$.
\end{proof}

\begin{lemma}[Floor-side LCB deficit]
\label{lem:lcbdeficit}
Let $Y \in [0, R]$ be bounded with population variance $V = \Var(Y)$, and let $\LCB$ be the Maurer--Pontil empirical-Bernstein lower confidence bound for $\E[Y]$ at level $\delta'$ from $N \ge 2$ i.i.d.\ draws. Then with probability $\ge 1 - 2\delta'$,
\begin{equation*}
\E[Y] - \LCB \;\le\; 3\sqrt{2 V \log(2/\delta')/N} + 11 R \log(2/\delta')/(N - 1).
\end{equation*}
\end{lemma}

\begin{proof}
This is Lemma~\ref{lem:overshoot} applied to $-Y$ (lower endpoint instead of upper). Two events: (i) the Maurer--Pontil lower tail on the mean gives $\bar Y \ge \E[Y] - [\sqrt{2\hat V\log(2/\delta')/N} + 7R\log(2/\delta')/(3(N-1))]$ with probability $\ge 1 - \delta'$; (ii) the empirical-to-population variance concentration $\hat V \le 2V + 4R^2\log(2/\delta')/(N-1)$ with probability $\ge 1 - \delta'$. Because the confidence \emph{endpoint} subtracts its own radius, $\E[Y] - \LCB = (\E[Y] - \bar Y) + (\bar Y - \LCB) \le 2\sqrt{2\hat V\log(2/\delta')/N} + \tfrac{14}{3}R\log(2/\delta')/(N-1)$: the deficit is twice the radius (\emph{both} terms doubled), not one radius. Substituting (ii) gives $2\sqrt{2\hat V\log(2/\delta')/N} \le 3\sqrt{2 V\log(2/\delta')/N} + 4\sqrt 2\,R\log(2/\delta')/(N-1)$, and $4\sqrt 2 + \tfrac{14}{3} < 11$ collects the display. (For a Bernoulli $Y = S_\lambda$ a Clopper--Pearson lower bound is also valid, with $V = \CQ(1-\CQ)$ exact and no separate variance event; its deficit constants differ from the displayed empirical-Bernstein ones, which are the route used by Algorithm~\ref{alg:bprime} and Theorem~\ref{thm:modelA}.)
\end{proof}

\begin{theoremrestate}[U2, restated]
Under Assumption~\ref{ass:bprime}, $\mathrm{LR}_{\mathrm{margin}}(s)$, and a regular margin $(\kappa, s_0)$ with $s \le s_0 \wedge c_0\beta$, Algorithm~\ref{alg:bprime} certifies every world with slack exactly $s = \betastar - \beta$ (the world's actual slack, not a lower bound on it) in two forms.
\emph{(i) Sharp (conditional) form.} Condition on $D_w$. Let $H(D_w)$ be the (random, $D_w$-measurable) size event that
$n_r \ge C\big(\bar V_{\lambda_s} \log(|\lat|/\delta)/(\kappa s)^2 + B \log(|\lat|/\delta)/(\kappa s)\big)$ and
$m_f \ge C\big(\beta \log(|\lat|/\delta)/s^2 + \log(|\lat|/\delta)/s\big)$,
where $\bar V_{\lambda_s} = \Var_P(Z^{\lambda_s} \mid D_w) \le \E_P[\hat w^2 S_{\lambda_s}] \le B\, \E_P[\hat w S_{\lambda_s}]$ is the realized estimated-weight variance proxy (a variance, bounded by the displayed accepted-region second moment, not equal to it).
For every realization $D_w \in E_w$ with $\rho_{\lambda_s} \le \kappa s/32$ and $H(D_w)$ holding, the conditional refusal probability obeys
$\Prob\big(\text{Algorithm~\ref{alg:bprime} outputs } \nocert \,\mid\, D_w\big) \le \delta_r' + \delta_f'$
(the algorithm certifies its floor-LCB-maximal passing point, not necessarily $\lambda_s$ itself; $\lambda_s$ passes both tests).
Marginalizing, $\Prob(\nocert) \le \delta_w + \Prob(\rho_{\lambda_s} > \kappa s/32) + \Prob(H(D_w)^c) + \delta_r' + \delta_f'$; form (ii) discharges the two middle terms at deterministic split sizes, giving the unconditional $1 - \delta$ guarantee.
\emph{(ii) Deterministic, unconditional.} At the population-mass split sizes $m_w \gtrsim K_{\lambda_s}\, q(A_{\lambda_s})/(\kappa s)^2 + L_w/(\kappa s)^2$ and $n_w \gtrsim B^2 K_{\lambda_s}\, p(A_{\lambda_s})/(\kappa s)^2 + B^2 L_w/(\kappa s)^2 + \log(4K/\delta_w)/p_{\min}$ (the last term the Assumption~\ref{ass:bprime}(ii) all-cell feasibility floor, needed for the Proposition~\ref{prop:I1} Chernoff event $\hat p_k \ge p_k/2$), with $n_r, m_f$ as in (i) and the deterministic envelope $\bar V_{\lambda_s} \le \tfrac54 B\,\CQ(\lambda_s) \lesssim B\beta$, the Proposition~\ref{prop:I1} good event $G_w$ secures $\rho_{\lambda_s} \le \kappa s/32$ at no extra confidence budget, so the certificate succeeds with probability $\ge 1 - \delta$ unconditionally (the bridge below).
\end{theoremrestate}

\begin{proof}
By $\mathrm{LR}_{\mathrm{margin}}(s)$ pick $\lambda_s$ with $\CQ(\lambda_s) \in [\beta + s/4,\, \betastar - s/4]$ and $-\GQ(\lambda_s) \ge \kappa s/8$.
On $E_w$ (Proposition~\ref{prop:I1}): $\E_P[Z^{\lambda_s}] \le \GQ(\lambda_s) + \rho_{\lambda_s} \le -\kappa s/8 + \rho_{\lambda_s}$.
Risk-pass requires $\UCB(\lambda_s) \le -\rho_{\lambda_s}$; by Lemma~\ref{lem:overshoot} it suffices that
$-\kappa s/8 + \rho_{\lambda_s} + \mathrm{Rad}(\lambda_s) \le -\rho_{\lambda_s}$, i.e.\ $\mathrm{Rad}(\lambda_s) \le \kappa s/8 - 2\rho_{\lambda_s}$,
and with the nuisance condition $\rho_{\lambda_s} \le \kappa s/32$ this reduces to $\mathrm{Rad}(\lambda_s) \le \kappa s/16$, which the risk condition guarantees (variance and range terms each $\le \kappa s/32$ for $C$ large enough).
The two Lemma-\ref{lem:overshoot} tails are folded by setting $\delta_r'' = \delta/(8|\lat|) \le \delta_r'$, so the two power-side risk tails sum to exactly $\delta_r' = \delta/(4|\lat|)$ (this per-tail allocation is internal to the \emph{power} proof and need not equal the \emph{validity} budget split $\delta_w = \delta/2$, $\delta_r' = \delta_f' = \delta/(4|\lat|)$ of Theorem~\ref{thm:U1}'s budget split that both certificates share). The half-levels $\delta_r''$ and (below) $\delta_f'/2$ are \emph{virtual endpoints internal to this power proof}: the released algorithm calibrates its risk UCB at the full $\delta_r'$ and its floor LCB at the full $\delta_f'$, as Algorithm~\ref{alg:bprime} displays. Remark~\ref{rmk:halflevel} supplies the monotone step that carries the conclusion from the virtual endpoints to the released ones, so the sample-size conditions proved here are sufficient for the algorithm as implemented.
Floor-pass: $\CQ(\lambda_s) - \beta \ge s/4$, and by Lemma~\ref{lem:lcbdeficit} applied at per-tail level $\delta_f'/2$ (to $Y = S_{\lambda_s} \in [0, 1]$, so $R = 1$ and $V = \CQ(1 - \CQ) \le \CQ$) the floor-LCB deficit is at most
$3\sqrt{2 \CQ(1 - \CQ)\log(4/\delta_f')/m_f} + 11\log(4/\delta_f')/(m_f - 1) \le s/4$
under the floor condition: each term is $\le s/8$ for $C$ large, the linear Bernstein term carried explicitly and dominated by the $\beta/s^2$ term precisely because $s \le c_0 \beta \le \beta$, and $\CQ(\lambda_s) \le \betastar = \beta + s \le 2\beta$. The two Lemma-\ref{lem:lcbdeficit} tails (mean deviation and variance concentration) sum to $\delta_f'$ via the half-split (whence the $\log(4/\delta_f')$); the doubling of the deficit relative to a single radius is what the enlarged constants and absolute $C$ absorb, exactly as on the risk side (Lemma~\ref{lem:overshoot}).
\emph{Unconditionally} (form (ii)), a union over the three events $E_w, E_r, E_f$ certifies with probability $\ge 1 - \delta_w - \delta_r' - \delta_f' \ge 1 - \delta$; \emph{conditionally} on a realization $D_w \in E_w$ (form (i)), $E_w$ is already secured, so only the two test tails $E_r, E_f$ enter and the conditional refusal probability is $\le \delta_r' + \delta_f'$ (the $\delta_w$ term reappears only when marginalizing back to the unconditional statement, as in the theorem's form-(i) display).
The histogram form of the nuisance condition is Proposition~\ref{prop:I1} solved for the split sizes, with the source term carrying $p(A_\lambda)$, not $q(A_\lambda)$; under additional two-sided overlap $w \ge 1/B$ one may substitute $p(A) \le B q(A)$, giving an all-$q$ form at a $B^3$ prefactor on the source split.
\end{proof}

\begin{remark}[Virtual half-levels versus the released endpoints]
\label{rmk:halflevel}
The power proof above evaluates the risk endpoint at $\delta_r'' = \delta_r'/2$ and the floor endpoint at $\delta_f'/2$, because each side folds two Lemma tails (a mean deviation and a variance concentration) into one budget.
The released implementation calibrates both tests at the \emph{full} levels $\delta_r'$ and $\delta_f'$.
The gap is closed by monotonicity of the endpoints in the confidence level, not by re-running anything.
Both endpoints are of the Maurer--Pontil form whose radius is increasing in $\log(2/\gamma)$, hence decreasing in $\gamma$: for $\gamma_1 \le \gamma_2$,
\begin{equation*}
\UCB_{\gamma_2}(\lambda) \;\le\; \UCB_{\gamma_1}(\lambda),
\qquad
\LCB_{\gamma_2}(\lambda) \;\ge\; \LCB_{\gamma_1}(\lambda),
\end{equation*}
on every realization.
Applying this with $\gamma_1 = \delta_r''$, $\gamma_2 = \delta_r'$ and with $\gamma_1 = \delta_f'/2$, $\gamma_2 = \delta_f'$: whenever the witness $\lambda_s$ passes the risk test at the virtual level ($\UCB_{\delta_r''}(\lambda_s) \le -\rho_{\lambda_s}$) it also passes at the released level, and whenever it passes the floor test at the virtual level ($\LCB_{\delta_f'/2}(\lambda_s) \ge \beta$) it also passes at the released level.
The certification event proved above is therefore contained in the certification event of the released algorithm, so the displayed sample-size conditions are sufficient for the algorithm as implemented, with the same probability bound.
Validity is unaffected and needs no such step: Theorem~\ref{thm:U1} is proved at the full levels $\delta_r'$, $\delta_f'$, which are exactly the levels the implementation uses.
The direction matters and only runs one way: a procedure calibrated at the full level certifies at least as often as one calibrated at the half level, so power transfers from the virtual endpoints to the released ones, while validity is established directly at the released ones.
\end{remark}

\paragraph{From random budgets to deterministic, world-dependent split-size envelopes.}
Theorem~\ref{thm:U2} is stated at the realized budget $\rho_{\lambda_s}$ and the estimated-weight variance proxy $\bar V_{\lambda_s}$, both random; we record the bridge to the displayed \emph{deterministic} split sizes \emph{on the single Proposition~\ref{prop:I1} good event $G_w$} (the full event of the proposition's proof, $\Prob(G_w) \ge 1 - \delta_w$), with no extra confidence budget.
These envelopes are world-dependent (they involve the witness $\lambda_s$ and its population masses), hence deterministic sufficient sizes rather than a computable stopping rule; the realized budget $\rho_{\lambda_s}$ itself is computable from data via Proposition~\ref{prop:I1}.
Part (iii) of Proposition~\ref{prop:I1} gives only the validity-side domination $q(A) \le \tilde q(A)$, $p(A) \le \tilde p(A)$; the power side needs the reverse direction, which is already available on $G_w$ from the same Step-3 control (no new tail event): since $|\hat q(A) - q(A)| \le D^q_A := \sqrt{K_{\lambda_s} q(A)/m_w} + \sqrt{2L_w/m_w}$ and likewise $|\hat p(A) - p(A)| \le D^p_A$, the plug-in budgets obey $\tilde q(A) = 2\hat q(A) + 4L_w/m_w \le 2q(A) + 2D^q_A + 4L_w/m_w$ and $\tilde p(A) \le 2p(A) + 2D^p_A + 4L_w/n_w$.
The witness obeys the margin relation $\kappa s/8 \le -\GQ(\lambda_s) = \E_Q[S_{\lambda_s}(\alpha - \eta)] \le \CQ(\lambda_s) = q(A_{\lambda_s})$, hence $\kappa s \le 8\, q(A_{\lambda_s})$ and, by bounded ratio, $\kappa s \le 8B\, p(A_{\lambda_s})$.
Substituting these and the displayed split sizes $m_w \gtrsim (K_{\lambda_s} q(A_{\lambda_s}) + L_w)/(\kappa s)^2$, $n_w \gtrsim B^2(K_{\lambda_s} p(A_{\lambda_s}) + L_w)/(\kappa s)^2$ into~\eqref{eq:nuisancebudget}, each of the four budget terms is $O(\kappa s/\sqrt C)$: the split sizes give $K_{\lambda_s} q(A_{\lambda_s})/m_w \le (\kappa s)^2/C$ and $L_w/m_w \le (\kappa s)^2/C$ separately, so $D^q_{A_{\lambda_s}} \le (1+\sqrt 2)\,\kappa s/\sqrt C$; with $\kappa s \le 8q(A_{\lambda_s})$ and $q(A_{\lambda_s}) \le 1$ this gives $\tilde q(A_{\lambda_s}) \le 2q(A_{\lambda_s}) + 2D^q_{A_{\lambda_s}} + 4L_w/m_w \le \big(2 + O(1/\sqrt C)\big)\, q(A_{\lambda_s})$, hence $\sqrt{K_{\lambda_s}\tilde q(A_{\lambda_s})/m_w} = O(\kappa s/\sqrt C)$; the source terms are identical under $\kappa s \le 8B\, p(A_{\lambda_s})$; so $\rho_{\lambda_s} \le \kappa s/32$ on $G_w$ for a large enough absolute constant $C$: the precise content of ``Proposition~\ref{prop:I1} solved for the split sizes,'' with no event outside $G_w$.
For the variance, the clipped estimator (Algorithm~\ref{alg:bprime}; $0 \le \hat w \le B$) and the importance identity $\CQ(\lambda_s) = \E_P[w S_{\lambda_s}]$ give, on $G_w$, $\bar V_{\lambda_s} \le \E_P[\hat w^2 S_{\lambda_s}] \le B\,\E_P[\hat w S_{\lambda_s}] \le B(\CQ(\lambda_s) + \rho_{\lambda_s})$; since $\kappa s \le 8\CQ(\lambda_s)$ forces $\rho_{\lambda_s} \le \kappa s/32 \le \CQ(\lambda_s)/4$, this is a deterministic envelope $\bar V_{\lambda_s} \le \tfrac{5}{4} B\,\CQ(\lambda_s) \lesssim B\beta$ (using $\CQ(\lambda_s) \le \beta + \tfrac34 s \le (1 + \tfrac34 c_0)\beta$ at the exact slack), so the $n_r$ condition is deterministic of order at most $B\beta\log(|\lat|/\delta)/(\kappa s)^2$; all tails fold into the single $\delta$-budget of Proposition~\ref{prop:I1} and Theorem~\ref{thm:U1}.

\paragraph{Matching discussion (binding language).}
For stratified Model-B$'$, the certification sample axes match the Model-B lower-bound orders of Theorem~\ref{thm:lower} up to constants and logarithms, with an additional histogram nuisance requirement depending on $K$ and the accepted-set masses.
The result applies in the regime $\nb \lesssim \kappa s$. The $K$-free core of this nuisance requirement \emph{is} minimax necessary for lattice-valued certification (Theorem~\ref{thm:t3-target}, Appendix~\ref{app:necessity}: a target $\Omega(q(A)/s^2)$ and the labeled $n$-axis); whether the full histogram $B^2 K$ rate is necessary reduces to the open unknown-$\eta$ question.

\paragraph{Proof of Corollary~\ref{cor:bprime-minimax} (two-axis rate correspondence within B$'$).}
\emph{Lower (re-run in B$'$).} The $n$-axis pair (Appendix~\ref{app:naxis}, $w \equiv 1$) and $m$-axis pair (Appendix~\ref{app:maxis}, $w$ piecewise-constant on $\{G, [g, g'), H'\}$) of Theorem~\ref{thm:lower} are weight-simple, hence $\cK$-measurable once the pre-registered score-refined $\cK$ aligns the three macro-boundaries to cell edges; the blocks then sit at $\cK$-cells and $\lat$-thresholds, so the worlds lie in the \emph{sample-independent} class $\cW_{B'}$ ($w$ $\cK$-measurable, $L \in [0,1]$, regular margin). Assumption~\ref{ass:bprime}'s feasibility ($p_{\min}$, met by the $\Omega(s)$ cell masses once $n_w \gtrsim \log/s$, and union-of-cells acceptance) is required only when Algorithm~\ref{alg:bprime} \emph{runs} on these worlds for the upper bound, not for lower-bound class membership. We do \emph{not} inherit Theorem~\ref{thm:lower} verbatim (a B$'$-valid procedure need not be valid on all of Model-B) but \emph{re-run} its two-point argument inside $\cW_{B'}$: both worlds of each pair are $\cK$-measurable, so a B$'$-valid procedure is bound on them, and the pair's $\TV \le 1/4$ bound is unchanged because knowing $\cK$ adds nothing to the data (per-cell weight estimates are measurable functions of the same $(n, m)$ draws; the $n$-axis pair differs only in source labels, the $m$-axis pair only in $Q_X$; Remark~\ref{rmk:bprime-match}). Le~Cam then forces refusal under the two displayed conditions.
\emph{Upper.} On $\cW_{B'}$ the hypotheses of Theorem~\ref{thm:U2} hold ($\mathrm{LR}_{\mathrm{margin}}(s)$ from the aligned $\lat$, regular margin $(\kappa, s_0)$, Assumption~\ref{ass:bprime}, and the stated $\rho_{\lambda_s} \le \kappa s/32$); it gives $n_r \asymp B\beta/(\kappa s)^2$ ($\beta/(\kappa s)^2$ at fixed $B$, via the deterministic envelope $\bar V_{\lambda_s} \le B\,\CQ(\lambda_s) \lesssim B\beta$ above) and $m_f \asymp \beta/s^2$, together with the weight splits $n_w \asymp B^2 K_{\lambda_s} p(A_{\lambda_s})/(\kappa s)^2$ and $m_w \asymp K_{\lambda_s} q(A_{\lambda_s})/(\kappa s)^2$ (plus $L_w$).
\emph{Match, totals, quantifier.} The test axes $n_r, m_f$ equal the lower-bound thresholds $\beta/(\kappa s)^2$, $\beta/s^2$ up to constants, logs, and $B$. The totals add the weight splits, with the additive $L_w$ and feasibility terms carried in full (not ``up to logs,'' since the bare $L_w/(\kappa s)^2$ or all-cell $p_{\min}$ term can dominate): $n = n_r + n_w \asymp n_r$ iff $B^2(K_{\lambda_s} p(A_{\lambda_s}) + L_w)/(\kappa s)^2 + \log(4K/\delta_w)/p_{\min} \lesssim B\beta/(\kappa s)^2$, and $m = m_f + m_w \asymp m_f$ iff $(K_{\lambda_s} q(A_{\lambda_s}) + L_w)/(\kappa s)^2 \lesssim \beta/s^2$: sample-size comparisons distinct from the accuracy condition $\rho_{\lambda_s} \le \kappa s/32$. Both bounds run on the same pre-registered $(\cK, \lat)$ of mesh $< s/4$, an existence statement for this $s$ (not partition- or $s$-uniform), reconciling Theorem~\ref{thm:lower}'s after-$s$ lattice with Theorem~\ref{thm:U2}'s pre-registered one. The correspondence on the totals $(n, m)$ in the non-dominated regime is a \emph{pointwise (instance-dependent) rate match}, not a uniform minimax statement: the upper split sizes depend on the unknown world through $p(A_{\lambda_s}), q(A_{\lambda_s}), K_{\lambda_s}$, and $p_{\min}$, so the quantifier is $\forall W\, \exists N(W, s)$, not $\exists N(s)\, \forall W$ (a class with $p_{\min} \downarrow 0$ makes the feasibility term $\log(4K/\delta_w)/p_{\min}$ unbounded). A genuinely \emph{uniform} within-B$'$ minimax match holds only over a fixed subclass imposing uniform $p_{\min} \ge p_0$ and accepted-mass bounds $K_{\lambda_s} p(A_{\lambda_s}) \le M_P$, $K_{\lambda_s} q(A_{\lambda_s}) \le M_Q$, with both lower pairs shown to lie in it, which we do not claim here. On the sub-blocks $(n_r, m_f)$ it is a rate correspondence only; the total-sample lower bound does not transfer to a single block, since a B$'$ procedure may reuse $D_w^Q, D_w^P$ for testing. The weight axis is matched only on its $K$-free core (Theorem~\ref{thm:t3-target}). \hfill$\square$

\begin{remark}[Misspecification honesty]
\label{rmk:misspec}
If $w$ is not $\cK$-measurable, the guarantee of Theorem~\ref{thm:U1} degrades gracefully: all statements hold with $w$ replaced by its $\cK$-projection $\bar w = \E_P[w \mid \cK]$ and an additive model-bias term $b_\cK(\lambda) = |\E_P[(w - \bar w)\, \Sacc\, (\eta - \alpha)]|$ added to $\nb$.
The bias $b_\cK$ is not estimable from covariates alone and must be treated as an assumption, exactly parallel to the covariate-shift assumption itself.
Validity is therefore stated as conditional on the shift model, with the same epistemic status as covariate shift.
\end{remark}

\begin{remark}[The shift model is a dial, not a disguise]
\label{rmk:dial}
Known-$w$ (Model-A) supplies the exact function; here the $K$ numbers $w_k = q_k/p_k$ must be learned from finite samples, and the resulting nuisance floor is the real price; it appears explicitly in Theorem~\ref{thm:U2} and inflates the certifiable-slack threshold by $\gtrsim \nb/\kappa$ (an additive nuisance term; we do not pin its two-sided constant).
Setting $K = 1$ forces $w \equiv 1$ (degenerate, no shift); growing $K$ approaches general covariate shift at the cost of a $\sqrt{K}$ rate factor, with the feasibility constraint $K \log(4K/\delta_w) \lesssim n_w/8$ of Assumption~\ref{ass:bprime}(ii) advertised, not hidden.
\end{remark}

\subsection{Separation from the worst-case-box baseline}
\label{app:separation}

\begin{proposition}[Separation]
\label{prop:separation}
Consider the baseline that certifies $\lambda$ iff the worst-case reweighted accepted ratio over the whole weight box passes:
\begin{equation*}
\sup\big\{\E_P[w' \Sacc \eta]/\E_P[w' \Sacc] : w' \text{ $\cK$-measurable},\, w' \in [1/B, B],\, \E_P[w'] = 1\big\} \;\le\; \alpha
\end{equation*}
(plus a floor LCB).
There is an instance on which the baseline certifies \emph{no lattice point at the population level} (every point fails either the box-risk test or the true floor), so on the floor-LCB validity event it refuses, and it can certify only via a floor-LCB \emph{overshoot}, with probability at most $|\lat|\delta_f'$; meanwhile Algorithm~\ref{alg:bprime} certifies at finite, explicit samples by Theorems~\ref{thm:U1}--\ref{thm:U2}.
\end{proposition}

\begin{proof}
Take $B = 5$, $\alpha = 0.2$; accepted cells with $(p, \eta) = (0.25, 0.10)$ and $(0.25, 0.28)$; rejected cells of total mass $0.5$ carrying $\eta = 1$, placed immediately above the accepted cells; true $w \equiv 1$.
The true accepted ratio is $(0.25 \cdot 0.10 + 0.25 \cdot 0.28)/0.5 = 0.19 < \alpha$, with slack.
Fix floor $\beta = 0.49$ and a nested-threshold lattice whose accepting point takes exactly these two cells: coverage $D = 0.5 \ge \beta$ (floor cleared) and risk budget $\budget_Q = \E_Q[\Sacc(\alpha - \eta)] = 0.25(0.2 - 0.1) + 0.25(0.2 - 0.28) = 0.005 > 0$, so $\RQ = 0.19 \le \alpha$ holds with margin. The relaxed frontier sits at $\betastar = 0.50625$; we take the lattice to also contain the threshold accepting an additional $0.00625$ of the adjacent loss-one mass (split its first cell), so the \emph{lattice} frontier coincides with it, $\betastar_\lat = 0.50625$, and the operative slack is $s = \betastar_\lat - \beta = 0.01625$ (this extra high-risk threshold only worsens the box-baseline's worst-case ratio, so it does not help the baseline). The increasing rearrangement is $\eta^*_Q = 0.10$ on $[0, 0.25)$, $0.28$ on $[0.25, 0.5)$, and $1$ on $[0.5, 1]$, so over the left-frontier neighborhood $[\betastar - s, \betastar] = [0.49, 0.50625]$ the binding sub-interval $[0.49, 0.5)$ has margin $\eta^*_Q - \alpha = 0.28 - 0.2 = 0.08$ (the segment $[0.5, \betastar]$ has the larger margin $0.8$); hence the regular margin is $\kappa = 0.08$, not $0.8$. The risk-margin hypothesis of Theorem~\ref{thm:U2} is still discharged at the witness of coverage $0.5$, $\budget_Q = 0.005 \ge \kappa s/8 = 0.0001625$, as is the floor; only the constants in the required sample sizes grow.
The baseline's adversarial reweighting is feasible ($w' = 0.2$ on the good cell, $3.4$ on the bad cell, $0.2$ on the rejected mass; all values in $[0.2, 5]$ and $\E_P[w'] = 0.25 \cdot 0.2 + 0.25 \cdot 3.4 + 0.5 \cdot 0.2 = 1$) and yields accepted ratio
$0.25(0.2 \cdot 0.10 + 3.4 \cdot 0.28)/(0.25(0.2 + 3.4)) = 0.243/0.90 = 0.27 > \alpha$.
The same reweighting defeats every coverage-$\ge\beta$ lattice point: accepting any of the $\eta = 1$ rejected cells only raises the worst-case accepted ratio, while dropping either accepted cell leaves coverage below $\beta = 0.49$. The only point with box-worst-case risk $\le \alpha$ is the one accepting the single good cell ($\eta = 0.10$ alone, box ratio $0.10$ since a single homogeneous cell's ratio is reweighting-invariant), but its true coverage is $0.25 < \beta$, so it fails the floor at the population level. Hence \emph{no} lattice point is both box-risk-feasible and floor-feasible in population. On the simultaneous floor-LCB validity event (probability $\ge 1 - |\lat|\delta_f'$, Theorem~\ref{thm:U1}'s floor step) the baseline therefore refuses every point; it can certify the good-cell-only point only by a floor-LCB \emph{overshoot} of the true coverage $0.25$ past $\beta$, an event of probability $\le |\lat|\delta_f'$. The budgeted certificate, by contrast, certifies the displayed coverage-$0.5$ point once $\nb \le \kappa s/32$ and the risk and floor conditions of Theorem~\ref{thm:U2} hold.
\end{proof}

The baseline uses no data about $w$: it optimizes over its own uncertainty set instead of estimating localized weights, and it is a distinct object from the budgeted test (the budgeted test never passes as $\nb \to \infty$). Its uncertainty set is the two-sided box $w' \in [1/B, B]$, a baseline modeling choice distinct from the paper's one-sided class $w \in [0, B]$ (widening the box to $[0, B]$ only enlarges the worst case, so it only \emph{strengthens} the separation).
The comparison scope is deliberately qualified: we do not claim the baseline is never more powerful in finite samples.
The defensible statement is that whenever the baseline certifies with strict margin, the estimated-weight test also certifies once its nuisance and sampling radii are below that margin, and there exist instances (above) where the baseline certifies no population-feasible point (refusing with probability $\ge 1 - |\lat|\delta_f'$) while the estimated-weight test certifies with probability $\ge 1 - \delta$.
The DRO-box arm certifying 0/1{,}024 cells in \S\ref{sec:validity} is the empirical face of this proposition.

\subsection{The SCoRE-inspired betting arm: self-contained validity}
\label{app:score}
The SCoRE-inspired betting arm of \S\ref{sec:validity} (Appendix~\ref{app:experiments}) replaces B$'$'s empirical-Bernstein risk UCB by a weighted betting e-value, keeping every other component of Algorithm~\ref{alg:bprime} (the $\hat w$, $\rho_\lambda$, weight-free floor LCB, four-block split, budget, and output rule) verbatim. It is an equal-budget risk-test ablation rather than a source-faithful implementation of \citet{bai2026score}, and the statement below is proved for the arm exactly as implemented. We state its validity self-contained.

\begin{proposition}[Betting risk e-value for the SCoRE-inspired arm]
\label{prop:score}
Fix the pre-registered $\cK$ and, from $D_w$, the clipped histogram $\hat w$ and budgets $\rho_\lambda$ of Algorithm~\ref{alg:bprime}; let $E_w = \{\forall\lambda:\ \E_P[|\hat w - w|\,\Sacc] \le \rho_\lambda\}$ be the Proposition~\ref{prop:I1} bias event ($\Prob(E_w^c) \le \delta_w$). The proposition uses $\rho_\lambda$ only through $E_w$, so it holds verbatim for the released per-cell radius $\rho^{\vtwo}_\lambda$ of Remark~\ref{rmk:releasedradius}, which supplies that event at the same budget; the reported SCoRE arm is the released one. For each $\lambda \in \lat$ put $Z_i = \hat w(X_i)\,\Sacc(X_i)\,(L_i - \alpha)$ and $U_i = -(Z_i + \rho_\lambda)$ on $D_r$, and over a $D_w$-measurable bet grid $\{\theta_g\}_{g=1}^G \subset [0, \theta_{\max}]$, $\theta_{\max} = \tfrac{1}{2}\big((1-\alpha)B + \max_\lambda \rho_\lambda\big)^{-1}$, define the mixture e-value
\[ \bar E_\lambda \;=\; \tfrac1G \sum_{g=1}^G \prod_{i \in D_r}\big(1 + \theta_g U_i\big). \]
The SCoRE arm certifies risk at $\lambda$ iff $\bar E_\lambda \ge 1/\delta_r'$ ($\delta_r' = \delta/(4|\lat|)$), tests the floor exactly as Algorithm~\ref{alg:bprime} (Bernstein LCB on $D_f$ at $\delta_f' = \delta/(4|\lat|)$), and outputs the floor-LCB-maximal $\lambda$ passing both, else $\nocert$. Then on $E_w$, for every world with $\GQ(\lambda) > 0$, $\E[\bar E_\lambda \mid D_w] \le 1$; hence the arm is $(\alpha, \beta, \delta)$-valid over the stratified-shift class (Theorem~\ref{thm:U1}): the identical budget $\delta = \delta_w + |\lat|\delta_r' + |\lat|\delta_f'$, only the risk test differing.
\end{proposition}

\begin{proof}
Conditional on $D_w$ the points of $D_r$ are i.i.d., and $U_i$ is a function of $(X_i, L_i)$ and the $D_w$-fixed $(\hat w, \rho_\lambda)$; so the $U_i$ are i.i.d.\ given $D_w$, and for a fixed bet $\theta_g \ge 0$,
\[ \E\Big[\textstyle\prod_i (1 + \theta_g U_i) \,\Big|\, D_w\Big] = \prod_i\big(1 + \theta_g\,\E[U_i \mid D_w]\big) \le \exp\!\big(\theta_g\, n_r\, \E[U \mid D_w]\big), \]
using $1 + x \le e^x$. On $E_w$, $\E_P[\hat w \Sacc(L - \alpha)] = \GQ(\lambda) + \E_P[(\hat w - w)\Sacc(\eta - \alpha)] \ge \GQ(\lambda) - \rho_\lambda > -\rho_\lambda$ for a risk-unsafe world ($\GQ(\lambda) > 0$, $|\text{bias}| \le \rho_\lambda$), so $\E[U \mid D_w] = -(\E_P[\hat w \Sacc(L-\alpha)] + \rho_\lambda) < 0$ and each factor's conditional expectation is $\le 1$; averaging preserves $\E[\bar E_\lambda \mid D_w] \le 1$. By Ville/Markov, $\Prob(\bar E_\lambda \ge 1/\delta_r' \mid D_w) \le \delta_r'$ at a risk-unsafe $\lambda$; a union over the $\le |\lat|$ scanned $\lambda$ caps risk-side false certification at $|\lat|\delta_r' = \delta/4$ (selection among simultaneously valid e-value tests is free, as in Theorem~\ref{thm:U1}). With $\Prob(E_w^c) \le \delta_w = \delta/2$ and the weight-free floor bound at $|\lat|\delta_f' = \delta/4$ (Theorem~\ref{thm:U1}'s floor step, unchanged), the total false-certification probability is $\le \delta$. The grid range keeps every factor positive: the most negative $U_i = -((1-\alpha)B + \rho_\lambda) \ge -1/(2\theta_{\max})$, so $1 + \theta_g U_i \ge 1 - \theta_g/(2\theta_{\max}) \ge \tfrac12 > 0$. \hfill$\square$
\end{proof}

This is the formal arm behind the head-to-head of \S\ref{sec:validity}/Appendix~\ref{app:experiments}: a faithful \emph{batch} (counts-computable, order-free) weighted e-value, with a per-sample sequential betting e-process the higher-fidelity follow-up. Its certify-frequency weakly dominates B$'$'s on the audited grid (Appendix~\ref{app:experiments}); validity here is by Proposition~\ref{prop:score}, not by the $0$ observed violations.

\section{Additional Experimental Records}
\label{app:experiments}

This appendix carries the protocol details and the appendix-level figures and tables referenced from \S\ref{sec:synthetic} and \S\ref{sec:real}.

\subsection{Synthetic protocol}
\label{app:protocol}
The full validity grid (\S\ref{sec:validity}) is the B1 grid: 4 shift intensities $\times$ 256 parameter cells $=$ 1{,}024 cells, 1{,}000 replications per cell (five higher-variance edge cells re-run at 2{,}000--8{,}000).
Six arms run on every cell: oracle-A (known weights, run with the released two-family oracle kernel of Remark~\ref{rmk:releasedoracle}, not the three-family procedure displayed under Theorem~\ref{thm:modelA}), B$'$ (Algorithm~\ref{alg:bprime}, run with the released nuisance radius of Remark~\ref{rmk:releasedradius}), and four demonstration arms: floor-free certification with a post-hoc floor check, weighted conformal \citep{tibshirani2019conformal}, plug-in (certify whenever point estimates pass), and the DRO worst-case box of Proposition~\ref{prop:separation}.
Violation accounting per arm: a certifying replication violates if the true $(\RQ, \CQ)$ of the certified $\lambda$ breaks either constraint; Holm's correction \citep{holm1979simple} runs across certifying cells within each arm; per-cell 95\% Clopper--Pearson upper confidence bounds \citep{clopper1934use} are evaluated against $\dtol = 0.0625$ over powered cells (those with $\ge 47$ certifying replications), and the pooled CP-UCB is evaluated against $\delta = 0.05$.
These three quantities are reported separately throughout (Table~\ref{tab:validity_constants}).
Generators draw conditional-risk profiles and shifted targets from registered families; world parameters, lattice grids, and acceptance bands were fixed before the runs they judge.

\subsection{Contour transfer: qualitative rate-shape only}
\label{app:contours}

\begin{figure}[t]
    \centering
    \includegraphics[width=0.95\textwidth]{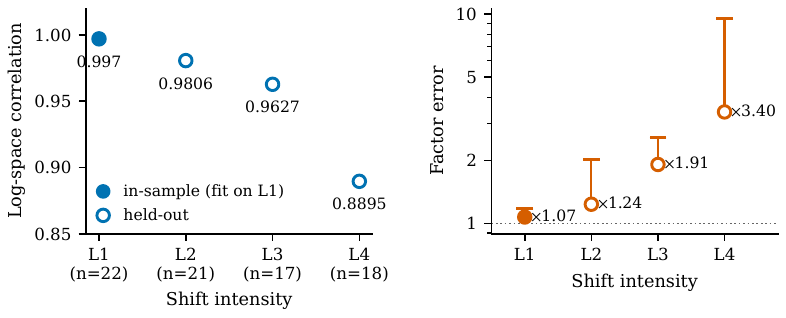}
    \caption{\textbf{Contour transfer is qualitative rate-shape only} (amended shape
    metrics: the authoritative reproducible record). Two-axis constants
    $(c_n, c_m)$ are least-squares fit in log space on the L1 contour points and
    applied \emph{unchanged} to L2--L4 (held-out). (a)~Log-space correlation between
    empirical and predicted boundary crossings per shift intensity: in-sample L1
    $.997$; held-out $.981/.963/.890$ at L2/L3/L4. (b)~Multiplicative factor error of
    the predicted boundary location: median (markers) with whiskers to the 90th
    percentile; held-out medians $\times1.24/\times1.91/\times3.40$. The strict
    pre-registered $\pm3\%$ bands failed at all intensities: ordering and rate shape
    transfer; constants do not.}
    \label{fig:contour_shape}
\end{figure}

Figure~\ref{fig:contour_shape} records the contour-transfer analysis behind \S\ref{sec:twoaxis} (Claim C3), from the amended shape metrics, which are the authoritative reproducible record.
Constants fit on the first intensity and held fixed track the shape and ordering of the certifiable boundary at higher intensities but miss its location by growing multiplicative factors; the strict pre-registered $\pm 3\%$ bands failed at all intensities.

\subsection{Localized geometry: recorded confounds}
\label{app:geometry}

The localized-geometry evidence is in the main text (Figure~\ref{fig:rank_correlation}, \S\ref{sec:geometry}, Claim~C4); two recorded confounds bound the claim's scope.
First, family-2 pooling: 162 of 640 worlds censor at the required-$n$ cap and the pooled rank correlation drops to $.503$ even though within-pair ratio agreement is $.995$; the within-pair structure is the meaningful signal, and we report it as qualified.
Second, the oracle-arm pooled correlations are not supportive of the localized functional, a range-constant artifact: the oracle arm's required-$n$ range is compressed by its much smaller constants (Table~\ref{tab:validity_constants}, bottom), leaving little rank variation for the functional to explain; we report this rather than excluding the arm.

\subsection{Two-axis CI-separation: failure at registered density, recovery at higher density}
\label{app:surfacecv}

\begin{figure}[t]
    \centering
    \includegraphics[width=0.95\textwidth]{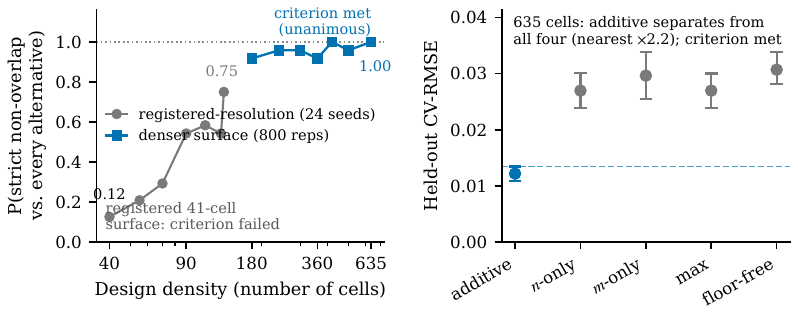}
    \caption{\textbf{Two-axis CI-separation: a sample-limited failure at the registered
    density that recovers at higher density} (Claim C2; surface CV, \emph{not} contour
    evidence). (a)~Probability that the additive form's held-out CV-RMSE strictly
    non-overlaps \emph{every} alternative, as design density grows: the
    registered-resolution run (24 seeds) rises $0.12$ (40 cells) $\to 0.75$ (134), and
    the denser surface (800 reps) continues $0.92 \to 1.00$ by 635 cells. At the
    registered 41-cell density the criterion fails (intervals overlap); it is
    met (unanimously across all 24 re-splits) by 635 cells. (b)~At 635 cells the
    additive held-out CV-RMSE ($0.0122$, 95\% CI $[0.0109, 0.0134]$; dashed guide at its
    upper CI) is disjoint from all four alternatives, the nearest $\times 2.2$ worse.
    This is the $s^*(n,m)$ surface-CV record (a Claim-C2 consistency check on the
    theorem-carried form); the CV-RMSE is not evidence of contour-constant transfer.}
    \label{fig:ci_separation}
\end{figure}

Figure~\ref{fig:ci_separation} records the model-comparison run behind the first half of \S\ref{sec:twoaxis} (Claim C2): at the registered 41-cell density the additive form attains the best held-out point estimate (surface-CV $R^2 = 0.9202$) but the registered CI-separation criterion fails (intervals overlap), while re-running the \emph{same} criterion at increasing design density resolves it; the criterion is met outright, and unanimously, by 635 cells (full numeric record: Appendix~\ref{app:revision}). The 41-cell failure was statistical power, not non-identifiability.

\subsection{\texorpdfstring{$\beta \to 0$}{beta -> 0} consistency diagnostic}
\label{app:beta0}

\begin{figure}[t]
    \centering
    \includegraphics[width=0.48\textwidth]{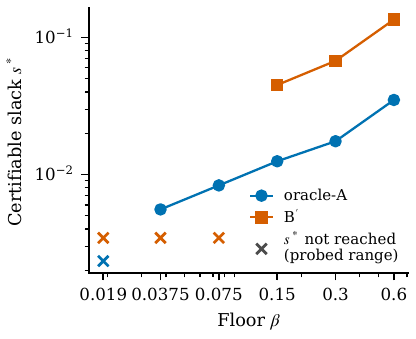}
    \caption{\textbf{$\beta \to 0$ consistency diagnostic} (registered
    theorem-aligned families): the smallest certifiable slack $s^*$ shrinks toward
    zero as the floor $\beta$ shrinks, for both oracle-A and B$'$, consistent with
    the floor-created reading of the map (both lower-bound axes vanish as
    $\beta \to 0$). Crosses mark floors where $s^*$ was not reached within the probed
    slack range (B$'$ for $\beta \le .075$; both arms at $\beta = .019$). No rate fit
    is drawn: none is recorded in the source run; the diagnostic is monotone
    vanishing, not an exponent estimate.}
    \label{fig:beta_vanishing}
\end{figure}

Because the floor-created reading of the map (Remark~\ref{rmk:floorcreated}) is front-loaded in the abstract and \S\ref{sec:intro}, Figure~\ref{fig:beta_vanishing} records its empirical consistency check: monotone vanishing of the smallest certifiable slack as $\beta$ shrinks, with censored floors marked rather than extrapolated.

\subsection{Sensitivity grid: refusal status only}
\label{app:sensitivity}

\begin{table}[t]
\centering
\caption{\textbf{Refusal-status robustness over registered nuisance choices (leg 2,
DeepSeek-V4-Flash).} Registered configuration $(K{=}16, B{=}10)$ inside the full $3\times3$ grid $K \in
\{8, 16, 32\} \times B \in \{5, 10, 20\}$ (9 cells), plus $\pm\tfrac{1}{2}$-cell lock-quantile jitter at the registered cell (2 rows; 11 variants in all);
full-calibration B$'$ on cached scores; no eval labels touched (status-only
sensitivity). All 11/11 variants return no-cert, with the risk side failing
all 64/64 lattice points and the floor side passing 15/64 in
every variant. This is refusal-status robustness only: it does not validate the stratified-shift
model (misspecification-to-breakage was not reached on this real workload; nuisance necessity is partly settled: a $K$-free core necessary in the known-$\eta$ submodel, the unknown-$\eta$ case open).}
\label{tab:sensitivity}
\small
\begin{tabular}{cccccc}
\toprule
$K$ & $B$ & Lock-quantile jitter & Status & Risk-side passes & Floor-side passes \\
\midrule
8 & 5 & 0 & no-cert & 0/64 & 15/64 \\
8 & 10 & 0 & no-cert & 0/64 & 15/64 \\
8 & 20 & 0 & no-cert & 0/64 & 15/64 \\
16 & 5 & 0 & no-cert & 0/64 & 15/64 \\
16 & 10 & 0 & no-cert & 0/64 & 15/64 \\
16 & 20 & 0 & no-cert & 0/64 & 15/64 \\
32 & 5 & 0 & no-cert & 0/64 & 15/64 \\
32 & 10 & 0 & no-cert & 0/64 & 15/64 \\
32 & 20 & 0 & no-cert & 0/64 & 15/64 \\
16 & 10 & $+\tfrac{1}{2}$ cell & no-cert & 0/64 & 15/64 \\
16 & 10 & $-\tfrac{1}{2}$ cell & no-cert & 0/64 & 15/64 \\
\bottomrule
\end{tabular}
\end{table}

Table~\ref{tab:sensitivity} gives the per-variant grid behind \S\ref{sec:sensitivity} (Claim C6): all 11 variants land in the same no-cert status, with identical risk-side ($0/64$) and floor-side ($15/64$) lattice counts throughout, so the leg-2 status does not move with the histogram resolution $K$, the bound $B$, or lock-quantile jitter.
The label is deliberate: this is refusal-status robustness only, not validation of the stratified-shift model; on this real workload misspecification-to-breakage was not reached. A dedicated synthetic sweep reaches it (the in-class breakage sweep of Appendix~\ref{app:revision}), and the nuisance necessity narrows to the unknown-$\eta$ edge (\S\ref{sec:notcertify}, Theorem~\ref{thm:t3-target}).

\subsection{Lattice diagnostic}
\label{app:latticediag}
The discrete-lattice phenomenon of Remark~\ref{rmk:lattice} and Appendix~\ref{app:exponent} is observed empirically: on coarse-lattice profiles, zero-margin cells return required-$n = \infty$ (no sample size certifies within the cell), and log--log exponent fits across cell sweeps land strictly between $-2$ and $0$, as the registered prediction states.
This is a lattice-design diagnostic: it indicates the lattice misses the frontier's margin, not that the map fails.

\subsection{Real-workload records and recorded deviations}
\label{app:realrecords}
The real-workload audit (\S\ref{sec:real}) ran on frozen, cached pipeline outputs with the SLA scenario $(\alpha, \beta, \delta) = (.10, .60, .05)$ fixed before any data was described.
Two deviations are on record, both availability-triggered and adjudicated under the registered escalation rule before any evaluation labels were read: leg 1 substituted Llama-3-8B \citep{grattafiori2024llama} for the originally registered model, and leg 2 re-entered with DeepSeek-V4-Flash as a conditional capability-ladder step, explicitly not independent validation, since it runs on the same corpus with the same eval set.
DeepSeek-V4-Flash had no citable public technical report at write time; it is therefore identified by name and by the registered deviation record rather than by a bibliography entry, and its inference outputs are hash-bound in the archived frozen cache (Appendix~\ref{app:repro}).
An access guard logged evaluation-label reads; the logs contain exactly one access, by the final evaluation.
The post-hoc frontier diagnosis uses 500 bootstrap resamples over evaluation items; the plug-in baseline's claim and violation rates are computed over its 1{,}000 per-leg certification replications (distinct from the 500 frontier bootstraps), leg-split.

\subsection{Extended empirical records: multi-family bite, B$'$ envelope, characterized failures}
\label{app:revision}
These records extend the main-text claims with the breadth, envelope, and failure-mode characterizations; every number traces to a committed result file.

\paragraph{Multi-family bite (Claim C1 breadth).}
The labeled-axis bite (required-$n$ vs.\ slack, oracle arm, floor freed) was re-run across three structurally distinct families; the steep $s^{-2}$ form holds across all three, the exact $-2$ being specific to the regular-margin family. The exponent ordering tracks the accepted-region margin geometry (margin-vs-slack exponent $b$):
\begin{center}\footnotesize
\begin{tabular}{@{}lcccc@{}}
\toprule
family & bite slope & 95\% CI & margin exp.\ $b$ & band $[-2.3,-1.7]$ \\
\midrule
continuum (registered) & $-2.002$ & $[-2.14,-1.86]$ & $1.10$ & in band \\
kcell (on-class) & $-1.617$ & $[-1.76,-1.47]$ & $0.91$ & below \\
non-theorem-aligned & $-1.845$ & $[-2.02,-1.67]$ & $0.99$ & -- \\
\bottomrule
\end{tabular}
\end{center}
The qualitative account (required-$n \propto \bar V/\mathrm{margin}^2$, so a sub-linear margin yields a shallower exponent) reproduces the ordering but not the exact exponent, which additionally carries finite-sample range terms.
\emph{Slope robustness (registered continuum family).} The continuum fit rests on 9 slack design points; a deterministic pairs bootstrap (seed $20260624$, $20{,}000$ resamples) returns a 95\% slope CI $[-2.11,-1.89]$ (median $-2.005$), consistent with the OLS normal-approximation CI $[-2.14,-1.86]$, while the log-space residuals stay $\le 0.10$ in magnitude and the quadratic-curvature CI contains $0$: no systematic departure from the linear log--log law. The two smallest-but-one slacks ($0.020,0.017$) resolve to the same rung of the required-$n$ search ladder (both $149{,}383$, while the smallest slack $0.014$ resolves to a higher rung, so this is a grid coincidence and not a ceiling), a mild flattening that does not bend the fit out of band.

\paragraph{B$'$ certifiable envelope and the $K$-premium (Claim C5; companion to Theorem~\ref{thm:t3-target}).}
B$'$ certifies across the partition sweep $K \in \{4,8,16,32,64\}$; its onset premium over oracle-A grows as $K^{0.955}$ (95\% CI $[0.755, 1.154] \ni 1$, Figure~\ref{fig:t1a_premium}), while the declared bound $B$ is second-order: onset is $B$-insensitive when the true ratios sit inside the bound and grows only as $B^{0.62}$ when they approach it (bounded-ratio coupling: heavy weight forces small source mass $p(A) \propto 1/B$, cancelling a power of $B$ in U2's $B^2 K\,p(A)$). The B$'$ labeled onset obeys the same $s^{-2}$ bite as oracle-A (slope $-2.005$, CI $[-2.20,-1.81]$). At $\kappa s = 0.042$, $K=4$ the envelope run sits in the same regime as the main-text constants (it does not re-derive them: the two use different onset conventions): its cert-frequency-$\ge\!\tfrac12$ onsets (oracle-A $6{,}367$, B$'$ $718{,}459$) fall inside the Table~\ref{tab:validity_constants} first-to-full-certification ranges ($4{,}096$--$16{,}384$ and $524{,}288$--${\sim}10^6$), and the plotted premium is measured on this $\ge\!\tfrac12$ onset ($\approx\!113\times$ at $K=4$, the same order as the $64\times$ full-certification ratio).

\begin{figure}[t]
\centering
\includegraphics[width=0.46\textwidth]{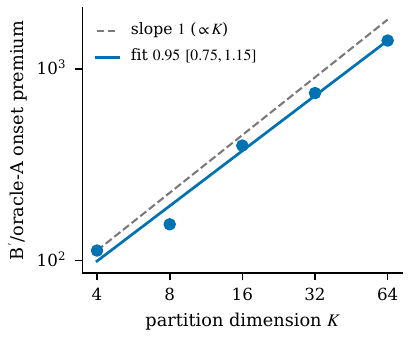}
\caption{\textbf{The B$'$ nuisance premium scales linearly in the partition dimension $K$.} Onset premium $=$ onset$_{\mathrm{B}'}(K)/$onset$_{\mathrm{oracle}}$, log--log; fitted slope $0.955$, 95\% CI $[0.755, 1.154] \ni 1$, against a reference slope-$1$ line. The $K$-premium is the histogram estimator's $\ell^1$ full-weight-vector recovery price (Proposition~\ref{prop:I1}), matched from below by the $K$-free necessary core (lattice-valued; Theorem~\ref{thm:t3-target}).}
\label{fig:t1a_premium}
\end{figure}

\paragraph{CI-separation recovers at higher design density (Claim C2).}
On the registered 41-cell $s^*(n,m)$ surface the additive form has the clearly best held-out CV-RMSE center ($0.0154$, $\approx 30\%$ below the nearest alternative, floor-free at $0.0220$) but the bootstrap CIs overlap. Re-running the \emph{same} registered model-comparison procedure (byte-identical criterion; only design density and Monte-Carlo depth raised: a pre-registration-allowed synthetic design axis, no formula or constant change) at increasing cell count, the probability of strict non-overlap against \emph{every} alternative rises $0.12, 0.21, 0.29, 0.54, 0.58, 0.54, 0.75$ at $40/55/70/90/110/130/134$ cells (the registered-resolution run, $24$ seeds per size) and continues $0.92, 0.96, 1.00$ at $180/240/635$ cells on a denser surface ($13^2 \times 5$ design, $800$ reps, $635$ certified cells), with the additive$\leftrightarrow$floor-free gap growing monotonically $+.004 \to +.016$ over this trajectory. At 635 cells the criterion is met outright: additive held-out CV-RMSE $0.0122$ $[0.0109, 0.0134]$ against the nearest (n-only / max-form) $0.0270$ $[0.0239, 0.0302]$ (a $2.2\times$ separation with the alternative's lower CI above the additive upper CI), and the verdict is unanimous across all $24$ re-splits (Figure~\ref{fig:ci_separation}). The 41-cell failure was statistical power, not non-identifiability: the additive two-axis surface is the genuinely best fit, recovered exactly as the registered sample-limited reading predicted (the per-axis exponent-containment sub-test is a separate check, not revisited here).

\paragraph{Contour-constant drift is systematic (Claim C3).}
The contour factor errors ($\times 1.07/1.24/1.91/3.40$ at L1--L4, L1 in-sample) are strictly monotone in shift intensity (rank correlation $1.0$) and grow $\times 3.2$ from L1 to L4 while the accepted-region variance proxy grows only $\times 1.3$; the log-correlation stays $\ge .889$ at every intensity. The non-transfer is a predictable, intensity-driven constant inflation with rate-shape preserved, not an unstructured breakdown; the constants drift, the ordering and shape transfer.

\paragraph{Misspecification-to-breakage, exhibited in-class (the B$'$ conditional-validity boundary).}
The real-workload sensitivity grid (\S\ref{sec:sensitivity}) reached only refusal-status invariance; a dedicated synthetic sweep exhibits the breakage boundary that B$'$'s $\cK$-measurability assumption guards. We redistribute target mass \emph{within} each accepted $\cK$-cell toward its higher-$\eta$ bins at \emph{fixed} per-cell totals: an adversarial member of $\cW_B$ ($w \le B = 10$ throughout the reported grid) that violates \emph{only} $\cK$-measurability, the within-accepted-region $w$--$\eta$ rank correlation rising from $-0.87$ (on-class) to $+0.47$. Because each cell total is preserved, the certificate's per-cell masses, weights $\hat w_k$, and nuisance budget $\rho_\lambda$ are unchanged, so B$'$ certifies the same aggressive near-frontier operating point at \emph{every} perturbation strength (certify-frequency $1.0$, spread $0.0$ across the sweep); yet truth-level violations jump from $0$ to $100\%$ once the receding true frontier crosses that locked point (onset $\epsilon = 0.8$ at $n = 4{,}194{,}304$: realized accepted risk $0.184 \!\to\! 0.208$ past $\alpha = 0.2$ as $\betastar$ falls $0.773 \!\to\! 0.715$). The bias is invisible to the certificate and is \emph{not} small-sample noise: at fixed $\epsilon = 1.0$ the violation rate \emph{rises} with $n$ ($0, 0, 1, 1$ at $n = 2^{18}\!\dots\!2^{24}$), since more data only makes B$'$ confident enough to certify the aggressive point the hidden bias then breaks. A same-worlds control pins this on misspecification rather than infeasibility: oracle-A (Model-A, true per-bin $w$) records $0$ violations at every $\epsilon$. This is an adversarial \emph{sensitivity/breakage} characterization (not an off-partition robustness guarantee and not a real-workload stratified-shift validation), and it \emph{confirms} the disclosed conditionality of Theorem~\ref{thm:U1}: validity holds for $\cK$-measurable $w$, and off-class the bias is not estimable from covariates alone (Remark~\ref{rmk:misspec}).

\paragraph{B$'$ positive certification: targeted existence check (companion to Claim C5).}
The full validity grid (\S\ref{sec:validity}) has B$'$ certify only $8/1{,}024$ cells, a harsh-constants slice. This is a targeted \emph{existence} check (not a breadth claim) that this low rate is not a fundamental inability of the implementable certificate. On the on-class \texttt{kcell} world ($K = 4$, $\mathrm{loc} = 2.0$, $B = 5$, $\alpha = 0.20$, a $50$-point lattice, $n = m$ over $2^{15}\!\dots\!2^{22}$, $2{,}000$ replications per cell), B$'$ positively certifies non-trivial operating points at both probed slacks: at $s = 0.20$ ($\beta = 0.691$) it first reaches certify-frequency $1.0$ at $n = m = 262{,}144$, with realized min true coverage $\approx 0.723$ and max true risk $\approx 0.150$; at $s = 0.12$ ($\beta = 0.771$) at $n = m = 1{,}048{,}576$ (the same onset as the main-text full-certification constant at this operating point), with min true coverage $\approx 0.805$ and max true risk $\approx 0.175$. Across all B$'$ certifying replications the violation count is $0/16{,}297$, pooled CP-UCB $1.8\times10^{-4}$. This denominator is certifying replications, so the interval is \emph{conditional} on a certificate having been issued; $\delta$ bounds the unconditional failure probability and does not control this conditional rate, so we report the interval without comparing it to $\delta$. This is targeted synthetic existence evidence only: it does \emph{not} show broad (multi-family/slack/$K$), practically easy (the onsets are large), or real-workload positive certification, and it does not overturn the $8/1{,}024$ rate; it reads it. Oracle-A is a reference arm here ($1/32{,}000$ certifying replications violated, a single near-boundary risk-side event at true risk $\approx \alpha$, pooled CP-UCB $1.5\times10^{-4}$).

\paragraph{SCoRE-inspired betting arm: the equal-construction risk-test head-to-head.}
\label{app:scorehh}
A SCoRE-inspired floor-augmented betting arm, in the spirit of \citet{bai2026score}, is evaluated here under identical conditions. We add it as a seventh arm whose $\hat w$, localized nuisance budget $\rho_\lambda$, weight-free floor LCB, four-block split, confidence budget, and output rule are \emph{verbatim-identical} to B$'$ (Algorithm~\ref{alg:bprime}); the \emph{only} difference is the risk test: a weighted betting e-value $\prod_i (1 + \lambda_{\mathrm{bet}} U_i)$, $U_i = -(Z_i + \rho_\lambda)$, mixed over a per-replicate bet grid, replaces B$'$'s empirical-Bernstein UCB. Conditional on $D_w$ the $D_r$ points are i.i.d., so for a fixed bet $\E[\prod_i(1+\lambda_{\mathrm{bet}} U_i) \mid D_w] = \prod_i(1+\lambda_{\mathrm{bet}}\E[U_i]) \le 1$ under the unsafe null on the same weight-good event B$'$ uses; the mixture is valid by averaging and the union over $|\lat|$ matches B$'$'s risk budget $\delta/4$ exactly (full statement and proof: Proposition~\ref{prop:score}; pre-registered). On the deep-audit subset (6 cells $\times$ 4{,}000 reps) SCoRE+floor records $0/24{,}000$ violations (pooled CP-UCB $1.2\times10^{-4} \le \delta$); across the 1{,}024-cell grid (matched reps) its certify-frequency is $\ge$ B$'$'s in \emph{every} cell, strictly greater in $8$, with $0$ violations and the \emph{same} powered-cell count (a non-trivial-certification count of $10$ vs B$'$'s $9$ at equal reps; the headline $8/583$ figure uses the registered escalated reps, so this matched-reps run is the head-to-head, not a re-derivation of the headline). The reading: this equal-construction e-value arm is valid and, in this fixed matched-replication run, weakly dominates B$'$ in \emph{observed} certify-frequency ($\ge$ in every cell, strictly $>$ in 8) \emph{without} enlarging the powered-cell set (a small, uniform margin); no paired test or interval is computed for that ordering, so it is an ordering observed in this Monte Carlo run and not an estimated power gap. B$'$'s practical conservatism is \emph{shared} by this alternative risk test; so the certificate's contribution is the lower-bound theory, two-resource map, and explicitly priced nuisance, not a finite-sample power advantage. (Determinism check: with the $\hat w$/$\rho$/floor block identical, B$'$ and oracle-A reproduce their $0$-violation tallies on these cells.)

\emph{Scope of the comparison.} This arm is an equal-budget \emph{risk-test ablation} inside B$'$'s architecture, not a source-faithful implementation of \citet{bai2026score}. That work builds conformal risk-adjusted e-values for individual test instances and thresholds them (with a reweighted extension under covariate shift), whereas this arm keeps B$'$'s weight estimator, nuisance radius, floor LCB, split, and output rule and bets on the calibration block; Proposition~\ref{prop:score} proves validity for the arm as implemented, and no claim is made that it reproduces SCoRE's selection rule or tuning. It is accordingly reported as the equal-construction alternative we tested against B$'$'s risk test, which is what the lower-bound thesis needs, and not as the strongest available same-target baseline; a source-faithful SCoRE comparison with its own selection rule and configuration is not run here.

\subsection{Pre-registered future work}
\label{app:futurework}
Of the three pre-registered directions, two are addressed in this work: multi-family bite sweeps on non-theorem-aligned families, and denser surface sweeps that \emph{resolve} the C2 CI-separation criterion (now met at 635 cells) and characterize the systematic C3 contour drift (Appendix~\ref{app:revision}); and one stands: an uncensored family-2 rerun with a fixed-range-constant oracle test.
The open theory edges connect to active lines: the nuisance-necessity question is now partly resolved (Theorem~\ref{thm:t3-target}, Appendix~\ref{app:necessity}: a $K$-free necessary core, lattice-valued, with unknown-$\eta$ the remaining edge) and the misspecification breakage boundary is now exhibited empirically (the in-class breakage sweep, Appendix~\ref{app:revision}); richer shift models beyond stratification \citep{laghuvarapu2026kmmcp, wang2025subpopulation, xu2025wasserstein, yang2026multidistribution, correia2025nonexchangeable}, multigroup transfer \citep{ye2026multicalibration}, and floor-aware extensions of deployment-facing risk control \citep{zollo2024prompt, wang2026inference}.

\subsection{Reproducibility and compute}
\label{app:repro}
Every headline number in the paper was produced by a pipeline in which derived-report scripts regenerate all aggregates from raw run files and figure scripts read only those files, hard-coding no numbers; rather than shipping result files, the released artifact regenerates the synthetic ones from pinned seeds and the real frontier maps from the shipped frozen evaluation arrays (the real refusal tallies additionally require rebuilding the inference cache from the external model weights).
The release carries a claim-to-command map; four of its entries need the specific subcommand that writes the artifact, rather than the parent command, and we state them here so the mapping is unambiguous: \texttt{run\_m2\_b2 analyze} writes the rank-correlation record (the sweeps must be run first, which the command hard-gates), \texttt{run\_t1a bsweep\_tight} is the sole producer of the tight-overlap B-sweep and is deliberately not part of \texttt{run\_t1a all}, \texttt{run\_t3 breakage} writes the misspecification-breakage record, and \texttt{derived\_reports sensitivity} writes the $K$/$B$/jitter sensitivity grid (the real-data lattice-stability record is the separate \texttt{real.certify\_real stability}).
The post-hoc bootstrap interval on the bite slope is computed by a script shipped alongside the run outputs rather than by a run command.
The synthetic suite runs at CPU scale (seconds to minutes per block); the two real legs consumed at most 2--4 GPU-days of cached frozen-pipeline inference in total, and the caches are hash-bound and archived read-only.

\paragraph{Generator and seeding (synthetic suite).}
All synthetic randomness derives from a single master seed $\sigma_0 = 20260611$ through NumPy \texttt{SeedSequence} spawning: each grid cell--replication pair draws an independent PCG64 stream keyed by $(\text{cell index}, \text{replication index})$, and each world-family draw is keyed by $(\text{family id}, \text{draw index})$, so every cell and replication is independently and reproducibly seeded.
A world fixes a finite bin partition with a declared $K$-cell stratified structure (the B$'$ class $\cK$) at a per-intensity bounded-ratio class constant $B \in \{2,3,4,5\}$ (one value per shift level L1--L4); on-class families set the density ratio cell-constant, off-class families perturb the ratio within cells (used only for the misspecification-direction sweeps).
The B1 validity grid is $4$ shift intensities $\times$ $256$ cells $\times$ $1{,}000$ replications, with per-cell levels $(\alpha, \delta) = (0.2, 0.05)$, a $50$-point nested-threshold lattice, and the floor $\beta$, frontier $\betastar$, and margin $\kappa$ recorded per cell in the grid files; the deep-audit subset re-runs $6$ cells at $4{,}000$ replications.
The certificate, arm, and generator definitions, the four pre-registration locks, and the per-cell grid files together fix every parameter behind the headline tables; the code, with the synthetic master seed pinned in it, and the text-free frozen real-data evaluation arrays with their numeric lock records are publicly available at \url{https://github.com/yahiko-l/certify-or-refuse} (a code-and-data release whose commands regenerate the synthetic results and the real frontier maps; reported result files are not shipped), while all datasets used are public (SQuAD and NewsQA via the MRQA distribution) and the synthetic worlds are fully parameterized here.

\section{Is the Model-B$'$ Nuisance Cost Minimax Necessary?}
\label{app:necessity}

\paragraph{The open edge, and the answer.}
Theorem~\ref{thm:U2} certifies in Model-B$'$ once the histogram split sizes satisfy the \emph{per-axis} conditions
\begin{equation}
\label{eq:t3-u2}
  m_w \;\gtrsim\; \frac{K_{\lambda_s}\, q(A_{\lambda_s})}{(\kappa s)^2},
  \qquad
  n_w \;\gtrsim\; \frac{B^2\, K_{\lambda_s}\, p(A_{\lambda_s})}{(\kappa s)^2}
  \qquad(\text{plus }L_w\text{-terms}),
\end{equation}
and Appendix~\ref{app:bprime} leaves the matching lower bound open. The honest resolution is \texttt{PARTIAL}: the $K$-bearing split sizes are \emph{not proved} necessary and are partly architecture artifacts, a $K$-free necessary core is established (for lattice-valued certification, the standing output model of \S\ref{sec:setup}), and one two-sided edge stays open:
\begin{quote}
\emph{The $K$-bearing split sizes \eqref{eq:t3-u2} are \textbf{not demonstrated} minimax necessary. (i) There is no separate population $p$-functional to estimate: by the identity below, the population feasibility and coverage functionals are independent of the source marginal $p$. Any lower bound matching the histogram's $n_w$ rate would therefore have to arise through the unknown-$\eta$ labeled-source/alignment problem, not from identifying a standalone $p$-functional; the task does not charge for $p$-estimation as a separate \emph{population} axis (a procedure may still estimate $p$, and $p$ does influence the finite-sample labeled cost through how often the label channel visits each cell). (ii) The factor $K$ on both axes is the $\ell^1$ price of recovering the entire $K$-dimensional weight vector (Proposition~\ref{prop:I1}), not an evident certification price: the \emph{established} necessary core is \emph{$K$-free} (the complete unknown-$\eta$ rate remaining open): a target term $\Omega(q(A)/s^2)$ (the floor $m$-axis localized; $\kappa$ cancels) and a labeled-source term $\Omega(\bar V/(\kappa s)^2)$ with $\bar V\le\E_P[w^2\Sacc]$ (the existing $n$-axis; worst-slice instances \emph{expected} to attain $\bar V\asymp B^2p(A)$). (iii) The $K$-free \emph{target} bound is proved in the \emph{known-$\eta$ submodel} (Theorem~\ref{thm:t3-target}); the $K$-free \emph{labeled-source} bound is the existing (varying-$\eta$) $n$-axis of Theorem~\ref{thm:lower}, read as a reduction. Whether the \emph{full unknown-$\eta$} Model-B$'$ also avoids $K$ (via a direct scalar-functional certificate) or instead reintroduces $K$ through the unknown-$\eta$ source$\leftrightarrow$target alignment (Theorem~\ref{thm:modelB-incon}) is the one genuinely open edge.}
\end{quote}
We prove the $K$-free target lower bound (Theorem~\ref{thm:t3-target}, rigorous), reduce the source lower bound to the map's $n$-axis (\S\ref{sec:t3-source}), show in which sense the $K$ in \eqref{eq:t3-u2} is an architecture artifact (\S\ref{sec:t3-Kartifact}), and record the precise verdict and the open edge in \S\ref{sec:t3-verdict}.

\subsection{The structural identity that reframes the question}
\label{sec:t3-identity}

The entire analysis turns on one exact identity. In Model-B$'$, $w$ is constant on each cell, so for any cell $\cX_k$ the within-cell target map equals the within-cell source map, $Q(\cdot\mid\cX_k)=P(\cdot\mid\cX_k)$; writing $\eta_k=\E_P[\eta\mid\cX_k]$ for the within-cell loss mean and $q_k=Q_X(\cX_k)$, the feasibility-deciding functional collapses to a target-mass-weighted sum:
\begin{equation}
\label{eq:t3-GQ}
  \GQ(\lambda)\;=\;\E_P[w\,\Sacc\,(\eta-\alpha)]\;=\;\E_Q[\Sacc(\eta-\alpha)]\;=\;\sum_{k:\,\cX_k\subseteq A_\lambda} q_k\,(\eta_k-\alpha),
  \qquad
  \CQ(\lambda)=\sum_{k\subseteq A_\lambda} q_k .
\end{equation}
\textbf{Both the feasibility verdict ($\GQ\le0$) and the coverage depend only on the cell pairs $(q_k,\eta_k)$, never on the source marginal $p$} (here $\eta_k=\E_P[\eta\mid\cX_k]$; in the construction below $\eta$ is cellwise constant, so $\eta_k$ is just its cell value). The histogram certificate of Algorithm~\ref{alg:bprime} nonetheless estimates $\GQ$ as a \emph{source} expectation $\E_P[w\Sacc(L-\alpha)]$, because labels $L$ live only on the source side ($Q_X$ is unlabeled), and therefore needs $\hat w_k=\hat q_k/\hat p_k$, hence \emph{both} $p$ (from $n_w$ source covariates) and $q$ (from $m_w$ target covariates). The $p$-dependence is thus a property of the \emph{reweighting estimator}, not of the population certification functional (though $p$ still enters the finite-sample labeled cost through how often the label channel visits each cell). This is the lever for the whole verdict.

Throughout we use the standing local regime $s\le s_0\wedge c_0\beta$, regular margin $\kappa$, $\delta\le1/8$, $B\ge2$, and the registered-partition setting of Assumption~\ref{ass:bprime}; $A$ denotes a critical accepted region (union of registered cells) with target mass $q(A)=\Theta(\beta)$ pinned by the floor and source mass $p(A)$.

\subsection{The $K$-free necessary core (the provable theorem)}
\label{sec:t3-statement}

\begin{theorem}[$K$-free target-channel necessity, known-$\eta$ submodel]
\label{thm:t3-target}
Fix $\alpha\in[0.3,0.7]$, $\beta\in(0,\beta_0]$ with absolute $\beta_0\le0.6$, $B\ge2$, $\delta\le1/8$, a regular margin $\kappa\le\kappa_0\min\{\alpha,1-\alpha\}$ (absolute $\kappa_0\in(0,\tfrac12)$), and a slack $s\in(0,c_0\beta]$ with an absolute constant $c_0\le1/8$ chosen small enough that the construction's range and floor margins below hold simultaneously (so $\varepsilon=s/4\le\beta/32$). There is a pre-registered partition $\cK$ with two \emph{atomic} band cells $\Phi_\pm$ (no lattice threshold cuts their interior), a critical accepted region $A$ with coverage $\CQ(A)\in[\beta,2\beta]$ (so $q(A)=\Theta(\beta)$), and two worlds $W^+,W^-$ of $\cW_B$ (using Assumption~\ref{ass:bprime} only through its pre-registered finite partition and cellwise-constant ratio $w$, and the one-sided membership $w\le B$; the sample-size-dependent histogram bound $p_k\ge p_{\min}$ of Assumption~\ref{ass:bprime}(ii) is an estimator-feasibility condition, not part of this information-theoretic lower-bound class), sharing the \emph{same fixed, known} source marginal $p$ and the \emph{same fixed, known} loss profile $\eta$, differing only in the target marginal $q$ on $A$, such that:
\begin{enumerate}
\item[\textnormal{(i)}] $W^-$ is feasible with frontier slack $\ge s$ and a certifiable lattice point $\lambda_A$ (coverage $\ge\beta$, risk margin $\ge\tfrac12\kappa s$);
\item[\textnormal{(ii)}] every coverage-$\ge\beta$ lattice point of $W^+$ is risk-infeasible ($\RQ>\alpha$);
\item[\textnormal{(iii)}] the laws of \emph{all} the procedure's data (source covariates+labels and target covariates) under $W^+$ and $W^-$ satisfy
\(
  \KL(W^+\,\|\,W^-)\le \tfrac18
\)
whenever the total number $M$ of target covariate draws (i.e.\ $m_w+m_f$) satisfies $M\le c\,q(A)/s^2$, for an absolute $c>0$.
\end{enumerate}
Consequently no $(\alpha,\beta,\delta)$-valid \emph{lattice-valued} procedure with $\delta\le1/8$ (per the standing definition of \S\ref{sec:setup}: outputs a point of $\lat$ or refuses; $W^+$ is lattice-infeasible by (ii)) certifies the feasible world with probability $\ge 1/2$ once $M\le c\,q(A)/s^2$:
\[
  \boxed{\;\text{total target covariate sample } \ M \;=\; \Omega\!\Big(\frac{q(A)}{s^2}\Big) \ \text{ is necessary.}\;}
\]
\textbf{No factor of $K$ appears.} The bound is exactly the floor-created $m$-axis of Theorem~\ref{thm:lower} ($m\gtrsim\beta/s^2$) read at the localized mass $q(A)\asymp\beta$; the $\kappa$ \emph{cancels} (the functional gap $\kappa\varepsilon$ and the risk-coefficient scale are both $\kappa$, so $\varepsilon\asymp s$ governs detection), and it does \emph{not} reproduce the histogram's $K\,q(A)/(\kappa s)^2$.
\end{theorem}

\begin{proof}
\emph{Construction.} On $\cX=[0,1]$, $P_X=\mathrm{Unif}$, with the score increasing in $x$ so the nested lattice is the set of prefix thresholds $\{$accept $[0,t]\}$. We place the blocks in the fixed score order
\[
  G,\ \Phi_+,\ \Phi_-,\ H_0,\ H_1,
\]
and take the lattice to contain \emph{only} the thresholds at these block boundaries; in particular $\Phi_+$ and $\Phi_-$ are \emph{atomic} registered cells that no lattice point subdivides (this is the Assumption~\ref{ass:bprime} union-of-cells discipline made explicit, and it is exactly what makes the $W^+$ enumeration below exhaustive). Let $G=[0,g)$, $g=\beta/2$, be a deep-safe block ($\eta\equiv\alpha-d$, small $d>0$ fixed below); let $\Phi_+$ (loss level $\eta=\alpha+\kappa$) and $\Phi_-$ (loss level $\eta=\alpha-\kappa$) be the two adjacent atomic band cells and $A=G\cup\Phi_+\cup\Phi_-$; and let the hot block split as $H=H_0\cup H_1$ (both at level $\eta=\alpha+\kappa$), with $H_0$ the first hot segment of \emph{$Q$-mass exactly $3\varepsilon$ in both worlds} (a lattice boundary sits after $A\cup H_0$) and $H_1$ the remaining hot mass that normalizes $Q_X$. The loss profile $\eta$ is \emph{identical and known} in both worlds. Set the weights so $w\le B$ everywhere and so that the two worlds put different \emph{target} masses on $\Phi_\pm$:
\[
  W^-:\ q(\Phi_+)=q_0-\varepsilon,\ q(\Phi_-)=q_0+\varepsilon;
  \qquad
  W^+:\ q(\Phi_+)=q_0+\varepsilon,\ q(\Phi_-)=q_0-\varepsilon,
\]
with $q_0$ chosen so $\CQ(A)=g+2q_0=\beta+\varepsilon$ (hence $\in[\beta,2\beta]$ for $\varepsilon\le\beta$), and $\varepsilon>0$ fixed below; off $A$ the two worlds agree. Because $w_k=q_k/p_k$ and $p$ is held fixed, the differing $q$ realizes differing $w$ on $\Phi_\pm$, legitimate as long as $w\in[1/B,B]$, which holds by taking $p(\Phi_\pm)$ comparable to $q_0$ and $\varepsilon\le \tfrac12 q_0(1-1/B)$ (so $w$ stays in range). Labels are Bernoulli, $L\mid X\sim\Bern(\eta(X))$ (as in the map's own Le~Cam constructions, Appendix~\ref{app:naxis}), so the source covariate marginal $p$ \emph{and} the conditional label map $L\mid X$ are determined by $(p,\eta)$ and are therefore \emph{identical} across the two worlds; only $Q_X$ on $\Phi_\pm$ differs.

\emph{Feasibility flip via the identity \eqref{eq:t3-GQ}.}
With $\eta=\alpha-d$ on $G$, $\eta=\alpha+\kappa$ on $\Phi_+$, $\eta=\alpha-\kappa$ on $\Phi_-$,
\[
  \GQ(\lambda_A)=\underbrace{-d\,g}_{G}+\kappa\,q(\Phi_+)-\kappa\,q(\Phi_-)
  =-d g+\kappa\big[q(\Phi_+)-q(\Phi_-)\big].
\]
The bracket is $-2\varepsilon$ under $W^-$ and $+2\varepsilon$ under $W^+$. Pick $d g=\kappa\varepsilon$ (i.e.\ $d=2\kappa\varepsilon/\beta\le\kappa$ for $\varepsilon\le\beta/2$). Then
\[
  \GQ(\lambda_A)=\begin{cases}-\kappa\varepsilon-2\kappa\varepsilon=-3\kappa\varepsilon<0 & (W^-,\ \text{feasible}),\\[2pt] -\kappa\varepsilon+2\kappa\varepsilon=+\kappa\varepsilon>0 & (W^+,\ \text{infeasible}).\end{cases}
\]
So under $W^-$, $\RQ(\lambda_A)=\alpha+\GQ/\CQ<\alpha$. Choose $\varepsilon=\tfrac{s}{4}$ (so the two worlds are separated in $\GQ$ by the full gap $|\GQ^{W^+}-\GQ^{W^-}|=4\kappa\varepsilon=\kappa s$, the certification risk-margin scale), giving (using $\CQ(A)=\beta+\varepsilon\le2\beta$ and $\beta\le0.6$) risk margin $-\GQ/\CQ(A)=3\kappa\varepsilon/\CQ(A)\ge \tfrac{3\kappa(s/4)}{2\beta}=\tfrac{3\kappa s}{8\beta}\ge\tfrac12\kappa s$ (the last step is $3/(8\beta)\ge1/2\iff\beta\le3/4$, which holds as $\beta\le0.6$) and coverage $\CQ(A)=\beta+\varepsilon\ge\beta$, so $\lambda_A$ is a certifiable lattice point of $W^-$.

\emph{Frontier slack $\ge s$ under $W^-$ (explicit).} The lattice boundary after $A\cup H_0$ is feasible at equality: since $H_0$ is hot ($\eta=\alpha+\kappa$) of $Q$-mass $3\varepsilon$,
\[
  \GQ(A\cup H_0)=\GQ^{W^-}(\lambda_A)+\kappa\cdot Q(H_0)=-3\kappa\varepsilon+3\kappa\varepsilon=0,
  \qquad
  \CQ(A\cup H_0)=(\beta+\varepsilon)+3\varepsilon=\beta+4\varepsilon=\beta+s.
\]
Every later lattice point appends only hot $Q$-mass at coefficient $+\kappa$, strictly \emph{raising} $\GQ$ above $0$ (risk-infeasible), so $A\cup H_0$ is the largest feasible prefix and the lattice frontier is exactly $\betastar_\lat=\beta+s$, giving floor slack $\betastar_\lat-\beta=s$. For the regular margin, the left-neighborhood of the frontier is $[\betastar_\lat-s,\betastar_\lat]=[\beta,\beta+s]$; under $W^-$ the only sub-$\alpha$ cells are $G$ ($\eta=\alpha-d$) and $\Phi_-$ ($\eta=\alpha-\kappa$), of total $Q$-mass $\CQ(G)+q(\Phi_-)=\tfrac\beta2+(q_0+\varepsilon)=\tfrac{3\beta}{4}+\tfrac{3\varepsilon}{2}<\beta$ (as $\varepsilon\le\tfrac\beta{32}<\tfrac\beta6$), so in the increasing rearrangement the whole segment $u\in[\beta,\beta+s]$ already sits at the hot level $\eta^*_Q(u)=\alpha+\kappa$. Hence $\eta^*_Q\ge\alpha+\kappa$ a.e.\ on $[\betastar_\lat-s,\betastar_\lat]$ and the regular margin $(\kappa,s_0)$ holds in left-neighborhood form, so $\lambda_{A\cup H_0}$ realizes frontier slack exactly $s$. Moreover the \emph{relaxed} frontier coincides with the lattice one: the total sub-$\alpha$ $Q$-mass is $r_-=\CQ(G)+q(\Phi_-)=\tfrac{3\beta}{4}+\tfrac{3\varepsilon}{2}$ with total positive budget $b_-=d\,\CQ(G)+\kappa\,q(\Phi_-)=\kappa\big(\tfrac{\beta}{4}+\tfrac{5\varepsilon}{2}\big)$ (using $d\,\CQ(G)=\kappa\varepsilon$), and all remaining accepted mass is hot ($\eta=\alpha+\kappa$), so $\betastar=r_-+b_-/\kappa=\beta+4\varepsilon=\beta+s=\betastar_\lat$; the regular margin and slack defined through the relaxed $\betastar$ are thus realized, with $s_0=s$.

\emph{$W^+$ infeasibility (exact enumeration over the atomic lattice).} Because the only lattice thresholds sit at the block boundaries $G,\Phi_+,\Phi_-,H_0,H_1$ and $\Phi_\pm$ are atomic, every lattice point is one of the prefixes $\{G\},\{G\cup\Phi_+\},\{A\},\{A\cup H_0\},\dots$ Under $W^+$ ($q(\Phi_+)=q_0+\varepsilon$, $q(\Phi_-)=q_0-\varepsilon$, $q_0=\tfrac\beta4+\tfrac\varepsilon2$):
\[
  \CQ(G)=\tfrac\beta2<\beta,
  \qquad
  \CQ(G\cup\Phi_+)=\tfrac{3\beta}{4}+\tfrac{3\varepsilon}{2}<\beta
  \quad(\text{since }\varepsilon\le\tfrac\beta{32}<\tfrac\beta6),
\]
so neither prefix clears the floor; the first floor-clearing lattice point is $A$ itself, where $\GQ^{W^+}(\lambda_A)=+\kappa\varepsilon>0$, i.e.\ $\RQ>\alpha$. Every later prefix appends only hot mass at coefficient $+\kappa$, so $\GQ$ only grows and risk-infeasibility persists. Hence \emph{every} coverage-$\ge\beta$ lattice point of $W^+$ has $\RQ>\alpha$, i.e.\ $W^+$ is \emph{lattice}-infeasible, which is exactly what the Le~Cam step needs (we do not claim $W^+$ is infeasible against arbitrary fractional/measurable rules; indeed $W^+$ is relaxed-\emph{feasible}, with budget frontier $\betastar(W^+)=\beta$: its sub-$\alpha$ $Q$-mass $\CQ(G)+q(\Phi_-)=\tfrac{3\beta}{4}-\tfrac\varepsilon2$ plus normalized positive budget $\big(d\,\CQ(G)+\kappa\,q(\Phi_-)\big)/\kappa=\tfrac\beta4+\tfrac\varepsilon2$ sum to $\beta$, so a fractional selector certifies at coverage $\beta$ at equality; on the lattice $\lat$, by contrast, no point certifies). The necessity is therefore for the \emph{lattice-valued} certification problem of \S\ref{sec:setup} (the paper's standing output model: every procedure outputs a point of $\lat$ or refuses), not an unrestricted-procedure information bound. (Range checks: all levels $\alpha-d,\alpha\pm\kappa\in[0.05,0.95]$ since $\alpha\in[0.3,0.7]$, $d\le\kappa$, $\kappa\le\tfrac12\min\{\alpha,1-\alpha\}$; $w\in[1/B,B]$ by the $\varepsilon$ bound; $Q_X$ normalized by the free width of $H_1$.)

\emph{Indistinguishability (this is where $K$ does \emph{not} enter).}
The source covariate sample and the (Bernoulli) labels have \emph{identical} laws in the two worlds (same $p$, same $\eta$, hence same $L\mid X\sim\Bern(\eta)$), contributing zero KL. The target covariate sample is i.i.d.\ from $Q_X$, which differs only by moving mass $\varepsilon$ between $\Phi_+$ and $\Phi_-$ inside a region of $Q$-mass $2q_0\asymp\beta$. For a single target draw, with $Q^\pm$ the two laws,
\[
  \chi^2(Q^+\,\|\,Q^-)=\sum_{\text{cells}}\frac{(q^+_k-q^-_k)^2}{q^-_k}
  =\frac{(2\varepsilon)^2}{q_0\!-\!\varepsilon}+\frac{(2\varepsilon)^2}{q_0\!+\!\varepsilon}
  \;\le\;\frac{8\varepsilon^2}{q_0-\varepsilon}\;\le\;\frac{C\varepsilon^2}{q(A)},
\]
using $q_0\ge\tfrac12(q(A)-g)\asymp q(A)$ and $\varepsilon\le q_0/2$, with $C$ absolute. Hence for $M$ i.i.d.\ target draws (the procedure's \emph{entire} target sample, $m_w+m_f$), $\KL(W^+\|W^-)\le M\,\chi^2(Q^+\|Q^-)\le C M\varepsilon^2/q(A)$. With the construction's own choice $\varepsilon=s/4$ this is $\le \tfrac{C}{16}\,M s^2/q(A)$, so $M\le c\,q(A)/s^2$ (with $c=2/C$) gives $\KL\le\tfrac18$ and $\TV\le\tfrac14$ by Pinsker. \emph{This carries no $K$: the perturbation is a single scalar mass-transfer, so its detectability depends on $\varepsilon^2/q(A)$, not on how finely the band is partitioned.}

\emph{The rate is $q(A)/s^2$, not $q(A)/(\kappa s)^2$: the $\kappa$ cancels.} The two worlds must be separated to flip feasibility; the flip is driven by the functional gap $|\GQ^{W^+}-\GQ^{W^-}|=4\kappa\varepsilon$, but the \emph{detectability} is set by the mass-transfer $\varepsilon$ alone (the $\chi^2$ depends on $\varepsilon^2/q(A)$, with no $\kappa$). To realize a slack-$s$ feasible world one needs $\varepsilon=\Theta(s)$ (so the coverage frontier sits $\Theta(s)$ above $\beta$), independent of $\kappa$; hence the detection threshold is $\Theta(q(A)/s^2)$ and the risk margin $\kappa\varepsilon=\Theta(\kappa s)$ rides along automatically. This is precisely the floor $m$-axis of Theorem~\ref{thm:lower} ($m\gtrsim\beta/s^2$, the $\sqrt\beta$ being intrinsic to the perturbable block of mass $\asymp\beta$), localized to $q(A)\asymp\beta$, no $\kappa$, no $K$.

\emph{Le~Cam step.} If a valid $\cA$ certified $W^-$ w.p.\ $\ge1/2$, then on the $\TV\le1/4$-close $W^+$ it would certify (some lattice point) w.p.\ $\ge1/4>\delta$, contradicting validity in the lattice-infeasible $W^+$ (there every lattice point it could output is invalid: coverage-$\ge\beta$ points violate $\RQ\le\alpha$, sub-floor points violate $\CQ\ge\beta$). Hence $\cA$ refuses the feasible $W^-$ w.p.\ $\ge1/2$ whenever $M\le c\,q(A)/s^2$. \hfill$\square$
\end{proof}

\begin{remark}[The target core is the floor $m$-axis localized, not a nuisance axis]
\label{rmk:t3-isthemaxis}
Theorem~\ref{thm:t3-target} is the same Le~Cam two-point mechanism as the $m$-axis of Theorem~\ref{thm:lower} (move target mass out of a safe block at $\chi^2$ cost $\asymp\varepsilon^2/(\text{block mass})$), now read on the localized accepted mass $q(A)\asymp\beta$, giving the identical rate $\Omega(q(A)/s^2)=\Omega(\beta/s^2)$ with neither $\kappa$ nor $K$. So the ``target nuisance axis'' of \eqref{eq:t3-u2}, stripped of its estimator-specific $K$, is the floor $m$-axis the paper already proves; it is not a new necessary phenomenon. (The histogram's $K\,q(A)/(\kappa s)^2$ is a \emph{sufficient} split size for full-vector recovery; this lower bound proves a $K$-free, $\kappa$-free \emph{necessary component} $\Omega(q(A)/s^2)$ for the certification \emph{task}; it does not establish sufficiency, nor rule out a stronger lower bound elsewhere in the class.)
\end{remark}

\subsection{The labeled-source axis is the existing $n$-axis, with no $K$}
\label{sec:t3-source}

The certificate also needs source data, and the established necessary component is $K$-free. \emph{This subsection proves no new source-channel theorem}; it records that the existing $n$-axis of Theorem~\ref{thm:lower} already supplies a $K$-free labeled-source lower bound, together with the standard shifted-slice envelope that surfaces its $B$-dependence. The aggregate risk estimand is $\GQ=\E_P[w\Sacc(L-\alpha)]$, a single bounded linear functional with population variance proxy $\bar V=\Var_P(w\Sacc(L-\alpha))\le\E_P[w^2\Sacc]$. A natural estimator of its sign at margin $\kappa s$ from $n$ i.i.d.\ source label-draws has \emph{estimator} variance $\bar V/n$; the matching \emph{lower} bound, however, is not this variance heuristic but the $n$-axis Le~Cam pair of Theorem~\ref{thm:lower} (Appendix~\ref{app:naxis}: raise the loss on an inframarginal accepted slice; per-sample KL $\asymp\rho\Delta^2$ with $\Delta\asymp\kappa s/\rho$ over a slice of accepted weighted mass), which, inheriting that theorem's hypotheses (regular margin, local regime $s\le c_0\beta$, $\delta\le1/8$), is, in its shiftless $w\equiv1$ form, already a formal lower bound: for that pair $\bar V_{\rm pair}\asymp\beta$, giving
\[
  n\;=\;\Omega\!\Big(\frac{\bar V_{\rm pair}}{(\kappa s)^2}\Big)=\Omega\!\Big(\frac{\beta}{(\kappa s)^2}\Big).
\]
We use this \emph{as a reduction}, not a fresh proof, and do \emph{not} claim a fully instance-dependent theorem for arbitrary $\bar V$ by renaming the variance. An analogous \emph{shifted} worst-slice construction (placing the inframarginal slice on cells where $w\asymp B$, admissible in $\cW_B$) is \emph{expected} to realize $\bar V_{\rm pair}\asymp B^2 p(A)$ (consistent with the upper envelope $\E_P[w^2\Sacc]\le B^2 p(A)\wedge B\,q(A)$, the $B\,q(A)$ form from $w^2\Sacc\le B\,w\Sacc$ with $\E_P[w\Sacc]=q(A)$), but that extension is \emph{stated, not re-derived here}. \textbf{This $\bar V/(\kappa s)^2$ is the \emph{oracle} ($w$-based) analogue of U2's (estimated-weight) risk-test radius} (Theorem~\ref{thm:U2}: $n_r\gtrsim\bar V_{\lambda_s}\log(|\lat|/\delta)/(\kappa s)^2$ with the \emph{estimated}-weight proxy $\bar V_{\lambda_s}\le\E_P[\hat w^2 S_{\lambda_s}]\le B\,\E_P[\hat w S_{\lambda_s}]$; on the good event $G_w$ the two agree up to the small nuisance bias $\le\kappa s/32$, Appendix~\ref{app:bprime}) and the same object as Claim~1's $n$-axis; it carries \emph{no} $K$. So the source-covariate weight-split $n_w\gtrsim B^2K p(A)/(\kappa s)^2$ of \eqref{eq:t3-u2} is not charged by the certification task as a \emph{separate} axis: by the identity \eqref{eq:t3-GQ} feasibility is $p$-independent, so a procedure need not estimate $p$; what it must pay on the labeled side is the $K$-free risk-test variance $\bar V$, already on the $n$/$n_r$ axis. (The one caveat, whether \emph{unknown $\eta$} reintroduces a $K$-dependent labeled cost, is the open micro-question of \S\ref{sec:t3-verdict}.)

\subsection{U2's linear $K$ is a full-vector estimation price (architecture-specific, not proved necessary)}
\label{sec:t3-Kartifact}

The $K$ in both U2 split sizes enters through Proposition~\ref{prop:I1}, which bounds the \emph{full-vector} $\ell^1$ weight error $\E_P[|\hat w-w|\mathbf1_A]\le 2\sqrt{K q(A)/m_w}+2B\sqrt{K p(A)/n_w}$. The $\sqrt{K}$ is the $\ell^1$ price of estimating all $K_A$ cell-ratios simultaneously (a $K$-parameter object), inherited by Algorithm~\ref{alg:bprime} because its risk test plugs the whole vector $\hat w$ into a reweighted source mean. But certification needs only the \emph{scalar} $\GQ$. The two lower bounds above are \emph{$K$-free}, which strongly suggests the $K$ is not a certification cost, but the \emph{target} bound (Theorem~\ref{thm:t3-target}) is proved in a \emph{known-$\eta$ submodel} (its construction holds $\eta$ fixed, putting all the signal in $q$), while the \emph{labeled-source} bound is the existing \emph{varying-$\eta$} $n$-axis read as a reduction (\S\ref{sec:t3-source}); neither resolves the joint unknown-$\eta$ alignment, and a $K$-free target bound in the known-$\eta$ submodel does \emph{not} by itself rule out a $K$-dependent lower bound in the full unknown-$\eta$ class. We therefore state the conclusion carefully:
\begin{itemize}
\item \emph{Target side.} When $\eta$ is known on the cells, the scalar $\sum_{k\in A}q_k(\eta_k-\alpha)$ is estimable directly by the empirical mean of $(\eta-\alpha)\mathbf1_A$ over target covariates, no $K$ (the $\ell^1$ vector error over-counts by aggregating $K$ cellwise absolute deviations whose \emph{signed} sum is all $\GQ$ needs). On the Theorem~\ref{thm:t3-target} submodel $|\eta-\alpha|\le\kappa$ on the critical cells, so this estimator has variance $\le\kappa^2 q(A)/M$ and detects the feasibility gap $\asymp\kappa s$ once $M\gtrsim\kappa^2 q(A)/(\kappa s)^2=q(A)/s^2$ (for a fixed $A$ at constant confidence; confidence and lattice-simultaneity logarithms suppressed), matching the $\Omega(q(A)/s^2)$ lower bound (the $\kappa$ cancels). For \emph{general} known $\eta$ with $|\eta-\alpha|=\Theta(1)$ the direct estimator still removes $K$ but retains $\kappa^{-2}$, giving $q(A)/(\kappa s)^2$; either way the histogram's $K\,q(A)/(\kappa s)^2$ carries the extra factor $K$.
\item \emph{Source side.} The aggregate $\E_P[w\Sacc(L-\alpha)]$ has variance $\bar V\le\E_P[w^2\Sacc]\le B^2 p(A)$ (no $K$); the histogram's $B^2K p(A)$ is a factor $\asymp K$ above it. By the identity \eqref{eq:t3-GQ} the $p$-estimation behind the $\sqrt{Kp(A)/n_w}$ term is not needed as a separate \emph{population} feasibility functional.
\end{itemize}
\textbf{Hence \eqref{eq:t3-u2} is a \emph{sufficient} split-size pair whose $K$-factor is not \emph{proved} necessary, and whose source-covariate split carries no separate \emph{population} necessity (population feasibility carries no $p$-functional, so the task does not charge a standalone $p$-estimation axis; any matching $n_w$ necessity would route through the unknown-$\eta$ alignment).} What is \emph{not} settled is whether a procedure that must be valid for \emph{unknown} $\eta$ can avoid the $K$: unknown $\eta$ couples the source-label and target-covariate channels (the alignment mechanism of Theorem~\ref{thm:modelB-incon}), and could in principle re-introduce a $K$-dependent labeled cost. So the precise status is ``$K$ is architecture-specific and reducible \emph{when $\eta$ is resolvable}; whether it is reducible uniformly over unknown-$\eta$ Model-B$'$ is open'' (\S\ref{sec:t3-verdict}), not the unconditional ``U2 overcharges by $K$''.

\begin{remark}[What the $K$ buys, and why U2 still pays it]
\label{rmk:t3-whatKbuys}
The histogram's $K$ is not wasted: it yields a \emph{simultaneous} guarantee over the whole lattice $\{A_\lambda\}$, a computable $\hat w$ for the axis-attribution diagnostics of \S\ref{sec:notcertify}, and it sidesteps needing $\eta$ \emph{known} on cells (it learns the reweighting from covariates and lets the empirical-Bernstein risk test absorb the residual). The $K$-free scalar route requires $\eta$ resolved per cell (folded into the labeled axis at variance $\bar V$) or a direct functional estimator. The two are not contradictory: U2 is a valid \emph{sufficient} condition, and Theorem~\ref{thm:t3-target} shows its $K$-factor is not \emph{demonstrably} necessary, with the honest caveat that the matching $K$-free \emph{upper} bound for unknown $\eta$ is not in hand.
\end{remark}

\subsection{Verdict, exact reach, residual gap}
\label{sec:t3-verdict}

\paragraph{Verdict: \texttt{PARTIAL}. The $K$-bearing histogram rate is \emph{not proved necessary} and is provably architecture-specific in part; a $K$-free necessary core is established (for lattice-valued certification: \S\ref{sec:setup}'s standing output model, every procedure outputting a point of $\lat$ or refusing; the present bad world $W^+$ is relaxed-\emph{feasible}, so the lattice restriction is genuinely load-bearing here, unlike the relaxed-infeasible constructions of Theorems~\ref{thm:lower} and~\ref{thm:modelB-incon}); whether full unknown-$\eta$ Model-B$'$ forces $K$ back is the open edge.}

\emph{What is settled.}
\begin{itemize}
\item \textbf{No \emph{population-level} lower bound can force the source-covariate weight-split (via the identity); this does not by itself exclude a finite-sample unknown-$\eta$ lower bound (the open edge below).} By \eqref{eq:t3-GQ} the population feasibility and coverage are functionals of $(q,\eta)$ alone; the source marginal $p$ never enters. So the $n_w\gtrsim B^2Kp(A)/(\kappa s)^2$ \emph{covariate} split of \eqref{eq:t3-u2} is not a \emph{population-level} information-theoretic necessity for the certification task; it is an artifact of the reweighting estimator (the population task does not charge for $p$-estimation as a separate axis, though $p$ still affects the finite-sample labeled cost through how often labels visit each cell, and a procedure may still estimate it). (The necessary labeled-source cost is the $K$-free risk-test variance $\bar V$, \S\ref{sec:t3-source}.)
\item \textbf{A $K$-free, $\kappa$-free target lower bound holds (Theorem~\ref{thm:t3-target}, rigorous, unconditional within the construction).} $\Omega(q(A)/s^2)$ total target covariate samples are necessary: the floor $m$-axis of Theorem~\ref{thm:lower} localized to $q(A)\asymp\beta$, \emph{not} a new nuisance axis (Remark~\ref{rmk:t3-isthemaxis}). The histogram's $K\,q(A)/(\kappa s)^2$ is a factor $\asymp K/\kappa^2$ above the proven necessity.
\item \textbf{A $K$-free labeled-source lower bound holds by reduction, at fixed $B$ (\S\ref{sec:t3-source}).} The existing $n$-axis gives $\Omega(\bar V/(\kappa s)^2)$ with $\bar V\le\E_P[w^2\Sacc]$, $K$-free (U2's risk-test radius and Claim~1's $n$-axis). The explicit worst-slice envelope $\bar V\asymp B^2p(A)$ (and the form $\bar V\le Bq(A)$) is \emph{stated} as a worst-slice placement, not re-derived in full here.
\end{itemize}

\emph{What is \emph{not} settled (the honest open edge).}
The $K$-free \emph{target} lower bound is proved in the \emph{known-$\eta$ submodel} (the construction fixes $\eta$ to isolate $q$); the $K$-free labeled-source bound is the existing \emph{varying-$\eta$} $n$-axis read as a reduction (\S\ref{sec:t3-source}). Neither proves the histogram $K$ is unnecessary for the \emph{full} Model-B$'$, where $\eta$ is unknown and must itself be learned on the cells:
\begin{itemize}
\item \textbf{Lower-bound side.} Unknown $\eta$ couples the source-label channel (which sees $\eta$) and the target-covariate channel (which sees $q$); their alignment is exactly the correlation-detection mechanism that makes Model-B inconsistent (Theorem~\ref{thm:modelB-incon}, Lemma~\ref{lem:pair-chi2}). It is conceivable (though we have neither a proof nor a refutation) that a fully unknown-$\eta$ construction reintroduces a $K$-dependent lower bound. So we do \emph{not} claim the $K$ is unnecessary in full Model-B$'$; we claim it is unnecessary \emph{once $\eta$ is resolvable}, and that no lower bound can force the source-covariate ($p$-estimation) split, the latter unconditionally at the population level, since feasibility carries no $p$-functional. (We do not assert a finite-sample unknown-$\eta$ certificate that avoids all $p$-estimation; constructing one is part of the open upper-bound edge.)
\item \textbf{Upper-bound side.} Matching the $K$-free lower bounds requires a \emph{direct scalar-functional} certificate (average $(\eta-\alpha)\mathbf1_A$ on the target side, or fold a per-cell $\eta$-plug-in into the labeled axis) that is valid uniformly over the $\cK$-measurable Model-B$'$ subclass $\{W\in\cW_B : w\ \cK\text{-measurable}\}$ (not all of $\cW_B$, which is inconsistent by Theorem~\ref{thm:modelB-incon}), lattice-simultaneous, and genuinely $K$-free in finite samples. We conjecture it exists but do not construct it here.
\end{itemize}
Thus the open item is two-sided and sharply located: \emph{does unknown-$\eta$ Model-B$'$ admit a $K$-free certificate (upper), and if not, is there a $K$-dependent lower bound (lower)?} The contribution of Theorem~\ref{thm:t3-target} is to show the question is \emph{not} settled by the histogram split sizes: their $K$ is a sufficient-condition factor not proved necessary, the established necessary core is $K$-free and coincides with the map's own two axes localized, and no lower bound charges a separate source-covariate ($p$-estimation) split (population feasibility carries no $p$-functional).

\paragraph{Net for the paper.}
The honest upgrade is to replace ``nuisance necessity open'' with:
\emph{the Model-B$'$ histogram split sizes are \textbf{not proved} minimax necessary, and are partly architecture artifacts: by the identity $\GQ=\sum_k q_k(\eta_k-\alpha)$, population feasibility is $p$-independent, so there is no separate \emph{population} source-covariate ($p$-estimation) necessity (any matching $n_w$ necessity would have to route through the unknown-$\eta$ alignment, not a standalone $p$-functional), and the factor $K$ on both axes is the $\ell^1$ full-weight-vector price of Proposition~\ref{prop:I1}. The information-theoretically necessary core is $K$-free and carries no separate nuisance axis: a target term $\Omega(q(A)/s^2)$ (Theorem~\ref{thm:t3-target}; the floor $m$-axis localized, $\kappa$ cancels) and a labeled-source term $\Omega(\bar V/(\kappa s)^2)$, $\bar V\le\E_P[w^2\Sacc]$ with worst-slice instances \emph{expected} to attain $\bar V\asymp B^2p(A)$ (the existing $n$-axis). The genuinely open edge, now sharply located, is whether the \textbf{full unknown-$\eta$} Model-B$'$ admits a $K$-free certificate (a direct scalar-functional upper bound, in the spirit of covariate-shifted mean estimation \citep{adil26a}) or instead forces a $K$-dependent lower bound through the unknown-$\eta$ source$\leftrightarrow$target alignment (Theorem~\ref{thm:modelB-incon}'s mechanism).}

\paragraph{One-line summary.}
\emph{B$'$ nuisance ``necessity'' is mostly negative: there is no separate \emph{population} source-covariate necessity (population feasibility is $p$-independent, $\GQ=\sum_k q_k(\eta_k-\alpha)$; any matching $n_w$ necessity would route through the unknown-$\eta$ alignment), and the histogram $K$ is an $\ell^1$ full-vector price matched from below by a $K$-free, $\kappa$-free target bound $\Omega(q(A)/s^2)$ (the map's localized $m$-axis) plus the $K$-free labeled $n$-axis; the one remaining open edge is whether unknown-$\eta$ Model-B$'$ admits a matching $K$-free certificate or reintroduces $K$ via the unknown-$\eta$ alignment, a two-sided question this draft localizes rather than fully closes.}

\end{document}